\documentclass[11pt]{article}

\usepackage[T1]{fontenc}
\usepackage{lmodern}

\usepackage{amsmath, amssymb, amsthm}
\usepackage[margin=2.5cm]{geometry}
\usepackage{color, graphicx, float}
\usepackage{tikz}
\usepackage{booktabs}
\usepackage[justification=centering]{caption}
\usepackage[ruled]{algorithm2e}
\usepackage{cite}
\usepackage{bm}

\usepackage{etoolbox}
\makeatletter
\patchcmd{\@maketitle}{\LARGE \@title}{\LARGE\bfseries\@title}{}{}

\renewcommand{\@seccntformat}[1]{\csname the#1\endcsname.\quad}
\makeatother

\definecolor{darkblue}{rgb}{0,0,.5}

\usepackage{hyperref}
\hypersetup{
	colorlinks=true,		
	linkcolor=darkblue,		
	citecolor=darkblue,		
	urlcolor=darkblue 		
}

\makeatletter
\def\th@plain{%
	\thm@notefont{}
	\itshape 
}
\def\th@definition{%
	\thm@notefont{}
	\normalfont 
}

\renewenvironment{proof}[1][\proofname]{\par
	\normalfont
	\topsep0\p@\@plus3\p@ \trivlist
	\item[\hskip\labelsep\itshape
	#1\@addpunct{.}]\ignorespaces
}{%
	\qed\endtrivlist
}
\makeatother

\newtheorem{theorem}{Theorem}[section]
\newtheorem{lemma}{Lemma}[section]
\newtheorem{corollary}{Corollary}[section]

\newtheorem{assumption}{Assumption}[section]
\theoremstyle{definition}
\newtheorem{definition}{Definition}[section]

\newtheorem{remark}{Remark}[section]

\renewcommand\theenumi{(\roman{enumi})}

\renewcommand{\labelenumi}{\rm (\roman{enumi})}

\usepackage[shortlabels]{enumitem}

\AtBeginDocument{%
  \setlength{\abovedisplayskip}{4pt}
  \setlength{\belowdisplayskip}{4pt}
  \setlength{\abovedisplayshortskip}{2pt}
  \setlength{\belowdisplayshortskip}{2pt}
}

\allowdisplaybreaks

\newcommand*{\circled}[1]{\lower.7ex\hbox{\tikz\draw (0pt, 0pt)%
		circle (.5em) node {\makebox[1em][c]{\small #1}};}}

\begin{document}

\title{A Momentum-Based Variance-Reduced Algorithm for Federated Multiobjective Optimization}

\author{
Yong Zhao\thanks{College of Mathematics and Statistics, Chongqing Jiaotong University, Chongqing 400074, China.} 
\thanks{Key Laboratory of Complex Systems Optimization and Intelligent Control of Chongqing Municipal Education Commission,  Chongqing 400074, China. Email: zhaoyongty@126.com.},
~
Chunlin You\thanks{College of Mathematics and Statistics, Chongqing Jiaotong University, Chongqing 400074, China. E-mail:1332967558@qq.com.},
~
Minh N. Dao\thanks{Corresponding author. School of Science, RMIT University, Melbourne, VIC 3000, Australia. E-mail: minh.dao@rmit.edu.au}, 
~
and Zai-Yun Peng\thanks{School of Mathematics, Yunnan Normal University, 650092 Kunming, Yunnan, China.}
\thanks{Yunnan Key Laboratory of Modern Analytical Mathematics and Applications, Kunming, 650500, People’s Republic of China. E-mail: pengzaiyun@126.com.}
}
	
\date{August 24, 2026}	

\maketitle

\begin{abstract}
Federated learning has traditionally been formulated as a single-objective optimization problem, primarily focused on maximizing model utility. In real-world applications, however, machine learning models often need to optimize multiple and potentially conflicting objectives simultaneously. This motivates federated multiobjective optimization (FMOO), which provides a natural framework for jointly handling multiple task-specific objectives in federated learning. In this paper, we propose a momentum-based variance-reduced algorithm for federated multiobjective optimization. The method incorporates a momentum-driven gradient estimator into the local updates to reduce the variance of stochastic updates, leading to an improved convergence rate. We establish theoretical guarantees showing that the expected Pareto stationarity measure of a randomly selected output iterate decays at a rate of $\mathcal{O}(T^{-2/3})$, improving upon the $\mathcal{O}(T^{-1/2})$ rates established for existing methods such as FSMGDA and FedCMOO. Numerical experiments on federated multiobjective optimization benchmarks demonstrate the effectiveness and competitive performance of the proposed algorithm.
\end{abstract}


\noindent{\bf Mathematics Subject Classification (MSC 2020):}
90C29,	
90C15,	
90C26.	

\section{Introduction}

In recent years, multiobjective optimization (MOO) has emerged as a foundational paradigm for multi-agent and multi-task learning, with broad applications spanning multi-task neural network training \cite{Sener2018}, recommender systems \cite{Zhou2023}, learning-to-rank tasks \cite{Mahapatra2023, Momma2020}, and economic or transportation scheduling \cite{Alshabibi2025, Moghaddam2021, Yang2021}. Unlike single-objective optimization, MOO seeks to simultaneously optimize multiple interrelated and often conflicting objectives. Consequently, it is typically impossible to find a single solution that minimizes all objectives at once. Therefore, MOO algorithms instead aim to find a \emph{Pareto optimal}/\emph{stationary} solution, where any further improvement in one loss necessarily worsens another. Existing MOO algorithms are predominantly designed for centralized learning settings, where data from all participants must be aggregated on a central server for joint optimization. This centralized paradigm not only limits scalability in distributed multi-agent systems but also conflicts with increasingly stringent privacy and data localization requirements, particularly in sensitive domains such as healthcare and finance.

Federated learning (FL) has emerged as a transformative distributed machine learning paradigm that enables multiple participants to collaboratively train models without sharing raw data, thereby addressing critical challenges related to data privacy and data silos \cite{McMahan2017, Kairouz2021}. Since the seminal work by McMahan et al. \cite{McMahan2017}, which introduced FedAvg as the first representative FL algorithm, substantial research efforts have been devoted to understanding and extending this approach. FedAvg follows a client–server architecture: in each communication round, participating clients perform several steps of local stochastic gradient descent (SGD) on their private datasets, after which the server aggregates the uploaded local model parameters via weighted averaging to update the global model. Under standard assumptions, it has been shown that FedAvg achieves a sublinear convergence rate when decaying learning rates are employed \cite{Stich2019, Yu2018}. However, despite SGD's statistically optimal convergence rate, its performance often depends on careful tuning of the learning rate, and this sensitivity is also inherited by FedAvg.
To address these challenges, a wide range of methodological innovations have been proposed. Regularization-based methods \cite{Li2020a, Li2021, Dinh2020, Mansour2020, Sun2023} modify the local objective by introducing a proximal term that penalizes deviations from the global model, thereby promoting consistency across clients and mitigating client drift. Laplacian regularization is also employed in federated multitask learning to exploit relationships among client models \cite{Dinh2024}. Momentum-based techniques \cite{Liu2020, Huo2020, Xu2021, Sun2024} accumulate historical gradient information to correct update directions, accelerating convergence and reducing bias induced by non-IID data. In response to difficulties in learning rate tuning and unstable local optimization, recent studies \cite{Reddi2020} have integrated adaptive optimization algorithms—such as Yogi \cite{Zaheer2018}, Adagrad \cite{Ward2020} and  Adam —into the federated setting. These methods enable automatic adjustment of step sizes based on gradient statistics at both the local and global levels, providing convergence guarantees under suitable assumptions for both IID and non-IID settings. Variance reduction (VR) techniques \cite{Karimireddy2020, Liang2019, Murata2021, Gao2022, Cheng2024,Khanduri2021} further improve convergence stability by employing control variates or historical global gradient corrections to mitigate the client drift and variance caused by data heterogeneity. 
For more research progress on federated learning algorithms, please refer to \cite{Yang2025}.

Despite a rich body of literature on federated learning algorithms, most existing studies remain limited to single-objective optimization, where all clients jointly minimize one global loss function. However, in real-world applications, machine learning models often need to optimize multiple objectives simultaneously—such as improving model accuracy, and reducing training costs—which frequently conflict with one another. It is in this context that federated multiobjective optimization (FMOO) has emerged. By integrating the principles of federated learning and multiobjective optimization, it aims to achieve the synergistic optimization of multiple learning goals while preserving data privacy, thereby meeting the diverse demands of practical applications. Yang et al. \cite{Yang2024} propose two federated
multiobjective optimization (FMOO) algorithms called federated multi-gradient
descent averaging (FMGDA) and federated stochastic multi-gradient descent averaging
(FSMGDA) to solve FMOO, and establish the convergence of the two proposed algorithms under suitable conditions. Askin et al. \cite{Askin2024} introduce FedCMOO, a communication-efficient federated multiobjective optimization algorithm that improves convergence performance without scaling communication costs with the number of objectives. Furthermore, the effectiveness and advantages of the proposed method are empirically validated through extensive comparisons with FSMGDA. Hartmann et al. \cite{Hartmann2025} present the first comprehensive survey on the integration of multiobjective methods with federated learning, introducing a novel taxonomy to systematically classify existing works while offering insights into recent trends, open challenges, and future research directions.

Inspired by the work in \cite{Cheng2024,Yang2025,Yang2024,Askin2024,Khanduri2021,Fernando2024,Hartmann2025}, we propose FSMGDA-M-VR, a momentum-based variance-reduced algorithm for federated multiobjective optimization. The design of our method is motivated by two related but previously separate research directions. The first is momentum-based variance reduction in federated learning, where methods such as STEM \cite{Khanduri2021} improve the quality of stochastic descent directions and mitigate the effect of client-side sampling noise. The second is variance-reduced stochastic multiobjective optimization, where recent methods such as \cite{Fernando2024} show that tracking or momentum-type estimators can effectively reduce the bias and variance in multi-gradient construction. However, existing federated variance-reduction methods are mainly designed for single-objective optimization, while existing variance-reduced multiobjective methods are typically developed in centralized environments. Thus, they do not directly address the combined challenges of multiple conflicting objectives, heterogeneous clients, local stochastic updates, and server-client communication.

FSMGDA-M-VR addresses these challenges by reorganizing how objective-wise descent information and Pareto weights are computed in federated multiobjective training. Unlike FSMGDA \cite{Yang2024}, which aggregates accumulated stochastic gradients without explicit variance reduction, FSMGDA-M-VR constructs objective-wise momentum variance-reduced directions at each client. This reduces the stochastic error accumulated across local steps and provides more stable objective-wise information for server aggregation. It also differs from FedCMOO \cite{Askin2024}. In FedCMOO, the server first approximates the Gram matrix of the objective gradients, updates the task weights through projected gradient steps, and then sends these weights to clients for weighted local updates. In contrast, FSMGDA-M-VR does not compute the Pareto weights before local training. Instead, clients first maintain objective-wise variance-reduced directions, and the server computes the Pareto balancing weights only after aggregating these directions. Therefore, the common descent direction is constructed from variance-reduced multiobjective information rather than from a pre-local-update weight vector.

This design leads to a clear communication--convergence trade-off. FedCMOO enjoys the advantage that its per-round communication cost can be independent of the number of objectives, whereas FSMGDA-M-VR requires each client to transmit one direction per objective and hence its per-round communication cost scales linearly with the number of objectives. The benefit is that the server receives richer objective-wise and variance-reduced information, which helps construct a more stable Pareto common descent direction. In this sense, FSMGDA-M-VR can be viewed as a federated multiobjective adaptation of momentum-based variance reduction, where objective-wise local momentum tracking, federated aggregation, and adaptive Pareto direction construction are integrated into a single training protocol. The main technical contribution lies in the convergence analysis of this scheme under federated multiobjective stochastic updates, showing that the estimator error can be controlled across local steps and communication rounds. Furthermore, we establish a convergence rate of $\mathcal{O}(T^{-2/3})$, improving upon the $\mathcal{O}(T^{-1/2})$ rates of FSMGDA \cite{Yang2024} and FedCMOO \cite{Askin2024}. Extensive experiments further demonstrate the competitive performance of FSMGDA-M-VR over existing federated multiobjective optimization methods.

\section{Preliminaries}

FMOO aims at optimizing multiple objectives simultaneously:
\begin{equation}\label{1}
\min_{x\in \mathbb{R}^d}F(x):=\left(f_1(x),...,f_S(x)\right),
\end{equation}
where $x\in \mathbb{R}^d$ is the model parameter, and each $f_s$ represents the objective function of task $s\in[S]:=\{1,\dots,S\}$, defined by
\begin{equation*}
	f_s(x)=\frac{1}{M}\sum\limits_{i=1}^Mf_{s,i}(x),
\end{equation*}  
with 
\begin{equation*}
f_{s,i}(x)=\mathbb{E}_{\xi_i \sim \mathcal{D}_i}[F_s(x;\xi_i)],
\end{equation*}  
Here, $M$ denotes the number of clients, and the random variable $\xi_i$ represents a local data point available at client $i$. We consider a synchronous full-participation federated setting, where all clients participate in each communication round with equal weights. Client dropout, communication compression, and quantization are not considered in the present analysis.

Unlike single-objective optimization, in MOO there generally does not exist a single solution that simultaneously minimizes all objectives. Therefore, a more appropriate notion of optimality in MOO is Pareto optimality, which is formally defined as follows.

\begin{definition}[Pareto optimality] 
For any two solutions $x, y\in \mathbb{R}^{d}$, we say $x$ \emph{dominates} $y$ if and only if $f_s(x)\leq f_s(y), \forall s\in[S]$ and $f_s(x)<f_s(y), \exists s\in[S]$. A solution $x\in \mathbb{R}^{d}$ is \emph{Pareto optimal} if it is not dominated by any other solution. A solution $x\in \mathbb{R}^{d}$ is \emph{weakly Pareto optimal} if there does not exist a solution $y$ such that $f_s(y)< f_s(x), \forall s\in[S]$.
\end{definition}

Similar to single-objective non-convex optimization, finding a Pareto optimal solution in MOO can be computationally challenging, particularly in non-convex settings. Therefore, Pareto stationarity is commonly adopted as a practical first-order optimality criterion.
\begin{definition}[Pareto stationarity] 
\label{defi2.2}
A solution $x\in \mathbb{R}^{d}$ is said to be \emph{Pareto stationary} if there is no common descent direction $d\in \mathbb{R}^d$ such that $\nabla f_s(x)^\top d<0,\forall s\in[S]$.
\end{definition}

Gradient-based MOO algorithms typically seek a common 
descent direction $d\in\mathbb{R}^d$ such that 
$\nabla f_s(x)^{\top} d<0$ for all $s\in[S]$. If no such common descent direction 
exists at a given point, then according to Definition \ref{defi2.2}, $x$ is a Pareto stationary 
solution. To determine such a direction, the multi-gradient method constructs the
gradient matrix
$G(x):=[\nabla f_1(x),\dots,\nabla f_S(x)]\in\mathbb{R}^{d\times S}$
and identifies the optimal weight $\lambda^*$ by solving the quadratic optimization
problem
$\lambda^*(x)\in\arg\min_{\lambda\in\mathcal{C}}\|G(x)\lambda\|^2$,
where
$\mathcal{C}:=\{\lambda\in[0,1]^S:\sum_{s\in[S]}\lambda_s=1\}$.
It then defines the multi-gradient direction as
$g=G(x)\lambda^*=\sum_{s\in[S]}\lambda_s^*\nabla f_s(x)$.
If $\|g\|=0$, then $x$ is a Pareto stationary solution; otherwise,
$d=-g$ is a common descent direction, and the iteration is performed according to
$x\leftarrow x+\eta d$, or equivalently, $x\leftarrow x-\eta g$, where $\eta$ 
denotes the learning rate, until a Pareto stationary solution is reached. The SMGD method follows a similar procedure to MGD, with the main difference being that the exact gradient matrix $G(x)$ is replaced by its stochastic counterpart. Thus,
$\|g\|^2=\|G(x)\lambda^*\|^2$ can serve as a stationarity measure for assessing the convergence of non-convex MOO algorithms.

\begin{definition}[$\epsilon$-Pareto stationarity {\cite{Zhou2022}}] 
\label{defi2.3}
An iterate $x^t$ generated by a stochastic algorithm is said to be an \emph{$\epsilon$-Pareto stationary solution in expectation} if
\begin{equation*}
	\mathbb{E}\left[\mathcal{G}(x^t)\right] \leq \epsilon,
\end{equation*}
where
\begin{equation*}
	\mathcal{G}(x):=\min_{\lambda\in\mathcal{C}}
	\left\|\sum_{s=1}^{S}\lambda_s\nabla f_s(x)\right\|^2.
\end{equation*}
The expectation is taken with respect to all the randomness generated by the algorithm up to iteration $t$.
\end{definition}

To facilitate the subsequent convergence analysis, we introduce the following assumptions.
\begin{assumption}\label{A1}
\begin{enumerate}
\renewcommand\theenumi{(A\arabic{enumi})}
\renewcommand{\labelenumi}{\rm (A\arabic{enumi})}
\item
For each  $s \in [S]$,  $f_s$ is $L$-smooth, i.e., $\left\|\nabla f_s(x)-\nabla f_s(y)\right\| \leq L\|x-y\|$, for any $x, y \in \mathbb{R}^d$;
\item 
For each $s \in[S]$,  $i\in [M]$, and any realization $\xi_i$ drawn from $\mathcal{D}_i$, $F_s(\cdot; \xi_i)$ is $L$-smooth, i.e., $\|\nabla F_s(x ; \xi_{i})-\nabla F_s(y ; \xi_{i})\| \leq L\|x-y\|,\quad \forall\,x,y\in\mathbb{R}^d.$ 
\end{enumerate}
\end{assumption}

By Assumption \ref{A1} {\rm (A2)} and Jensen's inequality, each $f_{s,i}$ is also $L$-smooth. Indeed,
\begin{align*}
	\|\nabla f_{s,i}(x)-\nabla f_{s,i}(y)\|
	&=
	\left\|
	\mathbb{E}_{\xi_i}
	\left[
	\nabla F_s(x;\xi_i)-\nabla F_s(y;\xi_i)
	\right]
	\right\| \\
	&\leq
	\mathbb{E}_{\xi_i}
	\left[
	\|\nabla F_s(x;\xi_i)-\nabla F_s(y;\xi_i)\|
	\right] \\
	&\leq
	L\|x-y\|,
	\qquad \forall\,x,y.
\end{align*}

\begin{assumption}\label{A2}\rm There exists a constant $\sigma>0$ such that, for any
	$x\in\mathbb{R}^d$, $s\in[S]$, and $i\in[M]$,
	\[
	\mathbb{E}_{\xi_i\sim\mathcal{D}_i}
	\left[\nabla F_s(x;\xi_i)\right]
	=
	\nabla f_{s,i}(x),
	\]
	and
	\[
	\mathbb{E}_{\xi_i\sim\mathcal{D}_i}
	\left[
	\|\nabla F_s(x;\xi_i)-\nabla f_{s,i}(x)\|^2
	\right]
	\leq \sigma^2.
	\]
\end{assumption}

\section{FSMGDA-M-VR algorithm}

In what follows, we propose a momentum-based variance-reduced federated stochastic multi-gradient descent algorithm (FSMGDA-M-VR) for solving \eqref{1}, as outlined in Algorithm \ref{VSOMD}.

\begin{algorithm}[!ht]
	\caption{FSMGDA-M-VR}\label{VSOMD}
	\KwIn{Initial model $x^{-1}=x^0$, initialization batch size $B\in \mathbb{N}$, 
		initial gradient estimate 
		$g_s^{-1}=\frac{1}{BM}\sum_{i=1}^{M}\sum_{b=1}^{B}
		\nabla F_s(x^{-1};\xi_i^b)$ for $s\in[S]$ with mutually independent samples $\{\xi_i^b:i\in[M],\,$ $b\in[B]\}$
		and $\xi_i^b\sim\mathcal{D}_i$, 
		local learning rate $\eta>0$, global learning rate $\gamma>0$, and momentum parameter $\beta\in (0,1)$}
	\KwOut{$(x_a,\lambda_a)$, where $a$ is chosen uniformly at random from $\{0,1,\dots,T-1\}$.}
	\For{$t=0,1,\dots,T-1$}{
		\For{each client $i$ in parallel}{
		Initial local model $x_{s,i}^{t, 0}=x^{t}$
		\For{each $s$ in parallel}{
			\For{$k=0,1,\dots,K-1$}{
				Compute direction $g_{s,i}^{t,k}=\nabla F_{s}(x_{s,i}^{t,k};\xi_{i}^{t, k})+(1-\beta)\left(g_{s}^{t-1}-\nabla F_{s}(x^{t-1} ; \xi_{i}^{t, k})\right)$\\
				Update local model $x_{s,i}^{t, k+1}=x_{s,i}^{t, k}-\eta g_{s,i}^{t, k}$
				}
				Aggregate local updates $g_{s,i}^{t}=\frac{1}{K} \sum_{k=0}^{K-1}g_{s,i}^{t,k}$
		}
		}
		Compute $g_{s}^{t}=\frac{1}{M} \sum_{i=1}^{M}g_{s,i}^{t},\forall s\in[S]$\\
		Compute $\lambda^{t,*} \in[0,1]^{S}$ by solving\\
		\begin{equation*}
			\min_{\lambda_s^t\geq 0}\left\| \sum_{s\in[S]}\lambda_s^{t}g_s^t\right\|^2,\quad\text{s.t.}\sum_{s\in[S]}\lambda_s^t=1
		\end{equation*}\\
		Compute global multi-gradient direction $g^t = \sum_{s\in[S]} \lambda_s^{t,*}g_s^t$\\
		Update global model $x^{t+1} = x^t-\gamma g^t$
	}
\end{algorithm}

At communication round $t$ and local step $k$, each client $i$ independently
draws a sample $\xi_i^{t,k}\sim\mathcal{D}_i$. The samples are independent
across clients, local steps, and communication rounds, but are not necessarily
identically distributed across clients since the distributions
$\{\mathcal{D}_i\}_{i=1}^M$ may differ. For each fixed $(t,k,i)$, the same
sample $\xi_i^{t,k}$ is used for all objectives $s\in[S]$. In particular, the
variance-reduction mechanism evaluates
$\nabla F_s(x_{s,i}^{t,k};\xi_i^{t,k})$ and
$\nabla F_s(x^{t-1};\xi_i^{t,k})$ using the same realization
$\xi_i^{t,k}$.

As shown in Algorithm \ref{VSOMD}, after initializing the model parameters and historical gradients, the algorithm proceeds iteratively over multiple communication rounds. At each communication round $t$, all clients receive the current global model $x^t$ and, in parallel, perform $K$ local update steps for each objective function $s\in [S]$. In each local step $k$, client $i$ computes a variance-reduced descent direction for objective $s$ by combining the current stochastic gradient with a momentum-based correction term:
\begin{align*}
g_{s,i}^{t,k} = \nabla F_s(x_{s,i}^{t,k}; \xi_i^{t,k}) + (1 - \beta)\left(g_s^{t-1} - \nabla F_s(x^{t-1}; \xi_i^{t,k})\right).
\end{align*} 
The local model is then updated as $x_{s,i}^{t,k+1} = x_{s,i}^{t,k} - \eta g_{s,i}^{t,k}$. After completing all $K$ local steps, each client computes the average gradient
estimate for each objective $g_{s,i}^t = \frac{1}{K}\sum_{k=0}^{K-1} g_{s,i}^{t,k}$ and sends the directions to the server. On the server side, the uploaded estimates are aggregated across clients to form $g_s^t = \frac{1}{M}\sum_{i=1}^M g_{s,i}^t$. To determine a multi-gradient direction that balances multiple objectives, the server
solves the following convex quadratic optimization problem with linear constraints: $\min\limits_{\lambda_s^t\geq 0, \sum_{s\in[S]}\lambda_s^t=1}\left\| \sum_{s\in[S]}\lambda_s^{t}g_s^t\right\|^2$, which yields the optimal weight coefficients $\lambda_s^{t,*}$. These coefficients are used to construct the global direction $g^t = \sum_{s=1}^S \lambda_s^{t,*} g_s^t$. Finally, the global model is updated as $x^{t+1} = x^t - \gamma g^t$. 

\begin{remark}
 Relative to FSMGDA, the server-side Pareto aggregation mechanism remains unchanged. The main algorithmic modification is the introduction of an objective-wise momentum variance-reduced estimator into the local updates. The principal technical contribution is the analysis showing that this modification improves the stationarity rate from $\mathcal{O}(T^{-1/2})$ to $\mathcal{O}(T^{-2/3})$.
\end{remark}

\section{Convergence analysis}

In this section, we show that Algorithm~\ref{VSOMD} achieves a convergence rate of 
 $\mathcal{O}(T^{-\frac{2}{3}})$ in the nonconvex case. 
To facilitate the theoretical analysis, we first introduce the filtration used to characterize the evolution of the iterates and the randomness generated during the execution of the algorithm.

Let
\begin{equation*}
	\mathcal{F}^{0}
	=
	\sigma\left(
	\left\{\xi_i^b:
	1\leq i\leq M,\;1\leq b\leq B
	\right\}
	\right),
\end{equation*}
where $\{\xi_i^b:1\leq i\leq M,\;1\leq b\leq B\}$ are the samples used to initialize
$g_s^{-1}$, with $\xi_i^b\sim\mathcal{D}_i$. These initialization samples are mutually
independent and are independent of all samples generated in the subsequent communication
rounds. Hence, the initial gradient estimates $g_s^{-1}$, $s\in[S]$, are
$\mathcal{F}^{0}$-measurable. For each communication round $t$, $\mathcal{F}^{t}$ denotes the $\sigma$-algebra
containing all randomness revealed before the beginning of round $t$. Within round $t$,
we define the local-step filtration
\begin{equation*}
	\mathcal{F}^{t,k}
	=
	\sigma\left(
	\mathcal{F}^{t}
	\cup
	\left\{
	\xi_i^{t,j}:
	1\leq i\leq M,\;0\leq j<k
	\right\}
	\right),
	\qquad
	k=0,1,\dots,K.
\end{equation*}
Thus, $\mathcal{F}^{t,0}=\mathcal{F}^{t}$, and $\mathcal{F}^{t,k}$ contains all
randomness revealed before the $k$-th local update in round $t$. After completing the
$K$ local steps, we set
\begin{equation*}
	\mathcal{F}^{t+1}
	=
	\mathcal{F}^{t,K}
	=
	\sigma\left(
	\mathcal{F}^{t}
	\cup
	\left\{
	\xi_i^{t,j}:
	1\leq i\leq M,\;0\leq j<K
	\right\}
	\right).
\end{equation*}
Under this definition, the local variables $x_{s,i}^{t,k}$ and all global quantities
available before the $k$-th local update, including $x^t$, $x^{t-1}$, and
$g_s^{t-1}$, are $\mathcal{F}^{t,k}$-measurable, whereas the fresh sample
$\xi_i^{t,k}$ is independent of $\mathcal{F}^{t,k}$. Moreover, conditioned on
$\mathcal{F}^{t,k}$, the samples $\{\xi_i^{t,k}\}_{i=1}^{M}$ are independent across
clients and satisfy $\xi_i^{t,k}\sim\mathcal{D}_i.$
We use $\mathbb{E}[\cdot\mid\mathcal{F}^{t,k}]$ to denote the conditional expectation given the $\sigma$-algebra $\mathcal{F}^{t,k}$, which contains all randomness revealed before the $k$-th local update in round $t$, and $\mathbb{E}[\cdot]$ to denote the expectation over all randomness in the algorithm.
Throughout the proofs, we use $\sum_i$ to denote the summation over
$i\in\{1,\dots,M\}$ and $\sum_k$ to denote the summation over
$k\in\{0,\dots,K-1\}$. For all $t\ge0$, we define the auxiliary variables as follows:
\begin{align*}\mathcal{E}_{s}^t:=\mathbb{E}[\|\nabla f_s(x^t)-g_s^t\|^2],
\end{align*}
\begin{align*}U_{s}^t:=\frac{1}{KM}\sum_i\sum_k\mathbb{E}[\|x_{s,i}^{t,k}-x^t\|^2],
\end{align*}
\begin{align*}\zeta_{s,i}^{t,k}:=\mathbb{E}[x_{s,i}^{t,k+1}-x_{s,i}^{t,k}|\mathcal{F}^{t,k}],
\end{align*}
\begin{align*}\Xi_s^t:=\frac{1}{M}\sum_{i=1}^M\mathbb{E}[\|\zeta_{s,i}^{t,0}\|^2].
\end{align*}
Furthermore, let $x^{-1}:=x^0$ and
$\mathcal{E}_s^{-1}:=\mathbb{E}\left[\|\nabla f_s(x^0)-g_s^{-1}\|^2\right]$.
For the subsequent analysis, we also define the auxiliary initial aggregated
multi-gradient direction using the same minimum-norm aggregation rule as in
Algorithm \ref{VSOMD}. Specifically, let
\begin{equation*}
	\lambda^{-1,*}\in
	\arg\min_{\lambda\in\mathcal{C}}
	\left\|
	\sum_{s\in[S]}\lambda_s g_s^{-1}
	\right\|^2,
\end{equation*}
and define
\begin{equation*}
	g^{-1}:=
	\sum_{s\in[S]}\lambda_s^{-1,*}g_s^{-1}.
\end{equation*}
By construction, $g^{-1}$ is $\mathcal{F}^0$-measurable.

To establish the convergence properties of Algorithm \ref{VSOMD}, we further impose the following boundedness condition along the sequence of iterates generated by the algorithm.

\begin{assumption}\label{A3}\rm 
	Let $\{x^t\}_{t\geq 0}$ be the sequence generated by Algorithm \ref{VSOMD}. 
	There exists a constant $D>0$ such that
	\begin{equation*}
		\sup_{t\geq 0,\,s\in[S]}
		\frac{1}{M}\sum_{i=1}^{M}
		\mathbb{E}\left[\|\nabla f_{s,i}(x^t)\|^2\right]
		\leq D^2,
	\end{equation*}
	where the expectation is taken with respect to the randomness of the algorithm up to iteration $t$.
\end{assumption}

\begin{remark}
	Assumption \ref{A3} controls the local expected gradients along the sequence generated by Algorithm \ref{VSOMD}. Specifically, it requires the client-averaged squared norms of the local expected gradients, evaluated along the generated trajectory, to be uniformly bounded in expectation. Such boundedness conditions are commonly adopted in the analysis of federated learning and federated multiobjective learning
	\cite{Li2021,Reddi2020,Yang2024,Khanduri2021}, as well as in stochastic multiobjective optimization
	\cite{FernandoChen,XiaoHan,XuJu2025,Zhou2022}. Unlike a global bounded-gradient condition imposed for all $x$, Assumption \ref{A3} only requires boundedness along the sequence generated by the algorithm and is sufficient for our convergence analysis.
\end{remark}
Our subsequent analysis will be based on the following foundational lemmas.

\begin{lemma}\label{L1} 
Under Assumption \ref{A1}, if $\gamma L \leq \frac{1}{2}$, then, for all $t\ge 0 $,
\begin{align*}
	\mathbb{E}[f_s(x^{t+1})]\leq \mathbb{E}[f_s(x^t)]-\frac{\gamma}{4}\mathbb{E}[\|g^t\|^2]+\frac{\gamma}{2}\mathcal{E}_s^t, \quad s \in [S].
\end{align*}
and
\begin{align*}
     \frac{2\left(\mathbb{E}[f_s(x^T)]-f_s(x^0)\right)}{\gamma}+\frac{1}{2}\sum_{t=0}^{T-1}\mathbb{E}[\|g^t\|^2]\leq\sum_{t=0}^{T-1}\mathcal{E}_s^t, \quad s \in [S].
\end{align*}
\end{lemma}

\begin{proof}
By the $L$-smoothness of $f_s$, we have
\begin{align*}
f_{s}(x^{t+1}) &\leq f_{s}(x^{t})+\left\langle\nabla f_s(x^{t}),-\gamma g^t\right\rangle+\frac{1}{2} L \|\gamma g^{t}\|^{2}\\
	&= f_{s}(x^{t})+\left\langle\nabla f_s(x^{t})-g_s^t,-\gamma g^t\right\rangle-\gamma \left\langle g_s^t,g^t\right\rangle +\frac{1}{2} L \|\gamma g^{t}\|^{2}\\
	&\leq f_{s}(x^{t})-\gamma\left\langle\nabla f_s(x^{t})-g_s^t, g^t\right\rangle-\gamma \|g^t\|^2+\frac{1}{2} L \|\gamma g^{t}\|^{2}\\
	&\leq f_{s}(x^{t})+\frac{\gamma}{2}\|\nabla f_s(x^{t})-g_s^t\|^2+\frac{\gamma}{2}\|g^t\|^2-\gamma \|g^t\|^2+\frac{1}{2} L\gamma^2 \|g^{t}\|^{2}\\
	&=f_{s}(x^{t})+\frac{\gamma}{2}\|\nabla f_s(x^{t})-g_s^t\|^2-\gamma\left(\frac{1}{2}-\frac{1}{2}L\gamma\right)\|g^{t}\|^{2},
\end{align*}
where the second inequality follows from the fact that $g^t$ is the minimum-norm element in $\operatorname{conv}\{g_s^t:s\in[S]\}$. By the first-order optimality condition of this projection problem,
\begin{equation*}
	\langle g_s^t-g^t,g^t\rangle\ge 0,\quad \forall s\in[S],
\end{equation*}
which implies that
\begin{equation*}
\langle g_s^t,g^t\rangle\ge \|g^t\|^2,\quad \forall s\in[S].
\end{equation*}
With $\gamma \leq \frac{1}{2L}$ and taking the global expectation, we obtain
\begin{align*}
	\mathbb{E}[f_s(x^{t+1})]\leq \mathbb{E}[f_s(x^t)]-\frac{\gamma}{4}\mathbb{E}[\|g^t\|^2]+\frac{\gamma}{2}\mathcal{E}_s^t.
\end{align*}
Then, summing the above inequality over $t=0,\dots,T-1$ yields
\begin{align*}
       \frac{2}{\gamma}\sum_{t=0}^{T-1}\left(\mathbb{E}[f_s(x^{t+1})-f_s(x^t)]\right)+\frac{1}{2}\sum_{t=0}^{T-1}\mathbb{E}[\|g^t\|^2]=\frac{2\left(\mathbb{E}[f_s(x^T)]-f_s(x^0)\right)}{\gamma}+\frac{1}{2}\sum_{t=0}^{T-1}\mathbb{E}[\|g^t\|^2]\leq\sum_{t=0}^{T-1}\mathcal{E}_s^t,
\end{align*}
which completes the proof.
\end{proof}

 To handle local updates and client sampling, we will also use the following technical lemmas.

\begin{lemma}[See {\cite[Lemma 4]{Karimireddy2020}}]
\label{L2} 
Let $\{X_1,\dots,X_\tau\}$ be a collection of $\mathbb{R}^d$-valued random variables that are potentially dependent. Suppose that $\mathbb{E}[X_i]=\mu_i$ and $\mathbb{E}[\|X_i-\mu_i\|^2]\leq\sigma^2$, $i=1,\dots,\tau$. Then
\begin{equation*}
\mathbb{E}\left[\left\|\sum_{i=1}^{\tau}X_i\right\|^2\right]\leq\left\|\sum_{i=1}^{\tau}\mu_i\right\|^2
+\tau^2\sigma^2.
\end{equation*}
Suppose that $\{\mathcal{G}_i\}_{i=0}^{\tau}$ is a filtration, i.e.,
\begin{equation*}
\mathcal{G}_0\subseteq\mathcal{G}_1\subseteq\cdots\subseteq\mathcal{G}_\tau
\end{equation*}
and that $X_i$ is $\mathcal{G}_i$-measurable for each $i=1,\dots,\tau$. Define
\begin{equation*}
\mu_i=\mathbb{E}\left[X_i\mid\mathcal{G}_{i-1}\right],
\end{equation*}
and suppose that
\begin{equation*}
\mathbb{E}\left[\|X_i-\mu_i\|^2\mid\mathcal{G}_{i-1}\right]\leq\sigma^2,\qquad i=1,\dots,\tau.
\end{equation*}
Then $\{X_i-\mu_i\}_{i=1}^{\tau}$ forms a martingale-difference sequence with respect to $\{\mathcal{G}_i\}_{i=0}^{\tau}$, and
\begin{equation*}
			\mathbb{E}\left[
			\left\|\sum_{i=1}^{\tau}X_i\right\|^2
			\right]
			\leq
			2\mathbb{E}\left[
			\left\|\sum_{i=1}^{\tau}\mu_i\right\|^2
			\right]
			+2\tau\sigma^2.
\end{equation*}
\end{lemma}

\begin{lemma}\label{L3} Suppose that Assumptions \ref{A1}--\ref{A2} hold. If $\gamma L\leq\sqrt{\frac{\beta KM}{54}}$, then, for all $t\ge1$,
	\begin{align*}
	\mathcal{E}_s^t\leq(1-\beta)\mathcal{E}_s^{t-1}+\frac{4}{\beta}L^2U_s^t+\frac{3\beta^2\sigma^2}{KM}+\frac{2\beta}{9}\mathbb{E}[\|g^{t-1}\|^2].
	\end{align*}
and for $t=0$, 
\begin{align*}
	\mathcal{E}_s^0\leq(1-\beta)\mathcal{E}_s^{-1}+\frac4\beta L^2U_s^0+\frac{3\beta^2\sigma^2}{KM}.
\end{align*}
\end{lemma}

\begin{proof} By the definition of \(\mathcal{E}_s^t\), we have
	\begin{align*}
		\mathcal{E}_s^t 
&=\mathbb{E}\left[\left\|\frac{1}{KM}\sum_{i, k}\nabla F_{s}(x_{s,i}^{t,k};\xi_i^{t,k})+(1-\beta)\left(g_s^{t-1}-\frac{1}{KM}\sum_{i, k}\nabla F_{s}(x^{t-1};\xi_i^{t,k})\right)-\nabla f_s(x^t)\right\|^2\right] \\
&=\mathbb{E}\Biggl[\Biggl\|(1-\beta)\bigl(g_s^{t-1}-\nabla f_s(x^{t-1})\bigr)+\frac{1}{KM}\sum_{i, k}\nabla F_{s}(x_{s,i}^{t,k};\xi_i^{t,k})-\nabla f_s(x^t) \\
&\quad+(1-\beta)\Bigl(\nabla f_s(x^{t-1})-\frac{1}{KM}\sum_{i, k}\nabla F_{s}(x^{t-1};\xi_{i}^{t,k})\Bigr)\Biggr\|^2\Biggr] \\
	=&(1-\beta)^2\mathcal{E}_s^{t-1}+\underbrace{2\mathbb{E}\left[\left\langle(1-\beta)(g_s^{t-1}-\nabla f_s(x^{t-1})),\frac{1}{KM}\sum_{i, k}\nabla f_{s,i}(x_{s,i}^{t,k})-\nabla f_s(x^t)\right\rangle\right]}_{\Lambda_1} \\
	&+\underbrace{\mathbb{E}\left[\left\|\frac{1}{KM}\sum_{i, k}\nabla F_{s}(x_{s,i}^{t,k};\xi_i^{t,k})-\nabla f_s(x^t)+(1-\beta)\left(\nabla f_s(x^{t-1})-\frac{1}{KM}\sum_{i, k}\nabla F_{s}(x^{t-1};\xi_i^{t,k})\right)\right\|^2\right]}_{\Lambda_2}.	
	\end{align*}
By the AM-GM inequality and Assumption \ref{A1},
    \begin{align*}
    	\Lambda_1&\leq\beta(1-\beta)^2\mathcal{E}_s^{t-1}+\frac1\beta\mathbb{E}\left[\left\|\frac{1}{KM}\sum_{i, k}\nabla f_{s,i}(x_{s,i}^{t,k})-\nabla f_s(x^t)\right\|^2\right]\\
        &\leq\beta(1-\beta)^2\mathcal{E}_s^{t-1}+\frac{1}{\beta}\frac{1}{K^2 M^2} \mathbb{E}\left[\left\|\sum_{i, k}\left(\nabla f_{s,i}(x_{s,i}^{t,k})-\nabla f_{s,i}(x^t)\right)\right\|^2\right]\\
        &\leq\beta(1-\beta)^2\mathcal{E}_s^{t-1}+\frac{1}{\beta}\frac{1}{K M} \mathbb{E}\left[\sum_{i, k}\left\|\nabla f_{s,i}(x_{s,i}^{t,k})-\nabla f_{s,i}(x^t)\right\|^2\right]\\
        &\leq\beta(1-\beta)^2\mathcal{E}_s^{t-1}+\frac{1}{\beta}\frac{1}{K M} \mathbb{E}\left[\sum_{i, k}L^2\|x_{s,i}^{t,k}-x^t\|^2\right]\\
        &\leq\beta(1-\beta)^2\mathcal{E}_s^{t-1}+\frac1\beta L^2U_s^t.
    \end{align*}
For $\Lambda_2$, we decompose it into three terms as
    \begin{align*}
    \Lambda_{2} =&\mathbb{E}\left[\left\|\frac1{KM}\sum_{i,k}\left(\nabla F_{s}(x_{s,i}^{t,k};\xi_i^{t,k})-\nabla F_{s}(x^t;\xi_i^{t,k})\right)+\beta\left(\frac1{KM}\sum_{i,k}\nabla F_{s}(x^t;\xi_i^{t,k})-\nabla f_s(x^t)\right)\right.\right. \\
    &+(1-\beta)\left(\frac1{KM}\sum_{i, k}\left(\nabla F_{s}(x^t;\xi_i^{t,k})-\nabla F_{s}(x^{t-1};\xi_i^{t,k})\right)-\nabla f_s(x^t)+\nabla f_s(x^{t-1})\right)\Bigg\Vert^2\Bigg] \\
    \leq&\underbrace{3\mathbb{E}\left[\left\|\frac1{KM}\sum_{i,k}\left(\nabla F_{s}(x_{s,i}^{t,k};\xi_i^{t,k})-\nabla F_{s}(x^t;\xi_i^{t,k})\right)\right\|^2\right]}_{\Phi_1}\\
    &\underbrace{+3\mathbb{E}\left[\left\|\beta\left(\frac1{KM}\sum_{i,k}\nabla F_{s}(x^t;\xi_i^{t,k})-\nabla f_s(x^t)\right)\right\|^2\right]}_{\Phi_2}\\
    &\underbrace{+3\mathbb{E}\left[\left\|(1-\beta)\left(\frac1{KM}\sum_{i, k}\left(\nabla F_{s}(x^t;\xi_i^{t,k})-\nabla F_{s}(x^{t-1};\xi_i^{t,k})\right)\right)-\nabla f_s(x^t)+\nabla f_s(x^{t-1})\right\|^2\right]}_{\Phi_3}.
    \end{align*}
By Assumption \ref{A1}, 
\begin{align*}
 \Phi_1 
    &=\frac{3}{K^2 M^2}\mathbb{E}\left[\left\|\sum_{i, k}\left(\nabla F_{s}(x_{s,i}^{t, k} ; \xi_i^{t, k})-\nabla F_{s}(x^t ; \xi_i^{t, k})\right)\right\|^2\right]\\
    &\leq\frac{3}{K M} \mathbb{E}\left[\sum_{i, k}\left\|\nabla F_{s}(x_{s,i}^{t, k} ; \xi_i^{t, k})-\nabla F_{s}(x^{t} ; \xi_i^{t,k})\right\|^2\right]\\
    &\leq\frac{3}{K M}\sum_{i, k} \mathbb{E}\left[ L^2\left\|x_{s,i}^{t, k}-x^{t}\right\|^2\right]\\
    &=3L^2 U_s^t.
\end{align*}
By Assumption \ref{A2}, 
\begin{align*}
\Phi_2 
    &=\frac{3\beta^2}{K^2M^2}\mathbb{E}\left[\left\|\sum_{i,k}\left(\nabla F_{s}(x^t;\xi_i^{t,k})-\nabla f_{s,i}(x^t)\right)\right\|^2\right]\\
    &=\frac{3\beta^2}{K^2M^2}\mathbb{E}\left[\sum_{i,k}\left\|\nabla F_{s}(x^t;\xi_i^{t,k})-\nabla f_{s,i}(x^t)\right\|^2\right.\\
    &\quad \left.+2\sum_{(i,k)\ne(j,l)}\left\langle\nabla F_{s}(x^t;\xi_i^{t,k})-\nabla f_{s,i}(x^t),\nabla F_{s}(x^t;\xi_j^{t,l})-\nabla f_{s,j}(x^t)\right\rangle\right]\\
    &\leq \frac{3\beta^2}{K^2M^2}\sum_{i,k}\sigma^2=\frac{3\beta^2\sigma^2}{KM}.
\end{align*}
Again using Assumption \ref{A1}, we have
\begin{align*}
\Phi_3 
    &=\frac{3(1-\beta)^2}{K^2M^2}\mathbb{E}\left[\left\|\sum_{i, k}\left(\nabla F_{s}(x^t;\xi_i^{t,k})-\nabla F_{s}(x^{t-1};\xi_i^{t,k})+\nabla f_{s,i}(x^{t-1})-\nabla f_{s,i}(x^t)\right)\right\|^2\right]\\
    &=\frac{3(1-\beta)^2}{K^2M^2}\mathbb{E}\left[\sum_{i, k}\left\|\nabla F_{s}(x^t;\xi_i^{t,k})-\nabla F_{s}(x^{t-1};\xi_i^{t,k})+\nabla f_{s,i}(x^{t-1})-\nabla f_{s,i}(x^t)\right\|^2 \right.\\
    &\quad +2\sum_{(i,k)\ne (j,l)}\left\langle\nabla F_{s}(x^t;\xi_i^{t,k})-\nabla F_{s}(x^{t-1};\xi_i^{t,k})+\nabla f_{s,i}(x^{t-1})-\nabla f_{s,i}(x^t),\right.\\
    &\qquad\qquad \left.\nabla F_{s}(x^t;\xi_j^{t,l})-\nabla F_{s}(x^{t-1};\xi_j^{t,l})+\nabla f_{s,j}(x^{t-1})-\nabla f_{s,j}(x^t)\right\rangle\Bigg]\\
    &=\frac{3(1-\beta)^2}{K^2M^2}\mathbb{E}\left[\sum_{i, k}\left\|\nabla F_{s}(x^t;\xi_i^{t,k})-\nabla F_{s}(x^{t-1};\xi_i^{t,k})+\nabla f_{s,i}(x^{t-1})-\nabla f_{s,i}(x^t)\right\|^2\right]\\
    &\leq\frac{6(1-\beta)^2}{K^2M^2}\mathbb{E}\left[\sum_{i, k}\left\|\nabla F_{s}(x^t;\xi_i^{t,k})-\nabla F_{s}(x^{t-1};\xi_i^{t,k})\right\|^2+\sum_{i, k}\left\|\nabla f_{s,i}(x^{t-1})-\nabla f_{s,i}(x^t)\right\|^2\right]\\
    &\leq\frac{6(1-\beta)^2}{K^2M^2}\mathbb{E}\left[\sum_{i, k} L^2 \|x^t-x^{t-1}\|^2+\sum_{i, k} L^2 \|x^{t-1}-x^{t}\|^2\right]\\
    &=\frac{12(1-\beta)^2L^2}{KM}\mathbb{E}\left[||x^t-x^{t-1}\|^2\right].
\end{align*}
Substituting these three bounds, we obtain
    \begin{align*}
\Lambda_{2}\leq3L^{2}U_s^{t}+\frac{3\beta^{2}\sigma^{2}}{KM}+12(1-\beta)^{2}\frac{L^{2}}{KM}\mathbb{E}[\|x^{t}-x^{t-1}\|^{2}].
    \end{align*}
Therefore, for $t\ge 1$,
\begin{align*}
	\mathcal{E}_s^{t}&\leq(1-\beta)\mathcal{E}_s^{t-1}+\frac{4}{\beta}L^2U_s^t+\frac{3\beta^2\sigma^2}{KM}+12(1-\beta)^2\frac{L^2}{KM}\mathbb{E}[\|x^t-x^{t-1}\|^2] \\
	&\leq(1-\beta)\mathcal{E}_s^{t-1}+\frac{4}{\beta}L^2U_s^t+\frac{3\beta^2\sigma^2}{KM}+\frac{2\beta}{9}\mathbb{E}[\|g^{t-1}\|^2],
\end{align*}
where the last inequality is derived by $x^{t}=x^{t-1}-\gamma g^{t-1}$ and $\gamma L\leq\sqrt{\frac{\beta KM}{54}}$. Similarly, for $t=0$, we can obtain
\begin{align*}
	\mathcal{E}_s^0\leq(1-\beta)\mathcal{E}_s^{-1}+\frac4\beta L^2U_s^0+\frac{3\beta^2\sigma^2}{KM},
\end{align*}
which completes the proof.
\end{proof}

\begin{lemma}
\label{L4} 
Suppose that Assumptions \ref{A1}--\ref{A2} hold. If $\eta L K \leq 1$ and $24e^2(\eta L)^4K^4+24\eta^2L^2K\leq \frac{1}{2}$, then
	\begin{align*}
		U_s^t \leq 4 e^2 K^2 \Xi_s^t+12(\eta K)^2\left(8(\eta K L)^2+K^{-1}\right)\left(\beta^2 \sigma^2+4(\gamma L)^2 \mathbb{E}[\|g^{t-1}\|^2]\right).
	\end{align*}
\end{lemma}

\begin{proof}
Note that 
\begin{align*}
	\zeta_{s,i}^{t,k}
	&=\mathbb{E}[x_{s,i}^{t,k+1}-x_{s,i}^{t,k}\mid\mathcal{F}^{t,k}]
	=\mathbb{E}[-\eta g_{s,i}^{t,k}\mid\mathcal{F}^{t,k}]\\
	&=-\eta\left(\nabla f_{s,i}(x_{s,i}^{t,k})
	+(1-\beta)\left(g_s^{t-1}-\nabla f_{s,i}(x^{t-1})\right)\right).
\end{align*}
Throughout this proof, for a square-integrable random vector $X$, we use the scalar conditional variance
\begin{align*}
	\operatorname{Var}(X\mid\mathcal F)
	:=\mathbb E\!\left[\left\|X-\mathbb E[X\mid\mathcal F]\right\|^2\mid\mathcal F\right].
\end{align*}
Thus, we have
\begin{align*}
	\mathbb{E}[\|\zeta_{s,i}^{t,j}-\zeta_{s,i}^{t,j-1}\|^{2}] =&\mathbb{E}\left[\left\|-\eta\left(\nabla f_{s, i} (x_{s, i}^{t, j})+(1-\beta)\left(g_s^{t-1}-\nabla f_{s, i}(x^{t-1})\right)\right) \right.\right.\\
	&\left.\left.+\eta\left(\nabla f_{s, i} (x_{s, i}^{t, j-1})+(1-\beta)\left(g_s^{t-1}-\nabla f_{s, i}(x^{t-1})\right)\right) \right\|^2\right]\\
	=&\mathbb{E}\left[\left\|-\eta\nabla f_{s, i} (x_{s, i}^{t, j})+\eta\nabla f_{s, i} (x_{s, i}^{t, j-1})\right\|^2\right]\\
	\leq& \eta^{2} L^{2} \mathbb{E}[\|x_{s,i}^{t,j}-x_{s,i}^{t,j-1}\|^{2}] \\
	=& \eta^{2} L^{2}\left(\mathbb{E}[\|\zeta_{s,i}^{t,j-1}\|^{2}]+\mathbb{E}\left[\operatorname{Var}(x_{s,i}^{t,j}-x_{s,i}^{t,j-1} \mid \mathcal{F}^{t,j-1})\right]\right),
\end{align*}
where the last identity is the conditional bias--variance decomposition. Since
\begin{align*}
	& \mathbb{E}\left[\operatorname{Var}(x_{s,i}^{t, j}-x_{s,i}^{t, j-1} \mid \mathcal{F}^{t, j-1})\right] \\
	= &\mathbb{E}\left[\operatorname{Var}(-\eta g_{s,i}^{t,j-1} \mid \mathcal{F}^{t, j-1})\right] \\
	= &\mathbb{E}\left[ \mathbb{E}\left[ \| -\eta g_{s,i}^{t,j-1} - \mathbb{E}[-\eta g_{s,i}^{t,j-1} \mid \mathcal{F}^{t, j-1}] \|^2 \mid \mathcal{F}^{t, j-1} \right] \right] \\
	= &\eta^{2} \mathbb{E}\left[\left\|\nabla F_{s}(x_{s,i}^{t,j-1};\xi_{i}^{t,j-1})-\nabla f_{s,i}(x_{s,i}^{t, j-1})-(1-\beta)\left(\nabla F_{s}(x^{t-1} ; \xi_{i}^{t, j-1})-\nabla f_{s,i}(x^{t-1})\right)\right\|^{2}\right] \\
	= &\eta^2 \mathbb{E}\left[ \left\| \beta\left(\nabla F_{s}(x_{s,i}^{t,j-1} ; \xi_i^{t,j-1})-\nabla f_{s,i}(x_{s ,i}^{t,j-1})\right)+(1-\beta)\left(\nabla F_{s}(x_{s,i}^{t,j-1} ; \xi_i^{t,j-1})-\nabla F_{s}(x^{t-1} ; \xi_i^{t,j-1})\right)\right.\right. \\
	&\left.\left.+(1-\beta)\left(\nabla f_{s,i}(x^{t-1})-\nabla f_{s,i}(x_{s,i}^{t,j-1})\right)\right\|^2\right]\\
	\leq & \eta^2 \mathbb{E}\left[3 \beta^2\left\|\nabla F_{s}(x_{s,i}^{t,j-1} ; \xi_i^{t,j-1})-\nabla f_{s,i}(x_{s ,i}^{t,j-1})\right\|^2\right.\\
	&\left.+3(1-\beta)^2\left\|\nabla F_{s}(x_{s,i}^{t,j-1} ; \xi_i^{t,j-1})-\nabla F_{s}(x^{t-1} ; \xi_i^{t,j-1})\right\|^2+3(1-\beta)^2\left\|\nabla f_{s,i}(x^{t-1})-\nabla f_{s,i}(x_{s,i}^{t,j-1})\right\|^2\right]\\
	\leq & \eta^2\left(3\beta^2\sigma^2+6(1-\beta)^2L^2\mathbb{E}[\|x^{t-1}-x_{s,i}^{t,j-1}\|^2]\right),
\end{align*}
where the last inequality follows from Assumptions \ref{A1} and \ref{A2}. Thus, we have
\begin{align*}
	& \mathbb{E}[\|\zeta_{s,i}^{t, j}-\zeta_{s,i}^{t, j-1}\|^{2}] \\
	\leq& \eta^{2} L^{2}\left(\mathbb{E}[\|\zeta_{s,i}^{t, j-1}\|^{2}]+3 \beta^{2} \eta^{2} \sigma^{2}+6 \eta^{2}(1-\beta)^{2} L^{2} \mathbb{E}[\|x^{t-1}-x_{s,i}^{t, j-1}\|^{2}]\right) \\
	\leq&  \eta^{2} L^{2}\left(\mathbb{E}[\|\zeta_{s,i}^{t, j-1}\|^{2}]+3 \beta^{2} \eta^{2} \sigma^{2}+12 \eta^{2} L^{2} \mathbb{E}[\|x^{t-1}-x^{t}\|^{2}+\|x^{t}-x_{s,i}^{t, j-1}\|^{2}]\right).
\end{align*}
Fix $k\in\{1,\dots,K-1\}$. For every $j\in\{1,\dots,k-1\}$, Young's inequality and
$\eta L\leq K^{-1}\leq(k+1)^{-1}$ give
\begin{align}\label{3.2}
	& \mathbb{E}[\|\zeta_{s,i}^{t, j}\|^{2}] \leq(1+\frac {1}{k})\mathbb{E}[\|\zeta_{s,i}^{t,j-1}\|^2]+(1+k) \mathbb{E}[\|\zeta_{s,i}^{t, j}-\zeta_{s,i}^{t, j-1}\|^{2}] \nonumber\\
	\leq & (1+\frac{2}{k}) \mathbb{E}[\|\zeta_{s,i}^{t, j-1}\|^{2}]+(1+k) \eta^{2} L^{2}\left(3 \beta^{2} \eta^{2} \sigma^{2}+12 \eta^{2} L^{2} \mathbb{E}[\|x^{t-1}-x^{t}\|^{2}+\|x^{t}-x_{s,i}^{t, j-1}\|^{2}]\right)\nonumber \\
	=& (1+\frac{2}{k}) \mathbb{E}[\|\zeta_{s,i}^{t, j-1}\|^{2}]+12(1+k)(\eta L)^4\mathbb{E}[\|x^{t}-x_{s,i}^{t, j-1}\|^{2}]+(1+k) \eta^{4} L^{2}\left(3\beta^{2}  \sigma^{2}+12  L^{2} \mathbb{E}[\|x^{t-1}-x^{t}\|^{2}]\right)\nonumber \\
	\leq &(1+\frac{2}{k})\Big((1+\frac{2}{k}) \mathbb{E}[\|\zeta_{s,i}^{t, j-2}\|^{2}]+12(1+k)(\eta L)^4\mathbb{E}[\|x^{t}-x_{s,i}^{t, j-2}\|^{2}]\nonumber\\
	&+(1+k) \eta^{4} L^{2}\left(3\beta^{2}  \sigma^{2}+12  L^{2} \mathbb{E}[\|x^{t-1}-x^{t}\|^{2}]\right)\Big)+12(1+k)(\eta L)^4\mathbb{E}[\|x^{t}-x_{s,i}^{t, j-1}\|^{2}]\nonumber\\
	&+(1+k) \eta^{4} L^{2}\left(3\beta^{2} \sigma^{2}+12 L^{2} \mathbb{E}[\|x^{t-1}-x^{t}\|^{2}]\right)\nonumber\\
	=& (1+\frac{2}{k})^2 \mathbb{E}[\|\zeta_{s,i}^{t, j-2}\|^{2}]+(1+\frac{2}{k})12(1+k)(\eta L)^4\mathbb{E}[\|x^{t}-x_{s,i}^{t, j-2}\|^{2}]+12(1+k)(\eta L)^4\mathbb{E}[\|x^{t}-x_{s,i}^{t, j-1}\|^{2}]\nonumber\\
	&+(1+\frac{2}{k})(1+k) \eta^{4} L^{2}\left(3\beta^{2} \sigma^{2}+12 L^{2} \mathbb{E}[\|x^{t-1}-x^{t}\|^{2}]\right)+(1+k) \eta^{4} L^{2}\left(3\beta^{2} \sigma^{2}+12 L^{2} \mathbb{E}[\|x^{t-1}-x^{t}\|^{2}]\right) \nonumber\\
	& \vdots \nonumber\\
	\leq&(1+\frac{2}{k})^j \mathbb{E}[\|\zeta_{s,i}^{t, 0}\|^{2}]+12(1+k)(\eta L)^4\left((1+\frac{2}{k})^{j-1}\mathbb{E}[\|x^{t}-x_{s,i}^{t,0}\|^{2}]+\dots +(1+\frac{2}{k})^{0}\mathbb{E}[\|x^{t}-x_{s,i}^{t,j-1}\|^{2}]\right)\nonumber\\
	&+(1+k) \eta^{4} L^{2}\left(3\beta^{2} \sigma^{2}+12 L^{2} \mathbb{E}[\|x^{t-1}-x^{t}\|^{2}]\right)\left((1+\frac{2}{k})^{0}+\dots+(1+\frac{2}{k})^{j-1}\right)\nonumber\\
	\leq&(1+\frac{2}{k})^{k-1} \mathbb{E}[\|\zeta_{s,i}^{t, 0}\|^{2}]+12(1+k)(\eta L)^4\sum_{j'=0}^{j-1}\left((1+\frac{2}{k})^{j-1-j'}\mathbb{E}[\|x^{t}-x_{s,i}^{t,j'}\|^{2}]\right)\nonumber\\
	&+\frac{1-\left(1+\frac{2}{k}\right)^j}{1-\left(1+\frac{2}{k}\right)}(1+k) \eta^{4} L^{2}\left(3\beta^{2} \sigma^{2}+12 L^{2} \mathbb{E}[\|x^{t-1}-x^{t}\|^{2}]\right) \nonumber\\
	\leq & e^{2}\mathbb{E}[\|\zeta_{s,i}^{t, 0}\|^{2}]+12 e^{2} k(\eta L)^{4} \sum_{j^{\prime}=0}^{j-1} \mathbb{E}[\|x_{s,i}^{t, j^{\prime}}-x^{t}\|^{2}]+8 k^{2} L^{2} \eta^{4} \left(3 \beta^{2} \sigma^{2}+12 L^{2} \mathbb{E}[\|x^{t}-x^{t-1}\|^{2}]\right),
\end{align}
where the recursion is unrolled only over the index $j$ (with $k$ fixed), and we use
$1+k\leq2k$, and
$\left(1+\frac{2}{k}\right)^k\leq e^2$. Define
\begin{align*}
	\Delta_j&:=x_{s,i}^{t,j+1}-x_{s,i}^{t,j},
	&\varepsilon_j&:=\Delta_j-\zeta_{s,i}^{t,j},\qquad 0\leq j\leq k-1.
\end{align*}
Then $\varepsilon_j$ is $\mathcal F^{t,j+1}$-measurable and
$\mathbb E[\varepsilon_j\mid\mathcal F^{t,j}]=0$. If $0\leq j<\ell\leq k-1$, then
$\varepsilon_j$ is $\mathcal F^{t,\ell}$-measurable, and hence the tower property yields the martingale orthogonality relation
\begin{align*}
	\mathbb E\langle\varepsilon_j,\varepsilon_\ell\rangle
	&=\mathbb E\!\left[\mathbb E\bigl[\langle\varepsilon_j,\varepsilon_\ell\rangle
	\mid\mathcal F^{t,\ell}\bigr]\right]
	=\mathbb E\!\left[\left\langle\varepsilon_j,
	\mathbb E[\varepsilon_\ell\mid\mathcal F^{t,\ell}]\right\rangle\right]=0.
\end{align*}
Consequently,
\begin{align}\label{3.3}
	\mathbb{E}[\|x_{s,i}^{t,k}-x^t\|^2]
	&=\mathbb E\left[\left\|\sum_{j=0}^{k-1}(\zeta_{s,i}^{t,j}+\varepsilon_j)\right\|^2\right]\nonumber\\
	&\leq2\mathbb E\left[\left\|\sum_{j=0}^{k-1}\zeta_{s,i}^{t,j}\right\|^2\right]
	+2\mathbb E\left[\left\|\sum_{j=0}^{k-1}\varepsilon_j\right\|^2\right]\nonumber\\
	&=2\mathbb E\left[\left\|\sum_{j=0}^{k-1}\zeta_{s,i}^{t,j}\right\|^2\right]
	+2\sum_{j=0}^{k-1}\mathbb E[\|\varepsilon_j\|^2]\nonumber\\
	&=2\mathbb E\left[\left\|\sum_{j=0}^{k-1}\zeta_{s,i}^{t,j}\right\|^2\right]
	+2\sum_{j=0}^{k-1}\mathbb E\left[\operatorname{Var}(\Delta_j\mid\mathcal F^{t,j})\right]\nonumber\\
	\leq & 2 k \sum_{j=0}^{k-1} \mathbb{E}[\|\zeta_{s,i}^{t, j}\|^{2}]+2 \sum_{j=0}^{k-1}\left(3 \beta^{2} \eta^{2} \sigma^{2}+12 \eta^{2} L^{2} \mathbb{E}[\|x^{t-1}-x^{t}\|^{2}+\|x^{t}-x_{s,i}^{t, j}\|^{2}]\right),
\end{align}
where the last line uses Cauchy--Schwarz and the conditional-variance bound proved above. By combining \eqref{3.2} and \eqref{3.3},
\begin{align}\label{3.4}
	&\mathbb{E}[\|x_{s,i}^{t, k}-x^{t}\|^{2}] \nonumber\\
	\leq& 2k \sum_{j=0}^{k-1}\left(e^{2}\mathbb{E}[\|\zeta_{s,i}^{t, 0}\|^{2}]+12 e^{2} k(\eta L)^{4} \sum_{j^{\prime}=0}^{j-1} \mathbb{E}[\|x_{s,i}^{t, j^{\prime}}-x^{t}\|^{2}]+8 k^{2} L^{2} \eta^{4} \left(3 \beta^{2} \sigma^{2}+12 L^{2} \mathbb{E}[\|x^{t}-x^{t-1}\|^{2}]\right)      \right)\nonumber\\
	&+2 \sum_{j=0}^{k-1}\left(3 \beta^{2} \eta^{2} \sigma^{2}+12 \eta^{2} L^{2} \mathbb{E}[\|x^{t-1}-x^{t}\|^{2}+\|x^{t}-x_{s,i}^{t, j}\|^{2}]\right).
\end{align}
Summing \eqref{3.4} over $k=0,1,\dots,K-1$, we have
\begin{align*}
	&\sum_{k=0}^{K-1} \mathbb{E}[\|x_{s,i}^{t, k}-x^{t}\|^{2}] \\
	\leq& \underbrace{\sum_{k=0}^{K-1}2k \sum_{j=0}^{k-1}e^{2}\mathbb{E}[\|\zeta_{s,i}^{t, 0}\|^{2}]}_{\circled{1}}+ \underbrace{\sum_{k=0}^{K-1}2k \sum_{j=0}^{k-1}12 e^{2} k(\eta L)^{4} \sum_{j^{\prime}=0}^{j-1} \mathbb{E}[\|x_{s,i}^{t, j^{\prime}}-x^{t}\|^{2}]}_{\circled{2}}\\
	&+\underbrace{\sum_{k=0}^{K-1}2k \sum_{j=0}^{k-1}8 k^{2} L^{2} \eta^{4} \left(3 \beta^{2} \sigma^{2}+12 L^{2} \mathbb{E}[\|x^{t}-x^{t-1}\|^{2}]\right)}_{\circled{3}} +\underbrace{\sum_{k=0}^{K-1}2 \sum_{j=0}^{k-1}\left(3 \beta^{2} \eta^{2} \sigma^{2}+12 \eta^{2} L^{2} \mathbb{E}[\|x^{t-1}-x^{t}\|^{2}]\right)}_{\circled{4}}\\
	&+\underbrace{\sum_{k=0}^{K-1}2 \sum_{j=0}^{k-1}12 \eta^{2} L^{2}\mathbb{E}[\|x^{t}-x_{s,i}^{t, j}\|^{2}]}_{\circled{5}}.
\end{align*}
For $\circled{1}$, we have
\begin{align*}
	\sum_{k=0}^{K-1}2k \sum_{j=0}^{k-1}e^{2}\mathbb{E}[\|\zeta_{s,i}^{t, 0}\|^{2}]= \sum_{k=0}^{K-1}2k^2e^{2}\mathbb{E}[\|\zeta_{s,i}^{t, 0}\|^{2}] \leq 2K^3e^2\mathbb{E}[\|\zeta_{s,i}^{t, 0}\|^{2}].
\end{align*}
For $\circled{2}$, we have
\begin{align*}
	&\sum_{k=0}^{K-1}2k \sum_{j=0}^{k-1}12 e^{2} k(\eta L)^{4} \sum_{j^{\prime}=0}^{j-1} \mathbb{E}[\|x_{s,i}^{t, j^{\prime}}-x^{t}\|^{2}]\\
	=&\sum_{k=0}^{K-1}24e^{2} k^2(\eta L)^{4} \sum_{j=0}^{k-1}\left( \mathbb{E}[\|x_{s,i}^{t, 0}-x^{t}\|^{2}]+\mathbb{E}[\|x_{s,i}^{t, 1}-x^{t}\|^{2}]+\dots+\mathbb{E}[\|x_{s,i}^{t,j-1}-x^{t}\|^{2}]\right)\\
	=&\sum_{k=0}^{K-1}24e^{2} k^2(\eta L)^{4}\left(\mathbb{E}[\|x_{s,i}^{t, 0}-x^{t}\|^{2}]+\left(\mathbb{E}[\|x_{s,i}^{t, 0}-x^{t}\|^{2}]+\mathbb{E}[\|x_{s,i}^{t,1}-x^{t}\|^{2}]\right) \right.\\
	&\left.+\dots+\left(\mathbb{E}[\|x_{s,i}^{t, 1}-x^{t}\|^{2}]+\dots+\mathbb{E}[\|x_{s,i}^{t, k-2}-x^{t}\|^{2}]\right)\right)\\
	=&\sum_{k=0}^{K-1}24e^{2} k^2(\eta L)^{4}\left((k-2)\mathbb{E}[\|x_{s,i}^{t, 1}-x^{t}\|^{2}]+(k-3)\mathbb{E}[\|x_{s,i}^{t, 2}-x^{t}\|^{2}]+\dots+\mathbb{E}[\|x_{s,i}^{t,k-2}-x^{t}\|^{2}]\right)\\
	\leq&\sum_{k=0}^{K-1}24e^{2} k^2(\eta L)^{4}\left(k\mathbb{E}[\|x_{s,i}^{t, 1}-x^{t}\|^{2}]+k\mathbb{E}[\|x_{s,i}^{t, 2}-x^{t}\|^{2}]+\dots+k\mathbb{E}[\|x_{s,i}^{t,k-2}-x^{t}\|^{2}]\right)\\
	=&24e^{2} (\eta L)^{4}\sum_{k=0}^{K-1}k^3\left(\mathbb{E}[\|x_{s,i}^{t, 1}-x^{t}\|^{2}]+\mathbb{E}[\|x_{s,i}^{t, 2}-x^{t}\|^{2}]+\dots+\mathbb{E}[\|x_{s,i}^{t,k-2}-x^{t}\|^{2}]\right)\\
	=&24e^{2} (\eta L)^{4}\left(3^3\mathbb{E}[\|x_{s,i}^{t, 1}-x^{t}\|^{2}]+4^3\left(\mathbb{E}[\|x_{s,i}^{t, 1}-x^{t}\|^{2}]+\mathbb{E}[\|x_{s,i}^{t,2}-x^{t}\|^{2}]\right)\right.\\
	&\left.+\dots+(K-1)^3\left(\mathbb{E}[\|x_{s,i}^{t, 1}-x^{t}\|^{2}]+\dots+\mathbb{E}[\|x_{s,i}^{t, K-3}-x^{t}\|^{2}]\right)\right)\\
	\leq&24e^{2} (\eta L)^{4}\left(K^3\mathbb{E}[\|x_{s,i}^{t, 1}-x^{t}\|^{2}]+K^3\left(\mathbb{E}[\|x_{s,i}^{t, 1}-x^{t}\|^{2}]+\mathbb{E}[\|x_{s,i}^{t,2}-x^{t}\|^{2}]\right)\right.\\
	&\left.+\dots+K^3\left(\mathbb{E}[\|x_{s,i}^{t, 1}-x^{t}\|^{2}]+\dots+\mathbb{E}[\|x_{s,i}^{t, K-3}-x^{t}\|^{2}]\right)\right)\\
	=&24e^{2} (\eta L)^{4}\left((K-3)K^3\mathbb{E}[\|x_{s,i}^{t, 1}-x^{t}\|^{2}]+\dots+K^3\mathbb{E}[\|x_{s,i}^{t,K-3}-x^{t}\|^{2}]\right)\\
	\leq&24e^{2} (\eta L)^{4}K^4\sum_{k=0}^{K-1}\mathbb{E}[\|x_{s,i}^{t,k}-x^{t}\|^{2}].
\end{align*}
For $\circled{3}$, we have
\begin{align*}
	&\sum_{k=0}^{K-1}2k \sum_{j=0}^{k-1}8 k^{2} L^{2} \eta^{4} \left(3 \beta^{2} \sigma^{2}+12 L^{2} \mathbb{E}[\|x^{t}-x^{t-1}\|^{2}]\right)\\
	=&\sum_{k=0}^{K-1}16k^{4} L^{2} \eta^{4}\left(3 \beta^{2} \sigma^{2}+12 L^{2} \mathbb{E}[\|x^{t}-x^{t-1}\|^{2}]\right)\\
	\leq& 16K^{5} L^{2} \eta^{4}\left(3 \beta^{2} \sigma^{2}+12 L^{2} \mathbb{E}[\|x^{t}-x^{t-1}\|^{2}]\right).
\end{align*}
For $\circled{4}$, we have
\begin{align*}
	&\sum_{k=0}^{K-1}2 \sum_{j=0}^{k-1}\left(3 \beta^{2} \eta^{2} \sigma^{2}+12 \eta^{2} L^{2} \mathbb{E}[\|x^{t-1}-x^{t}\|^{2}]\right)\\
	=&\sum_{k=0}^{K-1}2k\left(3 \beta^{2} \eta^{2} \sigma^{2}+12 \eta^{2} L^{2} \mathbb{E}[\|x^{t-1}-x^{t}\|^{2}]\right)\\
	\leq&2K^2\left(3 \beta^{2} \eta^{2} \sigma^{2}+12 \eta^{2} L^{2} \mathbb{E}[\|x^{t-1}-x^{t}\|^{2}]\right).
\end{align*}
For $\circled{5}$, we have
\begin{align*}
	&\sum_{k=0}^{K-1}2 \sum_{j=0}^{k-1}12 \eta^{2} L^{2}\mathbb{E}[\|x^{t}-x_{s,i}^{t, j}\|^{2}]\\
	=&\sum_{k=0}^{K-1}24 \eta^{2} L^{2}\left(\mathbb{E}[\|x^{t}-x_{s,i}^{t, 1}\|^{2}]+\dots+\mathbb{E}[\|x^{t}-x_{s,i}^{t, k-1}\|^{2}]\right)\\
	=&24 \eta^{2} L^{2}\left(\mathbb{E}[\|x^{t}-x_{s,i}^{t, 1}\|^{2}]+\left(\mathbb{E}[\|x^{t}-x_{s,i}^{t, 1}\|^{2}]+\mathbb{E}[\|x^{t}-x_{s,i}^{t, 2}\|^{2}]\right)\right.\\
	&\left.+\dots+\left(\mathbb{E}[\|x^{t}-x_{s,i}^{t, 1}\|^{2}]+\dots+\mathbb{E}[\|x^{t}-x_{s,i}^{t, K-2}\|^{2}]\right)\right)\\
	=&24 \eta^{2} L^{2}\left((K-2)\mathbb{E}[\|x^{t}-x_{s,i}^{t, 1}\|^{2}]+(K-3)\mathbb{E}[\|x^{t}-x_{s,i}^{t, 2}\|^{2}]+\dots+\mathbb{E}[\|x^{t}-x_{s,i}^{t, K-2}\|^{2}]\right)\\
	\leq&24 \eta^{2} L^{2}K\sum_{k=0}^{K-1}\mathbb{E}[\|x^{t}-x_{s,i}^{t, k}\|^{2}].
\end{align*}
Combining $\circled{1}$--$\circled{5}$ yields
\begin{align*}
	&\sum_{k=0}^{K-1}\mathbb{E}[\|x_{s,i}^{t, k}-x^{t}\|^{2}]\\
	\leq&2K^3e^2\mathbb{E}[\|\zeta_{s,i}^{t, 0}\|^{2}]+24e^{2} (\eta L)^{4}K^4\sum_{k=0}^{K-1}\mathbb{E}[\|x_{s,i}^{t,k}-x^{t}\|^{2}]+16K^{5} L^{2} \eta^{4}\left(3 \beta^{2} \sigma^{2}+12 L^{2} \mathbb{E}[\|x^{t}-x^{t-1}\|^{2}]\right)\\
	&+2K^2\left(3 \beta^{2} \eta^{2} \sigma^{2}+12 \eta^{2} L^{2} \mathbb{E}[\|x^{t-1}-x^{t}\|^{2}]\right)+24 \eta^{2} L^{2}K\sum_{k=0}^{K-1}\mathbb{E}[\|x^{t}-x_{s,i}^{t, k}\|^{2}]\\
	=&2K^3e^2\mathbb{E}[\|\zeta_{s,i}^{t, 0}\|^{2}]+(48K^5\eta^4L^2+6K^2\eta^2)\left(\beta^{2} \sigma^{2}+4 L^{2} \mathbb{E}[\|x^{t}-x^{t-1}\|^{2}]\right)\\
	&+\left(24e^2(\eta L)^4K^4+24\eta^2L^2K\right)\sum_{k=0}^{K-1}\mathbb{E}[\|x^t-x_{s,i}^{t,k}\|^2]\\
	\leq&2K^3e^2\mathbb{E}[\|\zeta_{s,i}^{t, 0}\|^{2}]+(48K^5\eta^4L^2+6K^2\eta^2)\left(\beta^{2} \sigma^{2}+4 L^{2} \mathbb{E}[\|x^{t}-x^{t-1}\|^{2}]\right)+\frac{1}{2}\sum_{k=0}^{K-1}\mathbb{E}[\|x^t-x_{s,i}^{t,k}\|^2],
\end{align*}
where the last inequality follows from $24e^2(\eta L)^4K^4+24\eta^2L^2K\leq \frac{1}{2}$. Thus, we obtain
\begin{align*}
	\sum_{k=0}^{K-1}\mathbb{E}[\|x_{s,i}^{t, k}-x^{t}\|^{2}]\leq 4K^3e^2\mathbb{E}[\|\zeta_{s,i}^{t, 0}\|^{2}]+12(\eta K)^2(8K^3\eta^2L^2+1)\left(\beta^{2} \sigma^{2}+4 L^{2} \mathbb{E}[\|x^{t}-x^{t-1}\|^{2}]\right),
\end{align*}
Finally, summing over $i=1,\dots,M$, dividing by $KM$, using the definitions of
$U_s^t$ and $\Xi_s^t$, and substituting $x^t-x^{t-1}=\gamma g^{t-1}$, the conclusion follows.
\end{proof}

\begin{lemma}\label{L5} Suppose that Assumptions \ref{A1} and \ref{A3} hold. If $\max\{1152e^2(1-\beta)^2 \left(\eta KL\right)^2, 1152e^2(1-\beta)^2 \left(\eta \gamma KL^2\right)^2\} \leq \beta^{2}$, then the following inequality holds:
\begin{equation*}
	\sum_{t=0}^{T-1} \Xi_{s}^t \leq \frac{\beta^{2}}{288 e^2 K^{2} L^{2}} \sum_{t=0}^{T-1}\left(\mathcal{E}_{s}^{t-1}+\mathbb{E}[\|g^{t-1}\|^{2}]\right)+4 \eta^{2}TD^2.
\end{equation*}
\end{lemma}
\begin{proof}
Recall that $\zeta_{s,i}^{t, 0}=-\eta\left((1-\beta)\left(g_{s}^{t-1}-\nabla f_{s,i}(x^{t-1})\right)+\nabla f_{s,i}(x^{t})\right)$. Consequently,
\begin{align*}
	\|\zeta_{s,i}^{t, 0}\|^{2} & \leq 2 \eta^{2}\left((1-\beta)^{2}\|g_{s}^{t-1}\|^{2}+\|\nabla f_{s,i}(x^{t})-(1-\beta) \nabla f_{s,i}(x^{t-1})\|^{2}\right) \\
    &= 2 \eta^{2}(1-\beta)^{2}\|g_{s}^{t-1}\|^{2}+2 \eta^{2}\left\|(1-\beta)\left(\nabla f_{s,i}(x^{t})-\nabla f_{s,i}(x^{t-1})\right)+\beta\nabla f_{s,i}(x^{t})\right\|^2\\
    &\leq 2 \eta^{2}(1-\beta)^{2}\|g_{s}^{t-1}\|^{2}+4 \eta^{2}(1-\beta)^2\|\nabla f_{s,i}(x^{t})-\nabla f_{s,i}(x^{t-1})\|^2+4 (\eta\beta) ^2\|\nabla f_{s,i}(x^{t})\|^2\\
	& \leq 2 \eta^{2}(1-\beta)^2\|g_s^{t-1}\|^2+4(\eta L)^2(1-\beta)^2 \|x^t-x^{t-1}\|^2+4(\eta \beta)^2\|\nabla f_{s,i}(x^t)\|^2,
\end{align*}
where the last inequality follows from Assumption \ref{A1}. By Assumption \ref{A3} and the definition of $\Xi_{s}^t$, and summing over $t=0, \dots,T-1$, we have
\begin{align*}
	\sum_{t=0}^{T-1} \Xi_{s}^t = &\sum_{t=0}^{T-1}\frac{1}{M}\sum_{i=1}^{M}\mathbb{E}[\|\zeta_{s,i}^{t, 0}\|^{2}]\\
    \leq& \sum_{t=0}^{T-1}\mathbb{E}\left(2\eta^2(1-\beta)^2\|g_s^{t-1}\|^2+4(\eta L)^2(1-\beta)^2\|x^t-x^{t-1}\|^2+4(\eta \beta)^2 \frac{1}{M}\sum_{i=1}^{M}\|\nabla f_{s,i}(x^t)\|^2  \right) \\
	\leq& \sum_{t=0}^{T-1}\mathbb{E} \Bigg( 4\eta^2(1-\beta)^2\left(\|g_s^{t-1}-\nabla f_s(x^{t-1})\|^2+\|\nabla f_s(x^{t-1})\|^2\right)+4(\eta L)^2(1-\beta)^2\|x^t-x^{t-1}\|^2 \\
	&\left.+4(\eta \beta)^2 \frac{1}{M}\sum_{i=1}^{M}\|\nabla f_{s,i}(x^t)\|^2 \right) \\
	\leq& 4 \eta^2 (1-\beta)^2 \sum_{t=0}^{T-1}\mathcal{E}_s^{t-1}+4 \eta^2 (1-\beta)^2 TD^2+4 \eta^2 (1-\beta)^2(\gamma L)^2 \sum_{t=0}^{T-1}\mathbb{E}[\|g^{t-1}\|^2]+4 \eta^2 \beta^2 TD^2\\
    =&4 \eta^2 (1-\beta)^2 \sum_{t=0}^{T-1}\mathcal{E}_s^{t-1}+4 \eta^2 (1-\beta)^2(\gamma L)^2 \sum_{t=0}^{T-1}\mathbb{E}[\|g^{t-1}\|^2]+4 (1-2\beta+2\beta^2) \eta^2TD^2\\
	\leq& \frac{\beta^{2}}{288 e^2 K^{2} L^{2}} \sum_{t=0}^{T-1}\mathcal{E}_{s}^{t-1}+\frac{\beta^{2}}{288 e^2 K^{2} L^{2}}\sum_{t=0}^{T-1} \mathbb{E}[\|g^{t-1}\|^{2}]+4\eta^2 TD^2,
\end{align*}
where the last inequality follows from  $\max\{1152e^2(1-\beta)^2 \left(\eta KL\right)^2, 1152e^2(1-\beta)^2 \left(\eta \gamma KL^2\right)^2\} \leq \beta^{2}$.
\end{proof}

 In the following theorem, we establish the convergence rate of the proposed algorithm.

\begin{theorem}
\label{T1} 
Suppose that Assumptions \ref{A1}--\ref{A2} and \ref{A3} hold, and that $f_s(x)$ has a lower bound $f_s^{\rm min}$ for all $x\in \mathbb{R}^d$. Let $\Delta := \max_{s \in [S]} (f_s(x^0) - f_s^{\min})$. If we set
\begin{align}\label{3.4.1}
	& \beta =\left( \frac{M K L^2 \left(1+D^2\right)^2\Delta^2}{\sigma^4 T^2} \right)^{1/3}<1,  \quad \gamma=\min\left\{\frac{1}{2L},\frac{1}{L}\sqrt{\frac{\beta KM}{54}}\right\},  \nonumber \\
	&\quad \eta KL
	\le
	c_\eta
	\min\left\{
	\beta,\,
	\frac1{\sqrt K},\,
	\sqrt{\frac{\beta}{M}},\,
	\beta^{3/4}
	\left(
	\frac{L\Delta}{(KM)^{1/2}T}
	\right)^{1/2}
	\right\}, \nonumber
	\end{align}
where \(c_\eta>0\) is a sufficiently small absolute constant, and choose	$B=\left\lceil \frac{K}{T\beta^2}\right\rceil$, then Algorithm \ref{VSOMD} admits the convergence guarantee
\begin{align*}
\mathbb{E}\left[\mathcal{G}(x_a)\right]&
\lesssim\left(\frac{L\left(1+D^2\right)\sigma\Delta}{KMT}\right)^{\frac{2}{3}}(1+S^2)+\frac{LS^2\Delta}{T} \\
&\quad+
\frac{(1+S^2)G_0^2}{T} +
\frac{(1+S^2)\sigma^2}{KM}
\left(
\frac{MKL^2(1+D^2)^2\Delta^2}{\sigma^4T^2}
\right)^{2/3},
\end{align*}
where  $G_0^2:=\frac{1}{S}\sum_{s\in[S]}\left\|\nabla f_s(x^0)\right\|^2.$
\footnotetext[1]{The notation $\lesssim$ absorbs numerical constants.}
Alternatively, if $B=\Theta(KT)$ and $ \beta = \max\left\{\, \frac{1}{T},\; \left( \frac{M K L^2 \left(1+D^2\right)^2\Delta^2}{\sigma^4 T^2} \right)^{1/3} \right\}$, then Algorithm \ref{VSOMD} admits the convergence guarantee
\begin{align*}
  \mathbb{E}\left[\mathcal{G}(x_a)\right]&\lesssim \frac{\sigma^2(1+S^2)}{KMT}+\left(\frac{L\left(1+D^2\right)\sigma\Delta}{KMT}\right)^{\frac{2}{3}}(1+S^2)\\&\quad+ \frac{LS^2\Delta}{T}+\frac{(1+S^2)G_0^2}{T} + \frac{(1+S^2)\sigma^2}{KMT^2}.
\end{align*}

\end{theorem}

\begin{proof} 
	With the above choice of parameters, the conditions required in
	Lemmas \ref{L3}--\ref{L5} are satisfied. Indeed, by the definition of $\gamma$, we have
	\[
	\gamma\leq \frac{1}{2L},
	\qquad
	\gamma L\leq \sqrt{\frac{\beta KM}{54}}.
	\]
	Moreover, since $\eta KL\leq c_\eta\beta$, $\beta\leq c<1$, and $K\geq 1$, 
	by choosing $c_\eta>0$ sufficiently small, we obtain
	\[
	\eta KL\leq 1,
	\]
	and
	\begin{align*}
		24e^2(\eta L)^4K^4+24\eta^2L^2K
		&=
		24e^2(\eta KL)^4+24\frac{(\eta KL)^2}{K}  \\
		&\leq
		24e^2c_\eta^4+24c_\eta^2
		\leq \frac12 .
	\end{align*}
	Similarly,
	\[
	1152e^2(1-\beta)^2(\eta KL)^2
	\leq
	1152e^2c_\eta^2\beta^2
	\leq \beta^2.
	\]
	Since $\gamma L\leq 1/2$, we also have
	\begin{align*}
		1152e^2(1-\beta)^2(\eta\gamma KL^2)^2
		&=
		1152e^2(1-\beta)^2(\eta KL)^2(\gamma L)^2 \\
		&\leq
		1152e^2(1-\beta)^2(\eta KL)^2
		\leq \beta^2 .
	\end{align*}
	Therefore, all the conditions in Lemmas \ref{L3}--\ref{L5} hold.

It follows from Lemmas \ref{L3} and \ref{L4} that
\begin{align*}
	\mathcal{E}_{s}^t \leq &(1-\beta) \mathcal{E}_s^{t-1}+\frac{2 \beta }{9} \mathbb{E}[\|g^{t-1}\|^{2}]+\frac{3 \beta^{2} \sigma^{2}}{KM}+\frac{4}{\beta} L^{2}U_s^t \\
    \leq&(1-\beta) \mathcal{E}_s^{t-1}+\frac{2 \beta }{9} \mathbb{E}[\|g^{t-1}\|^{2}]+\frac{3 \beta^{2} \sigma^{2}}{KM} \\
	& +\frac{4}{\beta} L^{2}\left(4 e^2 K^2 \Xi_s^t+12(\eta K)^2\left(8(\eta K L)^2+K^{-1}\right)\left(\beta^2 \sigma^2+4(\gamma L)^2 \mathbb{E}[\|g^{t-1}\|^2]\right)\right) \\
    =&(1-\beta) \mathcal{E}_s^{t-1}+\frac{2 \beta }{9} \mathbb{E}[\|g^{t-1}\|^{2}]+\frac{3 \beta^{2} \sigma^{2}}{KM}+\frac{16( e L K)^2}{\beta} \Xi_{s}^t+48 \beta(\eta K L)^2\left(8(\eta K L)^2+K^{-1}\right) \sigma^2\\
    &+\frac{192}{\beta}(\eta K L)^2\left(8(\eta K L)^2+K^{-1}\right)(\gamma L)^2 \mathbb{E}[\|g^{t-1}\|^2]\\
	\leq & (1-\beta)\mathcal{E}_s^{t-1}+\frac{4\beta^2\sigma^2}{KM}+\frac{5 \beta }{18} \mathbb{E}[\|g^{t-1}\|^{2}]+\frac{16(eLK)^2}{\beta}\Xi_{s}^t,
\end{align*}
where the third inequality follows from
\begin{align*}
	\begin{cases}
		48\beta(\eta KL)^2\left(8(\eta KL)^2+K^{-1}\right)\leq\dfrac{\beta^2}{KM},\\[10pt]
		\dfrac{192(\eta KL)^2}{\beta}\left(8(\eta KL)^2+K^{-1}\right)(\gamma L)^2\leq\dfrac{\beta}{18}.
	\end{cases}
\end{align*}
Indeed, let $q:=\eta KL$. The stated stepsize condition gives
\begin{equation*}
q^2\leq c_\eta^2\frac{\beta}{M},\qquad
q^2\leq \frac{c_\eta^2}{K},\qquad
8q^2+K^{-1}\leq \frac{8c_\eta^2+1}{K}.
\end{equation*}
Consequently,
\begin{equation*}
	48\beta q^2(8q^2+K^{-1})
	\leq48c_\eta^2(8c_\eta^2+1)\frac{\beta^2}{KM}
	\leq\frac{\beta^2}{KM},
\end{equation*}
provided that $48c_\eta^2(8c_\eta^2+1)\leq1$. Moreover, the definition of
$\gamma$ gives $(\gamma L)^2\leq\beta KM/54$, so
\begin{align*}
\frac{192q^2}{\beta}(8q^2+K^{-1})(\gamma L)^2
&\leq\frac{32}{9}KMq^2(8q^2+K^{-1})\\
&\leq\frac{32}{9}c_\eta^2(8c_\eta^2+1)\beta
\leq\frac{\beta}{18},
\end{align*}
where the last inequality holds if
$c_\eta^2(8c_\eta^2+1)\leq1/64$. Thus, by choosing the absolute constant $c_\eta>0$ sufficiently small, both required inequalities hold.

Summing over $t$ from 0 to $T-1$ and applying Lemma \ref{L5}, we get
\begin{align*}
	\sum_{t=0}^{T-1} \mathcal{E}_s^{t}	\leq & (1-\beta)\sum_{t=0}^{T-1}\mathcal{E}_s^{t-1}+\frac{4\beta^2\sigma^2}{KM}T+\frac{5 \beta }{18}\sum_{t=0}^{T-1} \mathbb{E}[\|g^{t-1}\|^{2}]+\frac{16(eLK)^2}{\beta}\sum_{t=0}^{T-1}\Xi_{s}^t\\
	\leq & (1-\beta)\sum_{t=0}^{T-1}\mathcal{E}_s^{t-1}+\frac{4\beta^2\sigma^2}{KM}T+\frac{5 \beta }{18}\sum_{t=0}^{T-1} \mathbb{E}[\|g^{t-1}\|^{2}]\\
	&+\frac{16(eLK)^2}{\beta}\left(\frac{\beta^{2}}{288 e^2 K^{2} L^{2}} \sum_{t=0}^{T-1}\left(\mathcal{E}_{s}^{t-1}+\mathbb{E}[\|g^{t-1}\|^{2}]\right)+4 \eta^{2}TD^2\right)\\
	= & \left(1-\frac{17\beta}{18}\right)\sum_{t=-1}^{T-2}\mathcal{E}_s^{t}+\frac{\beta }{3}\sum_{t=0}^{T-1} \mathbb{E}[\|g^{t-1}\|^{2}]+\frac{4\beta^2\sigma^2}{KM}T+\frac{64}{\beta}(e\eta KL)^2TD^2.
\end{align*}
Thus,
\begin{align*}
    \frac{17\beta}{18}\sum_{t=-1}^{T-2}\mathcal{E}_s^{t} \leq \mathcal{E}_s^{-1}-\mathcal{E}_s^{T-1}+\frac{\beta }{3}\sum_{t=0}^{T-1} \mathbb{E}[\|g^{t-1}\|^{2}]+\frac{4\beta^2\sigma^2}{KM}T+\frac{64}{\beta}(e\eta KL)^2TD^2,
\end{align*}
which implies that
\begin{align*}
    \sum_{t=-1}^{T-2}\mathcal{E}_s^{t} \leq \frac{18}{17\beta}\mathcal{E}_s^{-1}-\frac{18}{17\beta}\mathcal{E}_s^{T-1}+\frac{6 }{17}\sum_{t=0}^{T-1} \mathbb{E}[\|g^{t-1}\|^{2}]+\frac{72\beta\sigma^2}{17KM}T+\frac{1152}{17\beta^2}(e\eta KL)^2TD^2.
\end{align*}
Therefore, 
\begin{align}\label{3.5}
	\sum_{t=-1}^{T-1}\mathcal{E}_s^{t}
	\leq&
	\frac{18}{17\beta}\mathcal{E}_s^{-1}
	+\frac{6}{17}\sum_{t=0}^{T-1} \mathbb{E}\!\left[\|g^{t-1}\|^{2}\right]
	+\frac{72\beta\sigma^2}{17KM}T
	+\frac{1152}{17\beta^2}(e\eta KL)^2TD^2 .
\end{align}
Notice that
\begin{equation*}
\sum_{t=0}^{T-1}\mathbb{E}\!\left[\|g^{t-1}\|^2\right]
=
\sum_{t=0}^{T-1}\mathbb{E}\!\left[\|g^t\|^2\right]
+
\mathbb{E}\!\left[\|g^{-1}\|^2\right]
-
\mathbb{E}\!\left[\|g^{T-1}\|^2\right].
\end{equation*}
Hence,
\begin{equation*}
\sum_{t=0}^{T-1}\mathbb{E}\!\left[\|g^{t-1}\|^2\right]
\leq
\sum_{t=0}^{T-1}\mathbb{E}\!\left[\|g^t\|^2\right]
+
\mathbb{E}\!\left[\|g^{-1}\|^2\right].
\end{equation*}
By Lemma~\ref{L1},
\begin{equation*}
\frac{2\left(\mathbb{E}[f_s(x^T)]-f_s(x^0)\right)}{\gamma}
+
\frac{1}{2}\sum_{t=0}^{T-1}\mathbb{E}\!\left[\|g^t\|^2\right]
\leq
\sum_{t=0}^{T-1}\mathcal{E}_s^t
\leq
\sum_{t=-1}^{T-1}\mathcal{E}_s^t .
\end{equation*}
Combining the above inequality with \eqref{3.5}, we get
\begin{align*}
	\frac{1}{2}\sum_{t=0}^{T-1}\mathbb{E}\!\left[\|g^t\|^2\right]
	\leq&
	\frac{2\left(f_s(x^0)-\mathbb{E}[f_s(x^T)]\right)}{\gamma}
	+\frac{18}{17\beta}\mathcal{E}_s^{-1}
	+\frac{6}{17}\sum_{t=0}^{T-1}\mathbb{E}\!\left[\|g^t\|^2\right]  \\
	&+\frac{6}{17}\mathbb{E}\!\left[\|g^{-1}\|^2\right]
	+\frac{72\beta\sigma^2}{17KM}T
	+\frac{1152}{17\beta^2}(e\eta KL)^2TD^2,
\end{align*}
which implies that 
\begin{align}\label{3.6}
	\sum_{t=0}^{T-1} \mathbb{E}\!\left[\|g^{t}\|^{2}\right]
	\leq&
	\frac{68\left(f_s(x^0)-\mathbb{E}[f_s(x^T)]\right)}{5\gamma}
	+\frac{36}{5\beta}\mathcal{E}_s^{-1}
	+\frac{144\beta\sigma^2}{5KM}T \nonumber\\
	&+\frac{2304}{5\beta^2}(e\eta KL)^2TD^2
	+\frac{12}{5}\mathbb{E}\!\left[\|g^{-1}\|^2\right] \nonumber\\
	\leq&
\frac{68\Delta}{5\gamma}
	+\frac{36}{5\beta}\mathcal{E}_s^{-1}
	+\frac{144\beta\sigma^2}{5KM}T \nonumber\\
	&+\frac{2304}{5\beta^2}(e\eta KL)^2TD^2
	+\frac{12}{5}\mathbb{E}\!\left[\|g^{-1}\|^2\right].
\end{align}
Averaging \eqref{3.6} over $s\in[S]$, we obtain
\begin{align}\label{ave:3.6}
	\sum_{t=0}^{T-1}\mathbb{E}\left[\|g^t\|^2\right]
	\leq&
	\frac{68\Delta}{5\gamma}
	+\frac{36}{5\beta S}\sum_{s\in[S]}\mathcal{E}_s^{-1}
	+\frac{144\beta\sigma^2}{5KM}T
	\nonumber\\
	&+\frac{2304}{5\beta^2}(e\eta KL)^2TD^2
	+\frac{12}{5}\mathbb{E}\left[\|g^{-1}\|^2\right].
\end{align}
Moreover, from \eqref{3.5}, we have
\begin{align*}
	\sum_{t=0}^{T-1} \mathcal{E}_s^{t}
	\leq&
	\frac{18}{17 \beta} \mathcal{E}_s^{-1}
	+\frac{6}{17}\sum_{t=0}^{T-1}\mathbb{E}\!\left[\|g^{t-1}\|^2\right]
	+\frac{72 \beta \sigma^{2}}{17KM} T
	+\frac{1152}{17\beta^2}(e \eta K L)^{2}TD^2  \\
	\leq&
	\frac{18}{17 \beta} \mathcal{E}_s^{-1}
	+\frac{6}{17}\sum_{t=0}^{T-1}\mathbb{E}\!\left[\|g^{t}\|^2\right]
	+\frac{6}{17}\mathbb{E}\!\left[\|g^{-1}\|^2\right]  \\
	&+\frac{72 \beta \sigma^{2}}{17KM} T
	+\frac{1152}{17\beta^2}(e \eta K L)^{2}TD^2 .
\end{align*}
Substituting \eqref{3.6} into the above inequality gives
\begin{align}\label{3.7}
	\sum_{t=0}^{T-1} \mathcal{E}_s^{t}
	\leq&
	\frac{306}{85\beta}\mathcal{E}_s^{-1}
	+\frac{24\Delta}{5\gamma}
	+\frac{72\beta\sigma^2}{5KM}T
	+\frac{1152}{5\beta^2}(e \eta K L)^{2}TD^2 \nonumber\\
	&+\frac{102}{85}\mathbb{E}\!\left[\|g^{-1}\|^2\right].
\end{align}
Summing \eqref{3.7} over $s\in[S]$ gives
\begin{align}\label{ave:3.7}
	\sum_{s\in[S]}\sum_{t=0}^{T-1}\mathcal{E}_s^t
	\leq&
	\frac{306}{85\beta}
	\sum_{s\in[S]}\mathcal{E}_s^{-1}
	+\frac{24S\Delta}{5\gamma}
	+\frac{72S\beta\sigma^2}{5KM}T
		\nonumber\\
	&+
	\frac{1152S}{5\beta^2}(e\eta KL)^2TD^2
	+\frac{102S}{85}
	\mathbb{E}\left[\|g^{-1}\|^2\right].
\end{align}
Hence,
\begin{align*} 
	\frac{1}{T} \sum_{t=0}^{T-1} 
	\mathbb{E}\left[
	\left\|\sum_{s \in[S]} \lambda_s^{t,*}\nabla f_{s}(x^{t})\right\|^{2}
	\right]
	\leq&
	\frac{2}{T} \sum_{t=0}^{T-1}
	\mathbb{E}\left[
	\left\|\sum_{s \in[S]}\lambda_s^{t,*} \nabla f_{s}(x^{t})-g^t\right\|^{2}
	\right]
	+\frac{2}{T} \sum_{t=0}^{T-1} \mathbb{E}\left[\|g^t\|^{2}\right]  
	\\
	=&
	\frac{2}{T} \sum_{t=0}^{T-1}
	\mathbb{E}\left[
	\left\|
	\sum_{s \in[S]} \lambda_s^{t,*}
	\left(\nabla f_{s}(x^{t})-g_s^t\right)
	\right\|^{2}
	\right]
	+\frac{2}{T} \sum_{t=0}^{T-1} \mathbb{E}\left[\|g^t\|^{2}\right]
	\\
	\leq&
	\frac{2}{T} \sum_{t=0}^{T-1}
	\mathbb{E}\left[
	\sum_{s \in[S]} \lambda_s^{t,*}
	\left\|\nabla f_s(x^t)-g_s^t\right\|^2
	\right]
	+\frac{2}{T} \sum_{t=0}^{T-1} \mathbb{E}\left[\|g^t\|^{2}\right]
	\\
	\leq&
	\frac{2}{T} \sum_{t=0}^{T-1}
	\sum_{s \in[S]}
	\mathbb{E}\left[
	\left\|\nabla f_s(x^t)-g_s^t\right\|^2
	\right]
	+\frac{2}{T} \sum_{t=0}^{T-1} \mathbb{E}\left[\|g^t\|^{2}\right]\\
		= & \frac{2}{T}  \sum_{s \in[S]} \sum_{t=0}^{T-1} \mathcal{E}_s^t+\frac{2}{T} \sum_{t=0}^{T-1} \mathbb{E}[\|g^t\|^{2}],
\end{align*}
where the second inequality follows from the convexity of $\|\cdot\|^2$. Plugging \eqref{ave:3.6} and \eqref{ave:3.7} into the above inequality, we have
\begin{align*}
	&\frac{1}{T} \sum_{t=0}^{T-1} \mathbb{E}\left[\left\|\sum_{s \in[S]} \lambda_s^{t,*} \nabla f_{s}(x^{t})\right\|^{2}\right]  \\
	\leq  &\frac{612}{85\beta T}\sum_{s \in[S]}\mathcal{E}_s^{-1}
	+\frac{48S\Delta}{5\gamma T}
	+\frac{144S\beta\sigma^2}{5KM} \\
	&+\frac{2304S}{5\beta^2}(e\eta KLD)^2
	+\frac{204S}{85T}\mathbb{E}[\|g^{-1}\|^2]
	+\frac{72}{5\beta S T}\sum_{s \in[S]}\mathcal{E}_s^{-1}  \\
	&+\frac{136\Delta}{5\gamma T}
	+\frac{288\beta\sigma^2}{5KM}
	+\frac{4608}{5\beta^2}(e\eta KLD)^2
	+\frac{24}{5T}\mathbb{E}[\|g^{-1}\|^2].
\end{align*}
Finally, for $B=\left\lceil\frac{K}{T\beta^2}\right\rceil$, noticing that
\begin{equation*}
	g_s^{-1}=\frac{1}{BM} \sum_{i=1}^{M} \sum_{b=1}^{B} \nabla F_{s,i}(x^{-1} ; \xi_{i}^{b}),
\end{equation*}
we have
\begin{align*}
	\mathcal{E}_s^{-1}
	=&\mathbb{E}[\|\nabla f_s(x^{-1})-g_s^{-1}\|^2]\\
	=&\mathbb{E}\left[\left\|\frac{1}{BM}\sum_{i,b} \nabla f_{s,i}(x^{-1})-\frac{1}{BM}\sum_{i,b} \nabla F_{s,i}(x^{-1};\xi_i^b)\right\|^2\right]\\
	=&\frac{1}{B^2M^2}\mathbb{E}\left[\left\|\sum_{i,b} \left(\nabla f_{s,i}(x^{-1})-\nabla F_{s,i}(x^{-1};\xi_i^b)\right)\right\|^2\right]\\
	=&\frac{1}{B^2M^2}\mathbb{E}\left[\sum_{i,b}\left\| \nabla f_{s,i}(x^{-1})-\nabla F_{s,i}(x^{-1};\xi_i^b)\right\|^2 \right.\\
	& \left.+2\sum_{(i,b)\ne(j,d)} \left\langle\nabla f_{s,i}(x^{-1})-\nabla F_{s,i}(x^{-1};\xi_i^b),\nabla f_{s,j}(x^{-1})-\nabla F_{s,j}(x^{-1};\xi_j^d)\right\rangle \right]\\
	\leq&\frac{1}{B^2M^2}\mathbb{E}\left[\sum_{i,b}\sigma^2\right]
	=\frac{\sigma^2}{BM}
	\leq \frac{\beta^2\sigma^2T}{KM}.
\end{align*}
Moreover, since $g^{-1}$ is obtained by applying the same Pareto aggregation rule to $\{g_s^{-1}\}_{s\in[S]}$, it holds that
\begin{equation*}
g^{-1}=\sum_{s\in[S]}\lambda_s^{-1,*}g_s^{-1},
\end{equation*}
where $\lambda^{-1,*}\in[0,1]^S$ solves
\begin{equation*}
\min_{\substack{\lambda_s\geq 0,\\ \sum_{s\in[S]}\lambda_s=1}}
\left\lVert
\sum_{s\in[S]}\lambda_s d_s^{-1}
\right\rVert^2.
\end{equation*}
By the optimality of $\lambda^{-1,*}$, choosing the uniform weight $\lambda_s=1/S$ gives
\begin{align*}
	\mathbb{E}\|g^{-1}\|^2
	&\leq
	\mathbb{E}\left[
	\left\|
	\frac{1}{S}\sum_{s\in[S]}g_s^{-1}
	\right\|^2
	\right]  \\
	&\leq
	\frac{1}{S}\sum_{s\in[S]}
	\mathbb{E}\left[\|g_s^{-1}\|^2\right].
\end{align*}
For each $s\in[S]$, we have
\begin{align*}
	\mathbb{E}\left[\|g_s^{-1}\|^2\right]
	&=
	\mathbb{E}\left[
	\left\|
	\nabla f_s(x^{-1})+
	g_s^{-1}-\nabla f_s(x^{-1})
	\right\|^2
	\right]  \\
	&\leq
	2\left\|\nabla f_s(x^{-1})\right\|^2
	+
	2\mathbb{E}\left[
	\left\|
	g_s^{-1}-\nabla f_s(x^{-1})
	\right\|^2
	\right]  \\
	&=
	2\left\|\nabla f_s(x^{-1})\right\|^2
	+
	2\mathcal{E}_s^{-1}.
\end{align*}
Since $x^{-1}=x^0$, let
\begin{equation*}
G_0^2:=\frac{1}{S}\sum_{s\in[S]}\left\|\nabla f_s(x^0)\right\|^2.
\end{equation*}
Then
\begin{align*}
	\mathbb{E}\|g^{-1}\|^2\leq
	2G_0^2+\frac{2}{S}\sum_{s\in[S]}\mathcal{E}_s^{-1}\leq
	2G_0^2+\frac{2\beta^2\sigma^2T}{KM}.
\end{align*}
Thus,
\begin{align}\label{equ:4.7}
	\frac{1}{T} \sum_{t=0}^{T-1} \mathbb{E}\left[\left\|\sum_{s \in[S]} \lambda_s^{t,*} \nabla f_{s}(x^{t})\right\|^{2}\right] 
	\lesssim  &\frac{1}{\beta T}\sum_{s \in[S]}\mathcal{E}_s^{-1}+\frac{S\Delta}{\gamma T} +\frac{S\beta\sigma^2}{KM}+\frac{S}{\beta^2}(\eta KLD)^2\nonumber\\
	&+\frac{1}{\beta TS}\sum_{s \in[S]}\mathcal{E}_s^{-1}+\frac{\Delta}{\gamma T} +\frac{\beta\sigma^2}{KM}+\frac{1}{\beta^2}(\eta KLD)^2  \nonumber\\
	&+\frac{1+S}{T}\mathbb{E}[\|g^{-1}\|^2]\nonumber\\
	\lesssim &\frac{\beta\sigma^2}{KM}(1+S)+\frac{(\eta KL)^2}{\beta^2}D^2(1+S)+\frac{S\Delta}{\gamma T} \nonumber\\
	&+\frac{(1+S)G_0^2}{T}
	+\frac{(1+S)\beta^2\sigma^2}{KM}.
\end{align}
By the choice of $\eta$,
\begin{equation*}
	\frac{(\eta KL)^2}{\beta^2}
	\lesssim
	\frac{L\Delta}{\beta^{1/2}(KM)^{1/2}T}.
\end{equation*}
Moreover,
\begin{equation*}
	\frac1\gamma
	\le
	2L+\sqrt{54}\frac{L}{\sqrt{\beta KM}},
\end{equation*}
which implies that
\begin{equation*}
	\frac{S\Delta}{\gamma T}
	\lesssim
	\frac{LS\Delta}{T}
	+
	\frac{LS\Delta}{\beta^{1/2}(KM)^{1/2}T}.
\end{equation*}
Therefore, from \eqref{equ:4.7}, we have
\begin{equation*}
	\frac1T\sum_{t=0}^{T-1}
	\mathbb E\left[
	\left\|
	\sum_{s\in[S]}\lambda_s^{t,*}\nabla f_s(x^t)
	\right\|^2
	\right]
	\lesssim
	\frac{\beta\sigma^2}{KM}(1+S)
	+
	\frac{L(1+D^2)\Delta}{\beta^{1/2}(KM)^{1/2}T}(1+S)
	+
	\frac{LS\Delta}{T}
	+
	\frac{(1+S)G_0^2}{T}
	+
	\frac{(1+S)\beta^2\sigma^2}{KM}.
\end{equation*}
Taking
\begin{equation*}
	\beta=
	\left(
	\frac{MKL^2(1+D^2)^2\Delta^2}{\sigma^4T^2}
	\right)^{1/3}
\end{equation*}
gives
\begin{align*}
	&\frac1T\sum_{t=0}^{T-1}
	\mathbb E\left[
	\left\|
	\sum_{s\in[S]}\lambda_s^{t,*}\nabla f_s(x^t)
	\right\|^2
	\right] \\
	&\lesssim
	\left(
	\frac{L(1+D^2)\sigma\Delta}{KMT}
	\right)^{2/3}(1+S)
	+
	\frac{LS\Delta}{T}
	+
	\frac{(1+S)G_0^2}{T}  \\
	&\quad+
	\frac{(1+S)\sigma^2}{KM}
	\left(
\frac{MKL^2(1+D^2)^2\Delta^2}{\sigma^4T^2}
	\right)^{2/3}.
\end{align*}
Since $\lambda^{t,*}\in\mathcal{C}$, by the definition of $\mathcal{G}(x^t)$,
\begin{equation*}
	\mathcal{G}(x^t)
	=
	\min_{\lambda\in\mathcal{C}}
	\left\|
	\sum_{s\in[S]}\lambda_s\nabla f_s(x^t)
	\right\|^2
	\leq
	\left\|
	\sum_{s\in[S]}\lambda_s^{t,*}\nabla f_s(x^t)
	\right\|^2.
\end{equation*}
Taking expectations and averaging over $t=0,\dots,T-1$, we obtain
\begin{align*}
	\frac{1}{T}\sum_{t=0}^{T-1}
	\mathbb{E}\left[\mathcal{G}(x^t)\right]
	\leq&
	\frac{1}{T}\sum_{t=0}^{T-1}
	\mathbb{E}\left[
	\left\|
	\sum_{s\in[S]}\lambda_s^{t,*}\nabla f_s(x^t)
	\right\|^2
	\right].
\end{align*}
Since $a$ is chosen uniformly at random from $\{0,1,\dots,T-1\}$,
\begin{align*}
	\mathbb{E}\left[\mathcal{G}(x_a)\right]
	=&\frac{1}{T}\sum_{t=0}^{T-1}
	\mathbb{E}\left[\mathcal{G}(x^t)\right] \\
	\leq&
	\frac{1}{T}\sum_{t=0}^{T-1}
	\mathbb{E}\left[
	\left\|
	\sum_{s\in[S]}\lambda_s^{t,*}\nabla f_s(x^t)
	\right\|^2
	\right]\\
	\lesssim&
	\left(
	\frac{L(1+D^2)\sigma\Delta}{KMT}
	\right)^{\frac{2}{3}}(1+S)
	+\frac{LS\Delta}{T}
	+\frac{(1+S)G_0^2}{T}
	\\
	&+
	\frac{(1+S)\sigma^2}{KM}
	\left(
	\frac{MKL^2(1+D^2)^2\Delta^2}
	{\sigma^4T^2}
	\right)^{\frac{2}{3}}.	
\end{align*}

Similarly, if $B=\Theta(KT)$, then
\begin{equation*}
	\mathcal{E}_s^{-1}\leq \frac{\sigma^2}{BM}\lesssim \frac{\sigma^2}{KMT}.
\end{equation*}
Moreover, by the same argument as above,
\begin{equation*}
	\mathbb{E}\|d^{-1}\|^2
	\lesssim
	G_0^2+\frac{\sigma^2}{KMT}.
\end{equation*}
Using the preceding estimate, we have
\begin{align*}
	&\frac{1}{T}\sum_{t=0}^{T-1}
	\mathbb{E}\left[
	\left\|
	\sum_{s\in[S]}\lambda_s^{t,*}\nabla f_s(x^t)
	\right\|^2
	\right] \\
	&\lesssim
	\frac{\sigma^2}{KM\beta T^2}(1+S)
	+
	\frac{\beta\sigma^2}{KM}(1+S)
	+
	\frac{(\eta KL)^2}{\beta^2}D^2(1+S)\\
	&\quad+
\frac{S\Delta}{\gamma T}
	+
	\frac{(1+S)G_0^2}{T}
	+
	\frac{(1+S)\sigma^2}{KMT^2}.
\end{align*}
By the choice of $\eta$,
\begin{equation*}
	\frac{(\eta KL)^2}{\beta^2}
	\lesssim
\frac{L\Delta}{\beta^{1/2}(KM)^{1/2}T}.
\end{equation*}
Moreover, since
\begin{equation*}
	\gamma=\min\left\{
	\frac1{2L},
	\frac1L\sqrt{\frac{\beta KM}{54}}
	\right\},
\end{equation*}
it follows that
\begin{equation*}
	\frac{1}{\gamma}
	\le
	2L+\sqrt{54}\frac{L}{\sqrt{\beta KM}}.
\end{equation*}
Thus,
\begin{equation*}
	\frac{S\Delta}{\gamma T}
	\lesssim
	\frac{LS\Delta}{T}
	+
\frac{LS\Delta}{\beta^{1/2}(KM)^{1/2}T}.
\end{equation*}
Combining the above estimates gives
\begin{align*}
\frac{1}{T}\sum_{t=0}^{T-1}
\mathbb{E}\left[
\left\|
\sum_{s\in[S]}\lambda_s^{t,*}\nabla f_s(x^t)
\right\|^2
\right] 
	&\lesssim
	\frac{\sigma^2}{KM\beta T^2}(1+S)
	+
	\frac{\beta\sigma^2}{KM}(1+S) \\
	&\quad+
	\frac{L(1+D^2)\Delta}{\beta^{1/2}(KM)^{1/2}T}(1+S)
	+
\frac{LS\Delta}{T}
	+
	\frac{(1+S)G_0^2}{T}.
\end{align*}
Now take
\begin{equation*}
	\beta=
	\max\left\{
\frac{1}{T},
	\left(
	\frac{MKL^2(1+D^2)^2\Delta^2}{\sigma^4T^2}
	\right)^{1/3}
	\right\}.
\end{equation*}
Since $\beta\geq 1/T$, we have
\begin{equation*}
	\frac{\sigma^2}{KM\beta T^2}
	\leq
	\frac{\sigma^2}{KMT}.
\end{equation*}
Furthermore, by the definition of $\beta$,
\begin{equation*}
	\frac{\beta\sigma^2}{KM}
	+
	\frac{L(1+D^2)\Delta}{\beta^{1/2}(KM)^{1/2}T}
	\lesssim
	\frac{\sigma^2}{KMT}
	+
	\left(
	\frac{L(1+D^2)\sigma\Delta}{KMT}
	\right)^{2/3}.
\end{equation*}
Therefore,
\begin{align*}
\frac{1}{T}\sum_{t=0}^{T-1}
\mathbb{E}\left[
\left\|
\sum_{s\in[S]}\lambda_s^{t,*}\nabla f_s(x^t)
\right\|^2
\right] 
	&\lesssim
	\frac{\sigma^2(1+S)}{KMT}
	+
	\left(
	\frac{L(1+D^2)\sigma\Delta}{KMT}
	\right)^{2/3}
	(1+S) \\
	&\quad+
	\frac{LS\Delta}{T}
	+
	\frac{(1+S)G_0^2}{T}
	+
	\frac{(1+S)\sigma^2}{KMT^2},
\end{align*}
which implies that 
\begin{align*}
\mathbb{E}\left[\mathcal{G}(x_a)\right]
	&\lesssim
	\frac{\sigma^2(1+S)}{KMT}
	+
	\left(
	\frac{L(1+D^2)\sigma\Delta}{KMT}
	\right)^{2/3}
	(1+S) \\
	&\quad+
	\frac{LS\Delta}{T}
	+
	\frac{(1+S)G_0^2}{T}
	+
	\frac{(1+S)\sigma^2}{KMT^2}.
\end{align*}
The proof is complete.
\end{proof}

\begin{remark} Theorem \ref{T1} establishes a random-iterate Pareto stationarity guarantee for Algorithm \ref{VSOMD}. Specifically, the algorithm returns an iterate $x_a$, where $a$ is chosen uniformly at random from $\{0,\dots,T-1\}$, and the corresponding expected Pareto stationarity measure $\mathbb{E}[\mathcal{G}(x_a)]$ decreases at a rate of $\mathcal{O}(T^{-2/3})$. This improves upon the $\mathcal{O}(T^{-1/2})$ rate established in \cite{Yang2024,Askin2024}. 
\end{remark}

\begin{corollary} Under the conditions of
	Theorem~\ref{T1}, FSMGDA-M-VR attains an $\epsilon$-Pareto stationary
	solution with communication complexity
\begin{equation*}
	\mathcal{O}\left(
	\max\left\{
	\frac{L(1+D^2)\sigma\Delta(1+S)^{\frac{3}{2}}}
	{KM\epsilon^{\frac{3}{2}}},
	\frac{LS\Delta}{\epsilon},
	\frac{(1+S)G_0^2}{\epsilon},
	\left(
	\frac{C_1}{\epsilon}
	\right)^{\frac{3}{4}},
	\frac{\sigma^2(1+S)}{KM\epsilon},
	\left(
	\frac{(1+S)\sigma^2}{KM\epsilon}
	\right)^{\frac{1}{2}}
	\right\}
	\right),
\end{equation*}
	where
\begin{equation*}
	C_1:=
	\frac{(1+S)\sigma^2}{KM}
	\left(
	\frac{MKL^2(1+D^2)^2\Delta^2}{\sigma^4}
	\right)^{2/3}.
\end{equation*}
Moreover, the corresponding sample complexity is
\begin{equation*}
	\mathcal{O}\left(SMK\max\left\{
	\frac{L(1+D^2)\sigma\Delta(1+S)^{\frac{3}{2}}}
	{KM\epsilon^{\frac{3}{2}}},
	\frac{LS\Delta}{\epsilon},
	\frac{(1+S)G_0^2}{\epsilon},
	\left(
	\frac{C_1}{\epsilon}
	\right)^{\frac{3}{4}},
	\frac{\sigma^2(1+S)}{KM\epsilon},
	\left(
	\frac{(1+S)\sigma^2}{KM\epsilon}
	\right)^{\frac{1}{2}}
	\right\}
	\right).
\end{equation*}
\end{corollary}

\begin{proof}
	For the first case $B=\left\lceil \frac{K}{T\beta^2}\right\rceil$,  by Theorem~\ref{T1}, 
\begin{align*}
	\mathbb{E}\left[\mathcal{G}(x_a)\right]	\lesssim&\left(\frac{L\left(1+D^2\right)\sigma\Delta}{KMT}\right)^{\frac{2}{3}}(1+S)+\frac{LS\Delta}{T} \\
	+&
	\frac{(1+S)G_0^2}{T} +
	\frac{(1+S)\sigma^2}{KM}
	\left(
	\frac{MKL^2(1+D^2)^2\Delta^2}{\sigma^4T^2}
	\right)^{2/3}.
\end{align*}
To obtain an $\epsilon$-Pareto stationary solution, it is sufficient, up to absolute
constant factors hidden in $\lesssim$, to require 
	\begin{equation*}
		\left(
		\frac{L(1+D^2)\sigma\Delta}{KMT}
		\right)^{\frac{2}{3}}(1+S)
		\leq \frac{\epsilon}{4},
	\end{equation*}
	\begin{equation*}
		\frac{LS\Delta}{T}\leq \frac{\epsilon}{4},
		\qquad
		\frac{(1+S)G_0^2}{T}\leq \frac{\epsilon}{4},
	\end{equation*}
	and
	\begin{equation*}
		\frac{(1+S)\sigma^2}{KM}
		\left(
		\frac{MKL^2(1+D^2)^2\Delta^2}{\sigma^4T^2}
		\right)^{\frac{2}{3}}
		\leq \frac{\epsilon}{4}.
	\end{equation*}
	The first inequality yields
	\begin{equation*}
		T\geq\frac{L(1+D^2)\sigma\Delta}{KM}
		\left(\frac{4(1+S)}{\epsilon}
		\right)^{\frac{3}{2}},
	\end{equation*}
	and hence
\begin{equation*}
T=\mathcal{O}\left(
		\frac{L(1+D^2)\sigma\Delta(1+S)^{\frac{3}{2}}}
		{KM\epsilon^{\frac{3}{2}}}
		\right).
\end{equation*}
The second and third inequalities give
\begin{equation*}
		T=\mathcal{O}\left(
		\frac{LS\Delta}{\epsilon}
		\right),
		\qquad
		T
		=
		\mathcal{O}\left(
		\frac{(1+S)G_0^2}{\epsilon}
		\right).
\end{equation*}
For the last inequality, let $C_1:=
		\frac{(1+S)\sigma^2}{KM}
		\left(
		\frac{MKL^2(1+D^2)^2\Delta^2}{\sigma^4}
		\right)^{\frac{2}{3}}$. Thus, it is sufficient to require
	\begin{equation*}
		\frac{C_1}{T^{\frac{4}{3}}}
		\leq
		\frac{\epsilon}{4},
	\end{equation*}
	which yields
	\begin{equation*}
		T=
		\mathcal{O}\left(
		\left(
		\frac{C_1}{\epsilon}
		\right)^{\frac{3}{4}}
		\right).
	\end{equation*}
Therefore, for the first choice of $B=\left\lceil \frac{K}{T\beta^2}\right\rceil$, it suffices to take
	\begin{equation*}
		T
		=
		\mathcal{O}\left(
		\max\left\{
		\frac{L(1+D^2)\sigma\Delta(1+S)^{\frac{3}{2}}}
		{KM\epsilon^{\frac{3}{2}}},
		\frac{LS\Delta}{\epsilon},
		\frac{(1+S)G_0^2}{\epsilon},
		\left(
		\frac{C_1}{\epsilon}
		\right)^{\frac{3}{4}}
		\right\}
		\right).
	\end{equation*}
	
	Next, for the second case $B=\Theta(KT)$, Theorem~\ref{T1} gives
\begin{equation*}
\mathbb{E}\left[\mathcal{G}(x_a)\right]
		\lesssim
		\frac{\sigma^2(1+S)}{KMT}
		+
		\left(
		\frac{L(1+D^2)\sigma\Delta}{KMT}
		\right)^{\frac{2}{3}}(1+S)
		+\frac{LS\Delta}{T}
		+\frac{(1+S)G_0^2}{T}
		+\frac{(1+S)\sigma^2}{KMT^2}.
\end{equation*}
To obtain an $\epsilon$-Pareto stationary solution, it is sufficient, up to absolute
constant factors hidden in $\lesssim$, to require
\begin{equation*}
		\frac{\sigma^2(1+S)}{KMT}
		\leq
		\frac{\epsilon}{5},
\end{equation*}
\begin{equation*}
		\left(
		\frac{L(1+D^2)\sigma\Delta}{KMT}
		\right)^{\frac{2}{3}}(1+S)
		\leq
		\frac{\epsilon}{5},
\end{equation*}
\begin{equation*}
		\frac{LS\Delta}{T}
		\leq
		\frac{\epsilon}{5},
		\qquad
		\frac{(1+S)G_0^2}{T}
		\leq
		\frac{\epsilon}{5},
\end{equation*}
	and
\begin{equation*}
		\frac{(1+S)\sigma^2}{KMT^2}
		\leq
		\frac{\epsilon}{5}.
\end{equation*}
These inequalities respectively yield
\begin{equation*}
		T=\mathcal{O}\left(
		\frac{\sigma^2(1+S)}{KM\epsilon}
		\right),
\end{equation*}
\begin{equation*}
		T=\mathcal{O}\left(
		\frac{L(1+D^2)\sigma\Delta(1+S)^{\frac{3}{2}}}
		{KM\epsilon^{\frac{3}{2}}}
		\right),
\end{equation*}
\begin{equation*}
		T
		=
		\mathcal{O}\left(
		\frac{LS\Delta}{\epsilon}
		\right),
		\qquad
		T
		=
		\mathcal{O}\left(
		\frac{(1+S)G_0^2}{\epsilon}
		\right),
\end{equation*}
and
\begin{equation*}
		T
		=
		\mathcal{O}\left(
		\left(
		\frac{(1+S)\sigma^2}{KM\epsilon}
		\right)^{\frac{1}{2}}
		\right).
	\end{equation*}
Therefore, for the second choice $B=\Theta(KT)$, it suffices to take
\begin{equation*}
		T
		=
		\mathcal{O}\left(
		\max\left\{
		\frac{L(1+D^2)\sigma\Delta(1+S)^{\frac{3}{2}}}
		{KM\epsilon^{\frac{3}{2}}},
		\frac{LS\Delta}{\epsilon},
		\frac{(1+S)G_0^2}{\epsilon},
		\frac{\sigma^2(1+S)}{KM\epsilon},
		\left(
		\frac{(1+S)\sigma^2}{KM\epsilon}
		\right)^{\frac{1}{2}}
		\right\}
		\right).
\end{equation*}
	
	We next derive the communication complexity and the sample complexity.
	
\item[$\bullet$] \textbf{Communication complexity:} In FSMGDA-M-VR, each global iteration corresponds to one round of communication between the clients and the server. Therefore, the total communication complexity is equal to the total number of global rounds $T$. Combining the above two cases, the communication complexity can be expressed as
\begin{equation*}
\mathcal{O}\left(
		\max\left\{
		\frac{L(1+D^2)\sigma\Delta(1+S)^{\frac{3}{2}}}
		{KM\epsilon^{\frac{3}{2}}},
		\frac{LS\Delta}{\epsilon},
		\frac{(1+S)G_0^2}{\epsilon},
		\left(
		\frac{C_1}{\epsilon}
		\right)^{\frac{3}{4}},
		\frac{\sigma^2(1+S)}{KM\epsilon},
		\left(
		\frac{(1+S)\sigma^2}{KM\epsilon}
		\right)^{\frac{1}{2}}
		\right\}
		\right).
\end{equation*}
	
\item[$\bullet$] \textbf{Sample Complexity:} Next, we estimate the total number of stochastic gradient computations. At the initialization stage, the algorithm requires $SMB$ stochastic gradient computations. At each local update, the algorithm computes stochastic gradients for all $S$ objectives, and both the current stochastic gradient and the gradient at the previous global model are involved in the momentum estimator. Therefore, each communication round requires $2SMK$ stochastic gradient computations. Hence, the total number of stochastic gradient computations is
	\begin{equation*}
		SMB+2SMKT.
	\end{equation*}
It remains to compare the initialization cost $SMB$ with the training-stage cost $2SMKT$. We consider the following two cases.
\item[(1)] $\bm{B=\left\lceil \frac{K}{T\beta^2}\right\rceil}$. In this case,
\begin{equation*}
		\frac{SMB}{2SMKT}=\frac{B}{2KT}.
\end{equation*}	
Since
\begin{equation*}
B=\left\lceil \frac{K}{T\beta^2}\right\rceil
\le\frac{K}{T\beta^2}+1,
\end{equation*}	
we have
\begin{equation*}
\frac{B}{2KT}\le \frac{1}{2T^2\beta^2}+\frac{1}{2KT}.
\end{equation*}	
Recall that
\begin{equation*}
\beta=\left(
\frac{MKL^2(1+D^2)^2\Delta^2}{\sigma^4T^2}\right)^{1/3},
\end{equation*}	
then
\begin{equation*}
\frac{1}{T^2\beta^2}=\frac{\sigma^{8/3}}
{M^{2/3}K^{2/3}L^{4/3}(1+D^2)^{4/3}\Delta^{4/3}}\cdot	\frac{1}{T^{2/3}}\to 0\qquad (T\to\infty).
\end{equation*}	
Therefore,
\begin{equation*}
		\frac{SMB}{2SMKT}\to 0.
\end{equation*}	
Thus, the initialization cost $SMB$ is of lower order than the training-stage cost $2SMKT$, and hence
\begin{equation*}
		SMB+2SMKT=\mathcal{O}(SMKT).
\end{equation*}	
	
\item[(2)] $\bm{B=\Theta(KT)}$. In this case, we have $SMB=\Theta(SMKT)$,
and hence
\begin{equation*}
SMB+2SMKT=\mathcal{O}(SMKT),
\end{equation*}
which implies that the total number of stochastic gradient computations is also of order $\mathcal{O}(SMKT)$.

Combining the above two cases, for either choice of $B$, the total number of stochastic gradient computations is of order $\mathcal{O}(SMKT)$. Multiplying the above communication complexity by $SMK$, we obtain the sample complexity
\begin{equation*}
\mathcal{O}\left(SMK\max\left\{
		\frac{L(1+D^2)\sigma\Delta(1+S)^{\frac{3}{2}}}
		{KM\epsilon^{\frac{3}{2}}},
		\frac{LS\Delta}{\epsilon},
		\frac{(1+S)G_0^2}{\epsilon},
		\left(
		\frac{C_1}{\epsilon}
		\right)^{\frac{3}{4}},
		\frac{\sigma^2(1+S)}{KM\epsilon},
		\left(
		\frac{(1+S)\sigma^2}{KM\epsilon}
		\right)^{\frac{1}{2}}
		\right\}
		\right),
\end{equation*}
which completes the proof.
\end{proof}

\begin{remark}
Compared with FSMGDA and FedCMOO, the proposed FSMGDA-M-VR improves the dependence on the target accuracy in terms of both communication-round complexity and the total number of stochastic gradient computations. Specifically, to obtain an $\epsilon$-Pareto stationary solution in expectation, FSMGDA-M-VR requires
\begin{equation*}
T=\mathcal{O}\left(
	\max\left\{
	\frac{L(1+D^2)\sigma\Delta(1+S)^{\frac{3}{2}}}
	{KM\epsilon^{\frac{3}{2}}},
	\frac{LS\Delta}{\epsilon},
	\frac{(1+S)G_0^2}{\epsilon},
	\left(
	\frac{C_1}{\epsilon}
	\right)^{\frac{3}{4}},
	\frac{\sigma^2(1+S)}{KM\epsilon},
	\left(
	\frac{(1+S)\sigma^2}{KM\epsilon}
	\right)^{\frac{1}{2}}
	\right\}
	\right)
\end{equation*}
communication rounds. Hence, for fixed problem-dependent constants and in the high-accuracy regime where the first term dominates, its communication-round complexity is of order $\mathcal{O}(\epsilon^{-3/2})$, improving upon the $\mathcal{O}(\epsilon^{-2})$ communication-round complexity of FSMGDA and FedCMOO.
\end{remark}

\section{Numerical experiments}

In this section, we evaluate Algorithm \ref{VSOMD} on three federated multiobjective learning benchmarks. We first describe the experimental setup in Section~\ref{ss:setup} and then present the results in Sections~\ref{ss:comparative analysis}--\ref{ss:hyperparameter effects}.

\subsection{Datasets and implementation}
\label{ss:setup}

We consider the following experimental settings:

\noindent\textbf{1) MultiMNIST:} We construct a two-objective dataset from MNIST by placing two randomly selected digits in each image—one in the top-left corner and the other in the bottom-right corner—while preserving the original image size.\\
\noindent\textbf{2) MNIST+FMNIST:} We generate a new dataset by randomly combining one image from MNIST and one from FashionMNIST, following a construction method similar to MultiMNIST.\\
\noindent\textbf{3) CIFAR10+MNIST:} We synthesize images by superimposing a randomly selected MNIST digit at a fixed location onto all three channels of CIFAR10 images.

In our experiments, we employ convolutional neural networks of varying scales, with model sizes spanning from 34.6k parameters to 1.73m parameters in the largest model, ResNet-18. All architectures are based on a shared encoder for all the objective functions, coupled with task-specific decoders comprising the final one or two linear layers.

Unless the number of clients is varied in Section 5.5, the main experiments use 100 clients. To induce data heterogeneity across clients, we follow the implementation in \cite{Acar2021} and use a Dirichlet distribution with concentration parameter $\alpha=0.3$ to determine the class proportions assigned to each client, while keeping the total number of data points equal across clients. To ensure a fair comparison, the same Dirichlet client partition, training/validation split, model initialization, random seeds, and minibatch sequences are used for all methods. For the MultiMNIST and MNIST+FMNIST datasets, we employ a LeNet-like CNN encoder \cite{Sener2018}, with two linear layers per task as decoders, resulting in a total of 34.6k parameters. For the CIFAR10+MNIST dataset, we use a CNN encoder adapted from \cite{Acar2021}, along with a single linear layer decoder, totaling 1.73m parameters. In terms of training configuration, we apply the negative log-likelihood loss across all three datasets and incorporate random image rotation for data augmentation during training. The batch size is set to 128, and unless otherwise noted, each training round comprises 10 local steps. Each experiment is repeated with 10 different random seeds. We report the mean and standard deviation across the runs, with the shaded regions in the figures indicating one standard deviation around the mean.

All datasets come with predefined training and test splits, which we adopt together with the default validation splits when available. In cases where a validation split is not provided, we randomly partition the training set into 80\% for training and 20\% for validation. Throughout all experiments, constant learning rates are used without any schedule or decay. The learning rates are selected by running each method under the corresponding experimental setting and choosing the value that yields the best performance on the validation set. 

For the MultiMNIST experiments, different learning-rate configurations are adopted across algorithms. Specifically, FSMGDA \cite{Yang2024} uses a global learning rate of 2 and a local learning rate of 0.1; FedCMOO \cite{Askin2024} employs a global learning rate of 1 and a local learning rate of 0.25; and FSMGDA-M-VR adopts a global learning rate of 1 with a local learning rate of 0.5. For the MNIST+FMNIST experiments, a similar configuration strategy is followed. FSMGDA uses global and local learning rates of 2 and 0.1, respectively; FedCMOO uses 1.2 and 0.2; and FSMGDA-M-VR uses 1.2 and 0.5. For the CIFAR10+MNIST experiments, FSMGDA and FedCMOO share the same learning-rate settings, with a global learning rate of 1.6 and a local learning rate of 0.3, whereas FSMGDA-M-VR uses a global learning rate of 1.6 and a local learning rate of 0.5.

\subsection{Comparative analysis of algorithm performance}
\label{ss:comparative analysis}

As shown in Figures \ref{fig:a} and \ref{fig:b}, FSMGDA-M-VR demonstrates clear advantages over FedCMOO and FSMGDA in terms  test accuracy and test loss on the MultiMNIST, MNIST+FMNIST, and CIFAR10+MNIST datasets. In terms of accuracy, FSMGDA-M-VR consistently achieves faster convergence in the early training rounds, reaching higher performance levels more rapidly than the other methods, particularly on the MultiMNIST and MNIST+FMNIST datasets, while matching or slightly surpassing the final accuracy of FedCMOO. Regarding test loss, FSMGDA-M-VR exhibits a steeper and more stable decline, indicating more efficient optimization and better robustness against gradient fluctuations. This smoother loss reduction, coupled with its rapid accuracy improvement, suggests that the momentum-based variance reduction mechanism effectively enhances training stability and convergence speed. Overall, FSMGDA-M-VR maintains communication efficiency while delivering superior or comparable optimization performance relative to FedCMOO and significantly outperforming FSMGDA in both convergence rate and stability.

\begin{figure}[!ht]
\centering
\begin{minipage}{0.33\textwidth}
\centering
\includegraphics[width=\linewidth]{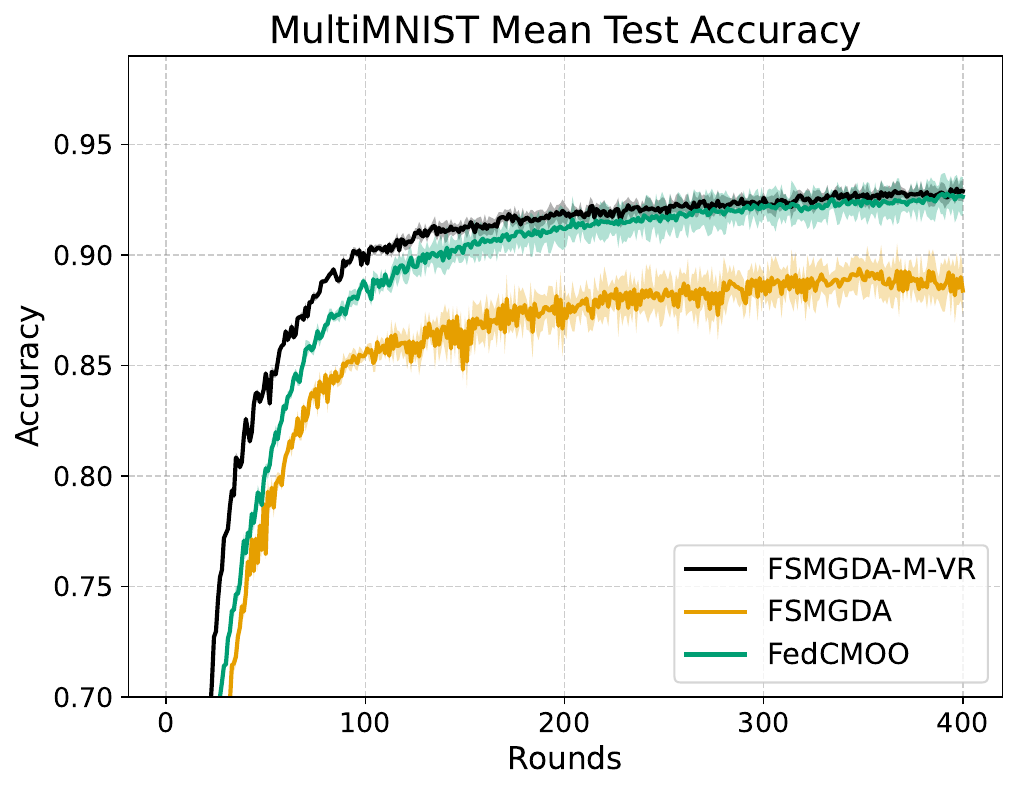}
\end{minipage}\hfill
\begin{minipage}{0.33\textwidth}
\centering
\includegraphics[width=\linewidth]{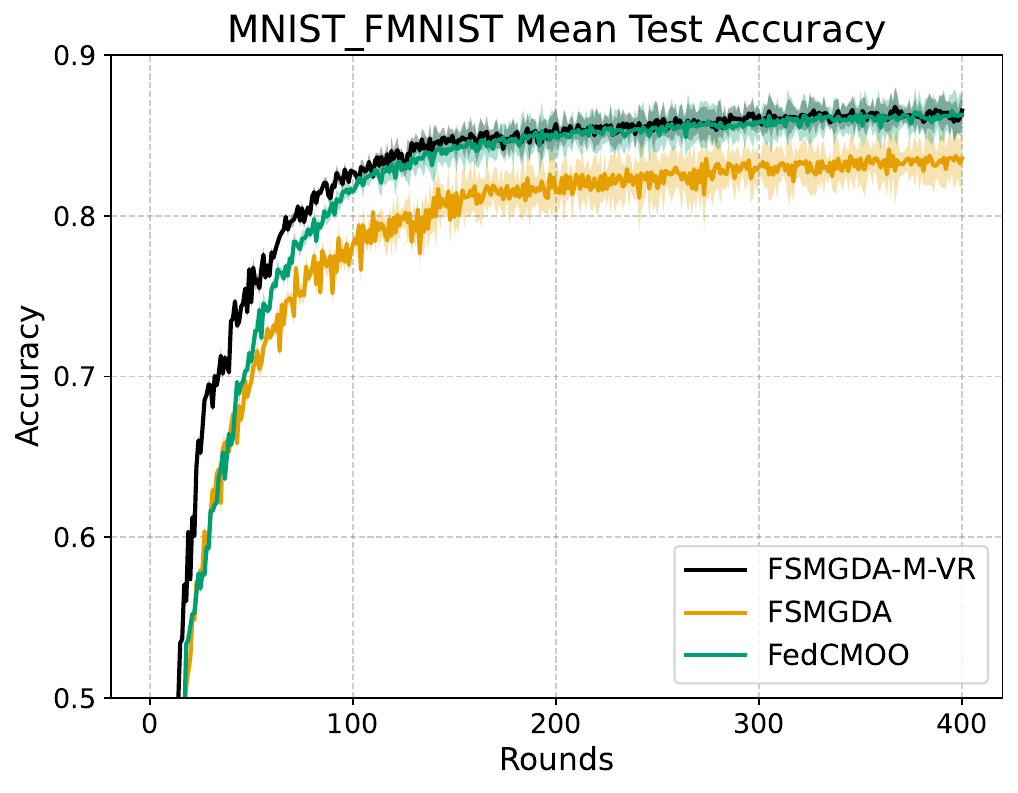}
\end{minipage}\hfill
\begin{minipage}{0.33\textwidth}
\centering
\includegraphics[width=\linewidth]{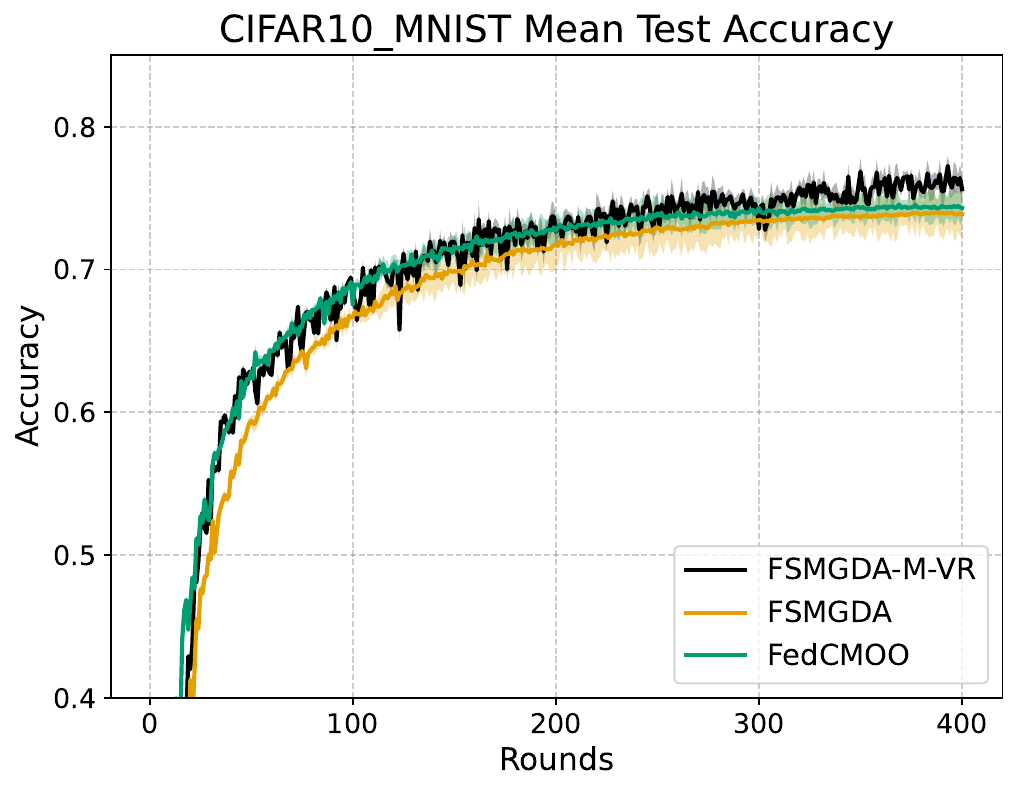}
\end{minipage}
\caption{Mean test accuracy with MultiMNIST, MNIST+FMNIST and CIFAR10+MNIST datasets.}\label{fig:a}
\end{figure}

\begin{figure}[!ht]
\centering
\begin{minipage}{0.33\textwidth}
\centering
\includegraphics[width=\linewidth]{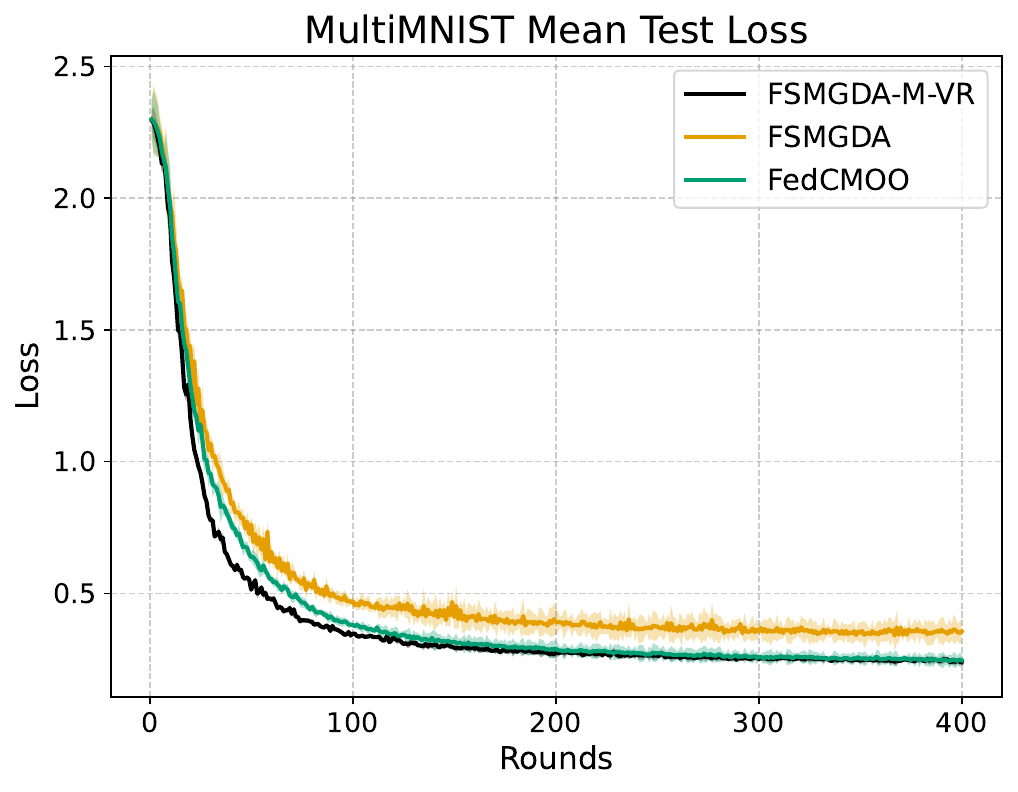}
\end{minipage}\hfill
\begin{minipage}{0.33\textwidth}
\centering
\includegraphics[width=\linewidth]{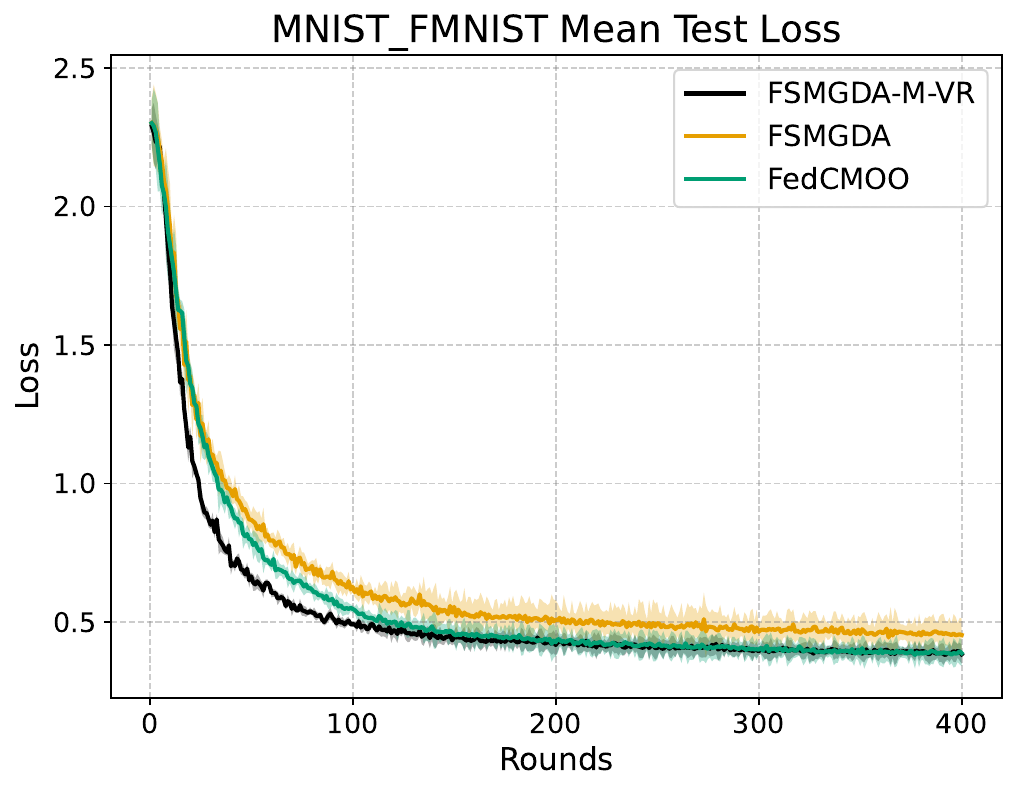}
\end{minipage}\hfill
\begin{minipage}{0.33\textwidth}
\centering
\includegraphics[width=\linewidth]{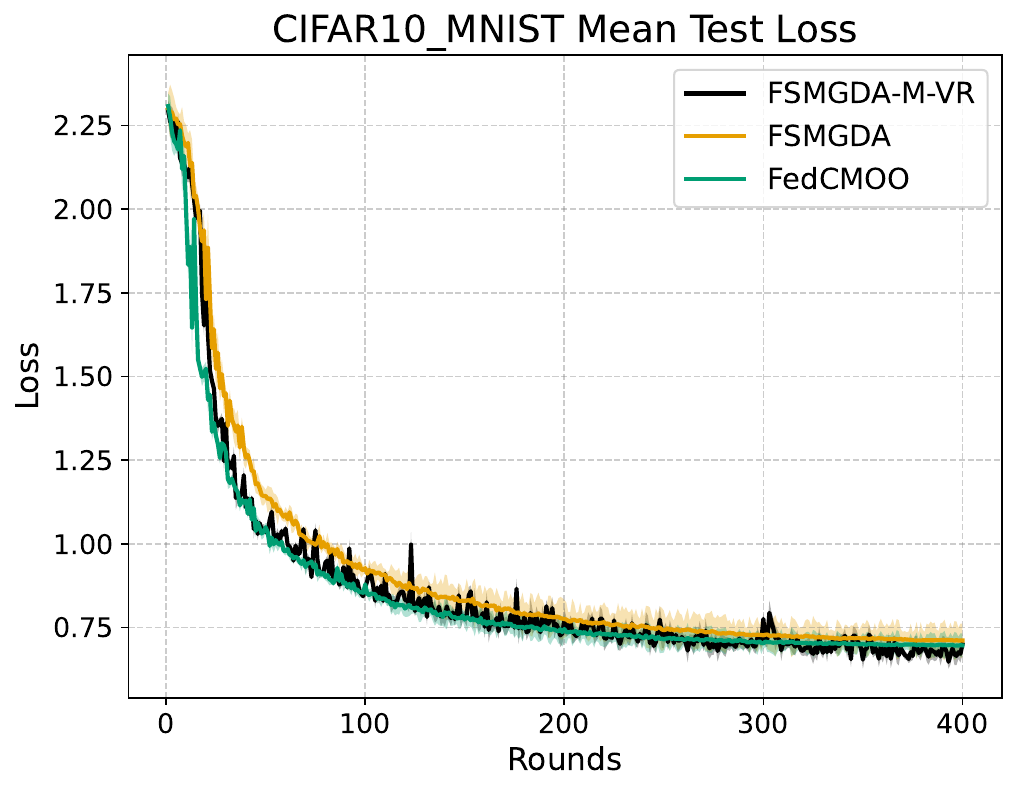}
\end{minipage}
\caption{Mean test loss in MultiMNIST, MNIST+FMNIST and CIFAR10+MNIST datasets.}\label{fig:b}
\end{figure}

Table \ref{tab:comparison_results} reports the final test accuracies of the three algorithms on the three bi-objective tasks, reported as mean $\pm$ standard deviation over 10 different random seeds. FSMGDA-M-VR performs particularly well on MultiMNIST and CIFAR10+MNIST: it achieves the highest mean accuracy on the first objective and ties with FedCMOO in mean accuracy on the second objective of MultiMNIST, while obtaining the highest mean accuracies on both objectives of CIFAR10+MNIST. On MNIST+FMNIST, FSMGDA-M-VR achieves the highest mean accuracy on the first objective, whereas FedCMOO performs better on the second objective.

\begin{table}[!ht]
\centering
\begin{tabular}{lccc}
\toprule
\textbf{Experiment} & \textbf{MNIST+FMNIST} & \textbf{MultiMNIST} & \textbf{CIFAR10+MNIST} \\
\midrule
FSMGDA & 93.2$\pm$1.1 / 73.4$\pm$0.9 & 91.1$\pm$0.6 / 87.1$\pm$0.9 & 54.2$\pm$0.7 / 93.3$\pm$1.1 \\
FedCMOO & 94.9$\pm$1.0 / \textbf{77.5$\pm$0.8} & 94.3$\pm$0.7 / 91.1 $\pm$ 0.8 &  55.2$\pm$0.6 / 93.7$\pm$1.0 \\
FSMGDA-M-VR & \textbf{96.4$\pm$0.5} / 76.0$\pm$0.4 & \textbf{94.5$\pm$0.3} / \textbf{91.1$\pm$0.3} & \textbf{57.2$\pm$0.3} / \textbf{96.0$\pm$0.5} \\
\bottomrule
\end{tabular}
\caption{The final accuracies (\%) of the first/second objectives in 2-objective settings. The bold values indicate the best accuracy for each objective.}
\label{tab:comparison_results}
\end{table}

\subsection{Comparison in terms of sample access cost}

Figures~\ref{fig:accuracy_subplots} and~\ref{fig:loss_subplots} report the mean test accuracy and mean test loss of FSMGDA, FedCMOO, and FSMGDA-M-VR against the cumulative number of samples accessed per client. On MultiMNIST, FSMGDA-M-VR achieves higher test accuracy than the other two methods over most of the considered range of sample accesses and yields the lowest final test loss. On MNIST+FMNIST, FSMGDA-M-VR generally achieves higher test accuracy over a substantial portion of the sample-access range and exhibits lower test loss than the other two methods. On CIFAR10+MNIST, FedCMOO exhibits faster initial accuracy improvement, while FSMGDA-M-VR catches up in later stages and attains a comparable final accuracy; its final test loss is also close to the lowest among the three methods.

\begin{figure}[!ht]
	\centering
	\begin{minipage}{0.33\textwidth}
		\centering
		\includegraphics[width=\linewidth]{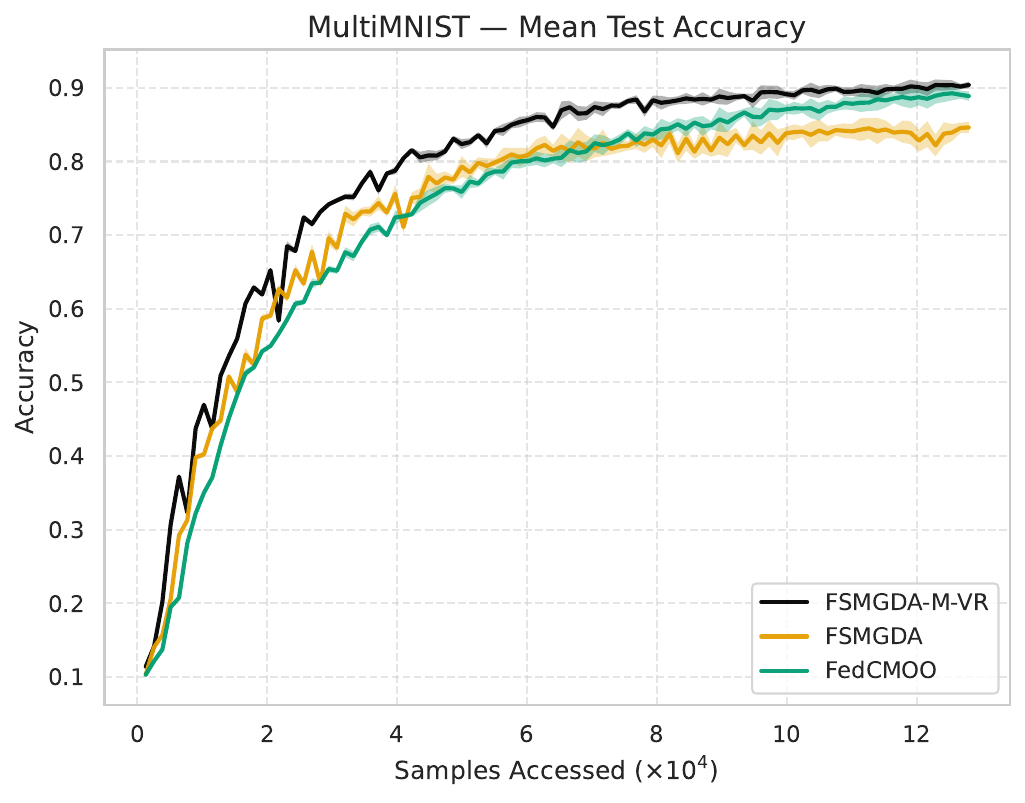}
	\end{minipage}\hfill
	\begin{minipage}{0.33\textwidth}
		\centering
		\includegraphics[width=\linewidth]{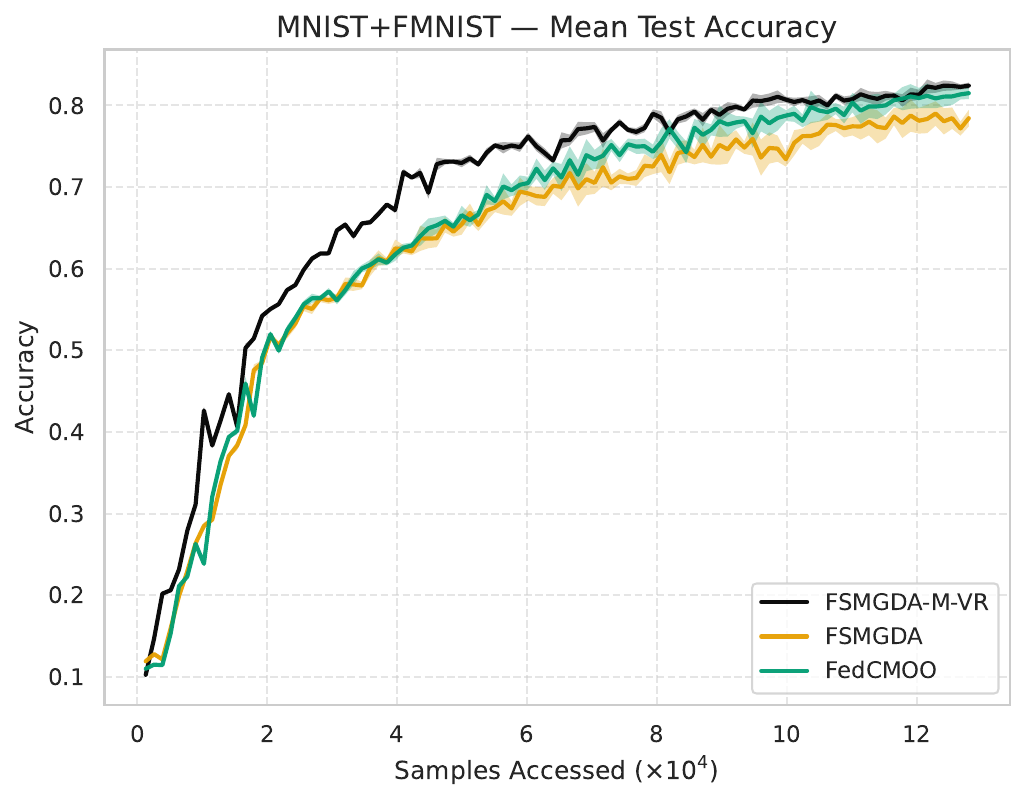}
	\end{minipage}\hfill
	\begin{minipage}{0.33\textwidth}
		\centering
		\includegraphics[width=\linewidth]{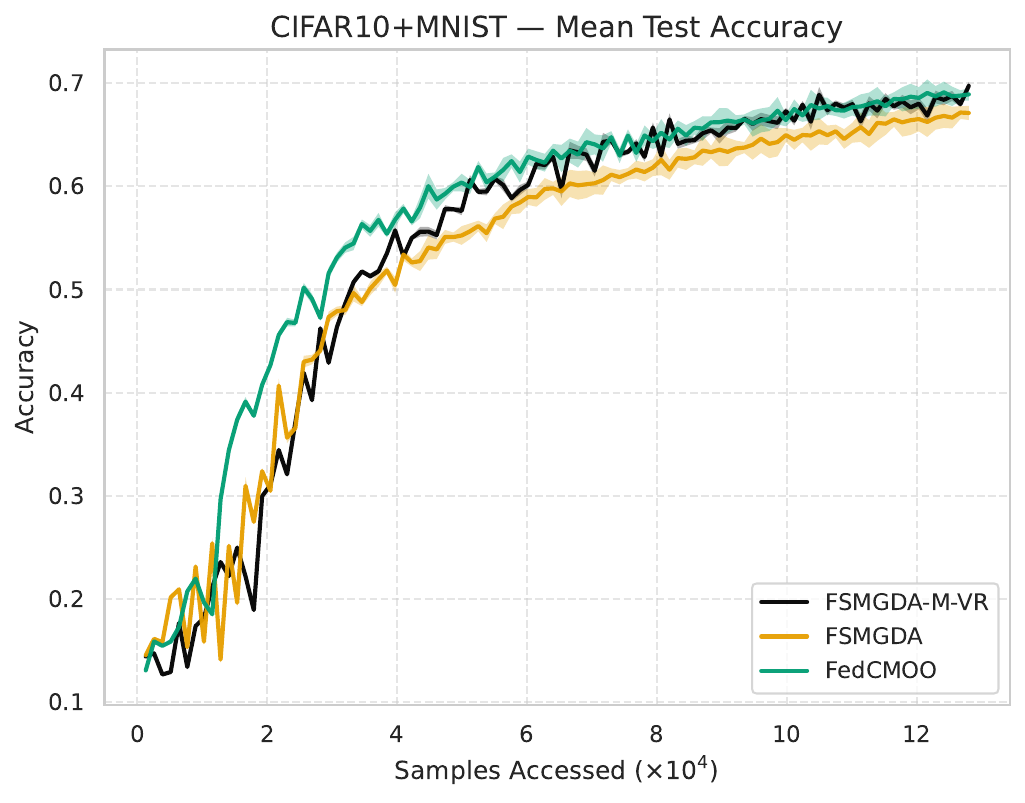}
	\end{minipage}
	\caption{Test accuracy versus cumulative samples accessed per client on the MultiMNIST, MNIST+FMNIST, and CIFAR10+MNIST datasets.}
	\label{fig:accuracy_subplots}
\end{figure}

\begin{figure}[!ht]
	\centering
	\begin{minipage}{0.33\textwidth}
		\centering
		\includegraphics[width=\linewidth]{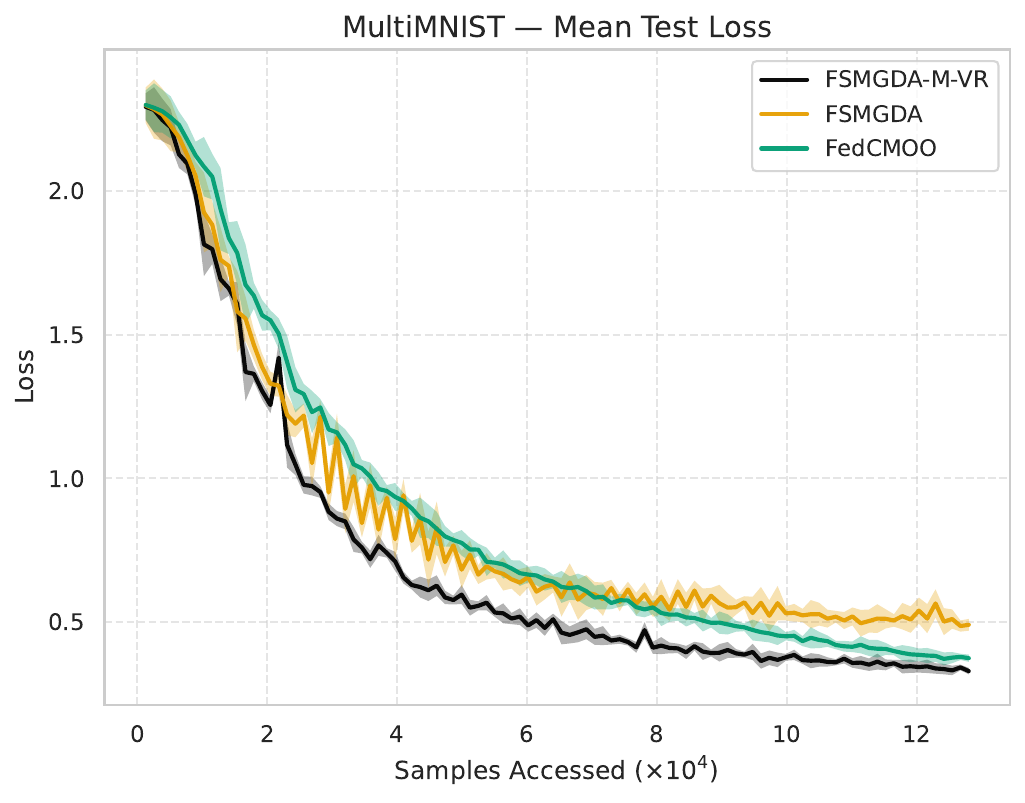}
	\end{minipage}\hfill
	\begin{minipage}{0.33\textwidth}
		\centering
		\includegraphics[width=\linewidth]{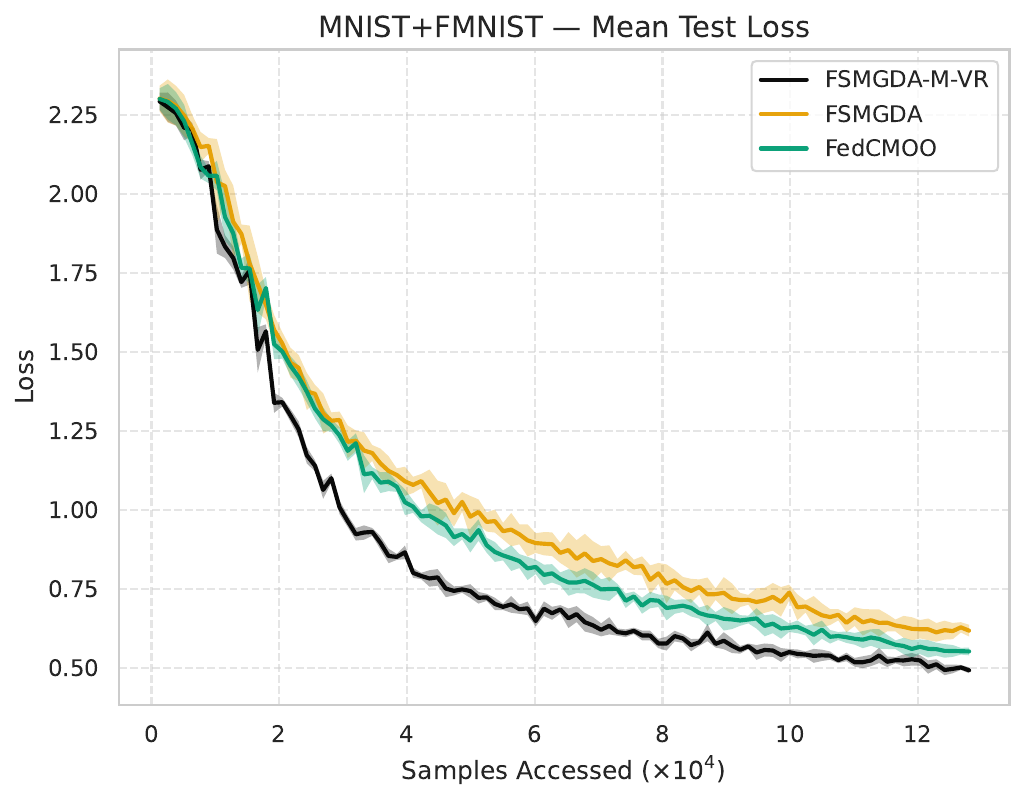}
	\end{minipage}\hfill
	\begin{minipage}{0.33\textwidth}
		\centering
		\includegraphics[width=\linewidth]{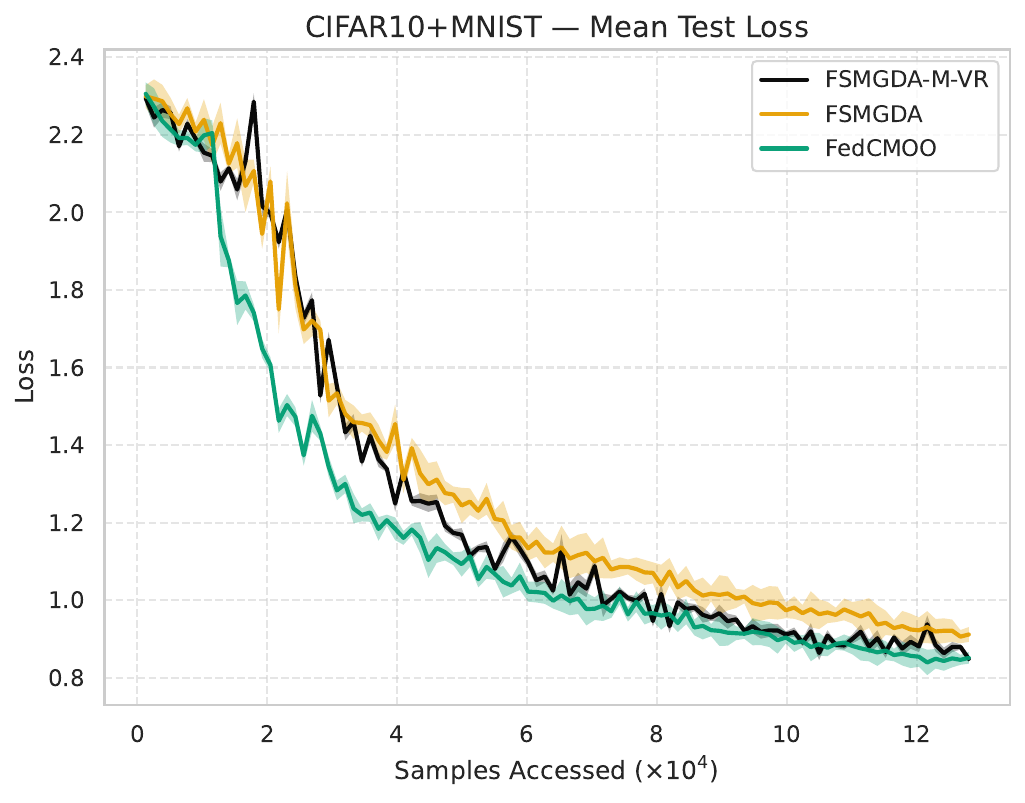}
	\end{minipage}
	\caption{Test loss versus cumulative samples accessed per client on the MultiMNIST, MNIST+FMNIST, and CIFAR10+MNIST datasets.}
	\label{fig:loss_subplots}
\end{figure}

\subsection{Comparing the local training performance of FSMGDA-M-VR, FSMGDA, and FedCMOO}

Figure \ref{fig:placeholder} compares the average local training progress across clients in terms of loss decrease and accuracy improvement. It can be observed that FSMGDA‑M‑VR demonstrates superior local optimization performance compared to FSMGDA, while maintaining results comparable to FedCMOO. Specifically, in terms of local loss reduction, FSMGDA‑M‑VR consistently achieves a steeper and larger decrease in loss, especially during the early rounds, indicating that it optimizes local objectives more effectively and rapidly. Regarding local accuracy improvement, FSMGDA‑M‑VR also exhibits sharper gains and faster convergence in the initial training phase, with its trajectory closely aligning with that of FedCMOO across both datasets.

\begin{figure}[!ht]
    \centering
    \includegraphics[width=\linewidth]{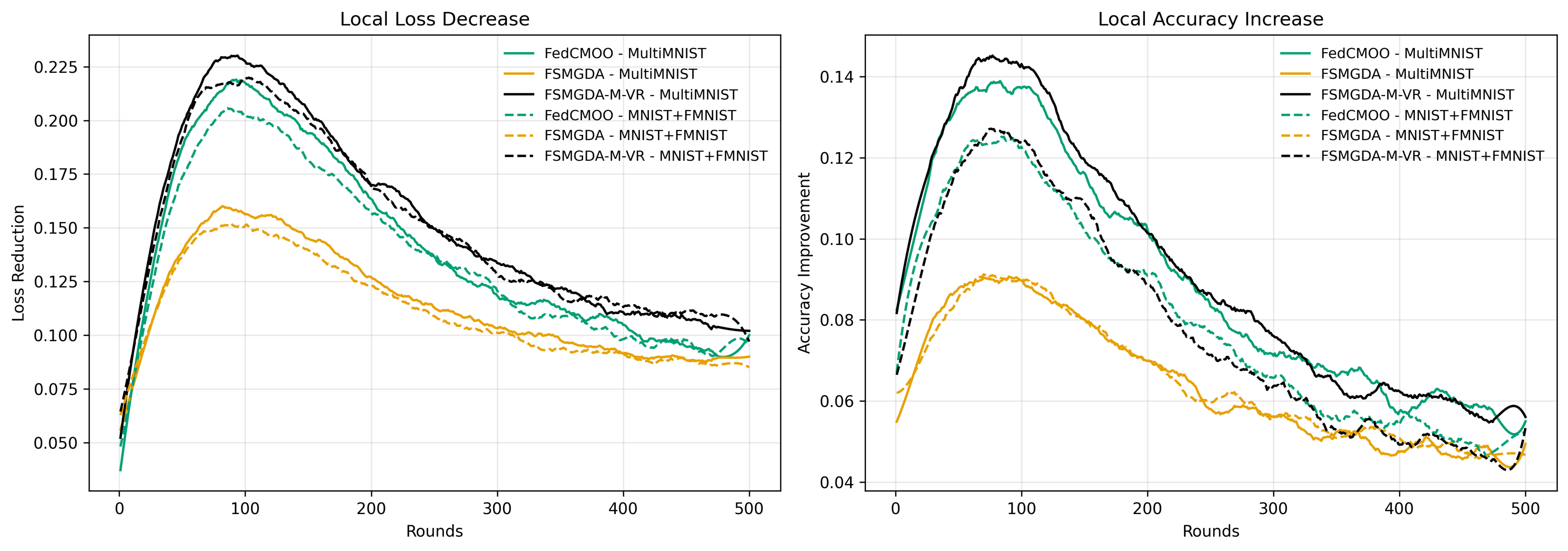}
    \caption{Comparison of average local training progress (loss ↓ and accuracy ↑) across clients among FSMGDA, FedCMOO, and FSMGDA-M-VR.}
    \label{fig:placeholder}
\end{figure}

\subsection{Analysis of hyperparameter effects on convergence}
\label{ss:hyperparameter effects}

\textbf{1) Impact of the batch size:}
Figures \ref{fig:c} and \ref{fig:d}  compare the performance of FSMGDA-M-VR with
different batch sizes on the MultiMNIST, MNIST+FMNIST, and CIFAR10+MNIST
datasets. The results indicate that increasing the batch size generally leads to faster convergence,
higher final test accuracy, and lower test loss, while also reducing the fluctuations
of the training curves. On MultiMNIST, the larger batch sizes, particularly 128 and
256, converge substantially faster than the batch size of 16 and attain higher final
accuracy and lower loss. A similar trend is observed on MNIST+FMNIST, where the
performance differences between batch sizes 128 and 256 become relatively small in
the later stages of training. On CIFAR10+MNIST, the effect of the batch size is more
pronounced: batch sizes 128 and 256 yield much faster convergence and noticeably
better final accuracy and loss than the smaller batch sizes, especially compared with
the batch size of 16.

\begin{figure}[!ht]
\centering
\begin{minipage}{0.33\textwidth}
\centering
\includegraphics[width=\linewidth]{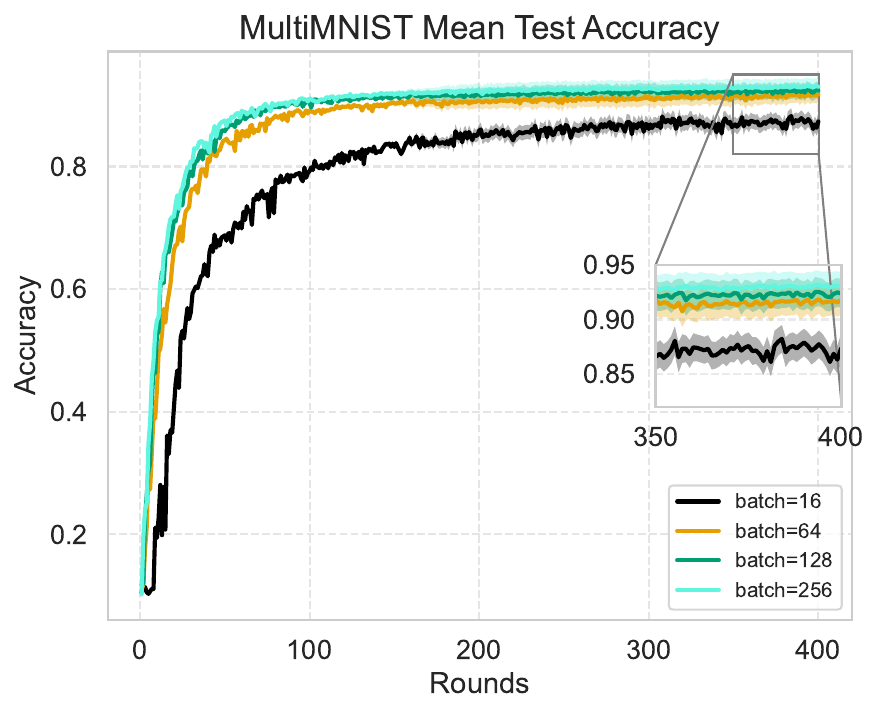}
\end{minipage}\hfill
\begin{minipage}{0.33\textwidth}
\centering
\includegraphics[width=\linewidth]{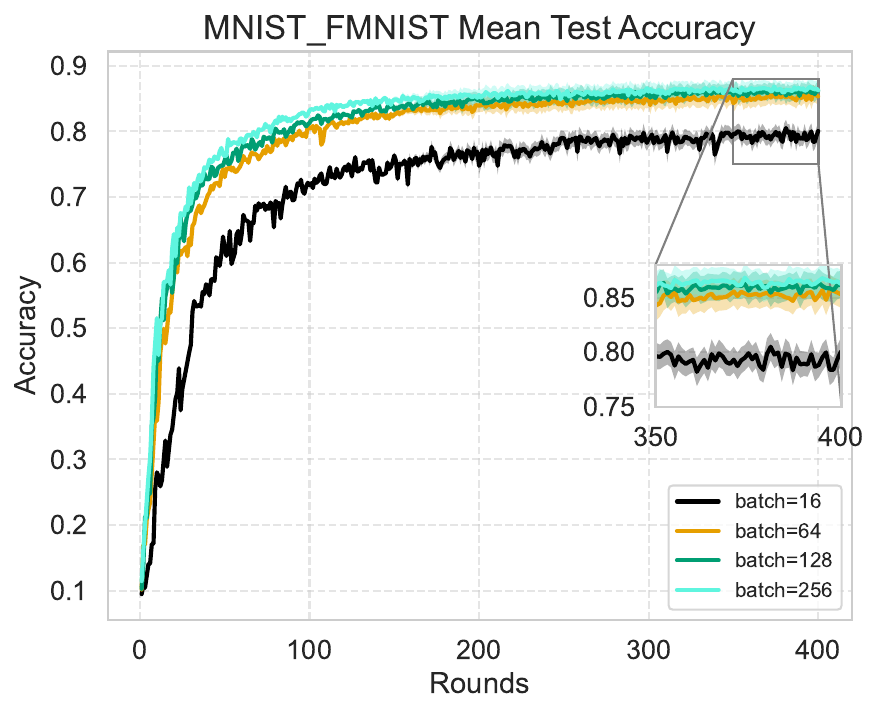}
\end{minipage}\hfill
\begin{minipage}{0.33\textwidth}
\centering
\includegraphics[width=\linewidth]{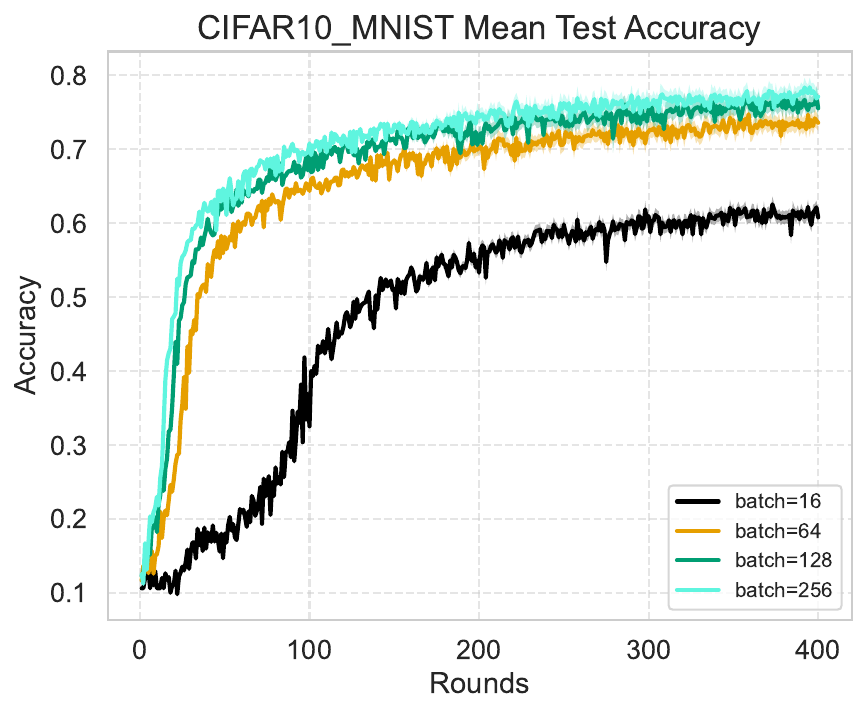}
\end{minipage}
\caption{Mean test accuracy of different batch sizes on MultiMNIST, MNIST+FMNIST and CIFAR10+MNIST datasets.}\label{fig:c}
\end{figure}

\begin{figure}[!ht]
	\centering
\begin{minipage}{0.33\textwidth}
\centering
\includegraphics[width=\linewidth]{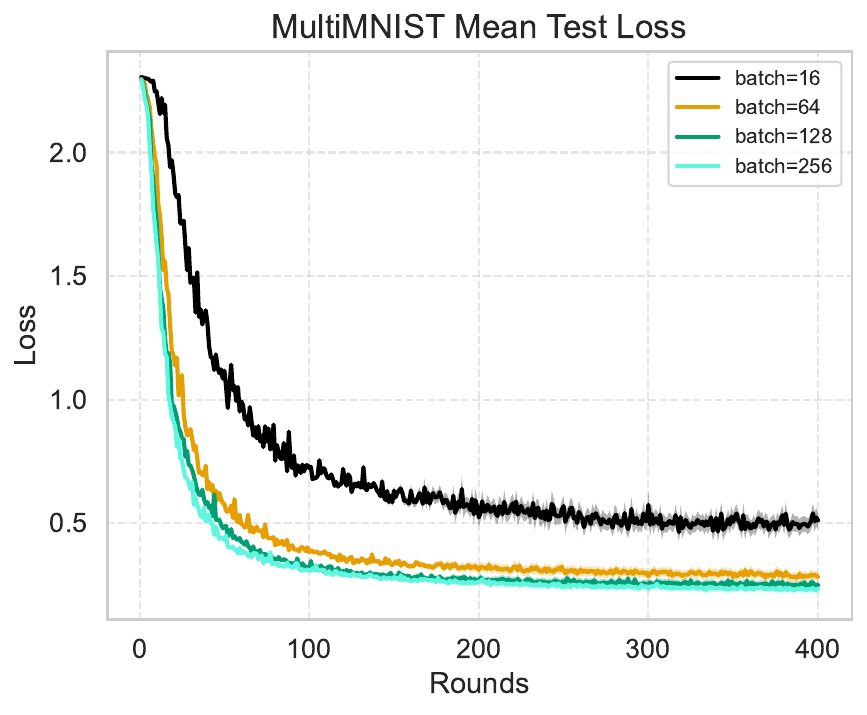}
\end{minipage}\hfill
\begin{minipage}{0.33\textwidth}
\centering
\includegraphics[width=\linewidth]{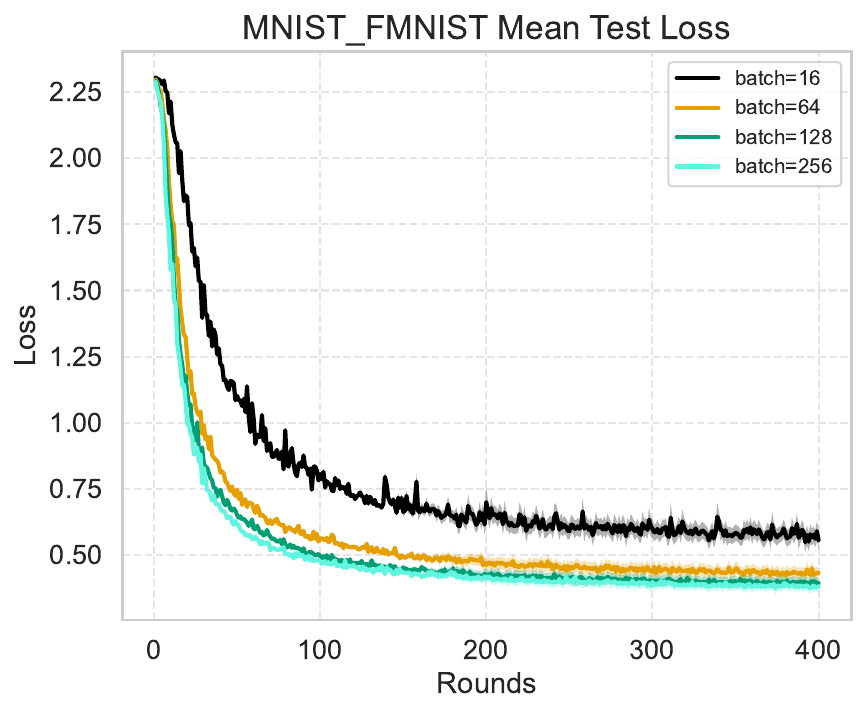}
\end{minipage}\hfill
\begin{minipage}{0.33\textwidth}
\centering
\includegraphics[width=\linewidth]{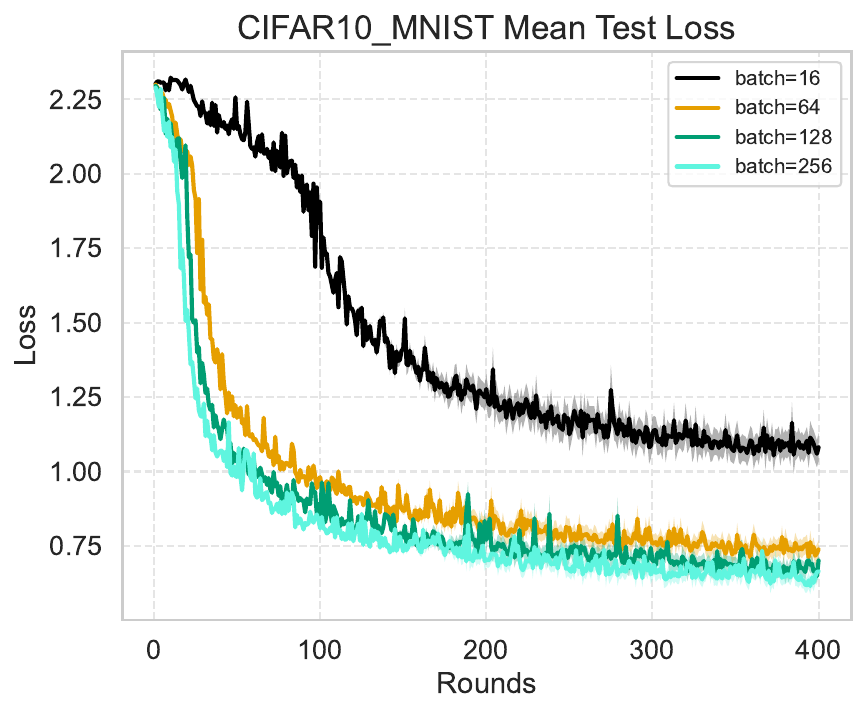}
\end{minipage}
\caption{Mean test loss of different batch sizes on MultiMNIST, MNIST+FMNIST and CIFAR10+MNIST datasets.}\label{fig:d}
\end{figure}

\noindent\textbf{2) Impact of the number of clients:}
Figures \ref{fig:e} and \ref{fig:f} show how the number of clients affects the performance of FSMGDA‑M‑VR across the MultiMNIST, MNIST+FMNIST, and CIFAR10+MNIST datasets. The results indicate that using more clients (e.g., 30) consistently leads to higher final test accuracy and lower final test loss. This improvement is most noticeable on the more heterogeneous CIFAR10+MNIST dataset, where a larger client pool better captures diverse data, and on MultiMNIST, where correlated objectives benefit from updates aggregated across more participants. On the relatively balanced MNIST+FMNIST tasks, increasing the number of clients has little impact on final test accuracy and final test loss.

\begin{figure}[!ht]
\centering
\begin{minipage}{0.33\textwidth}
\centering
\includegraphics[width=\linewidth]{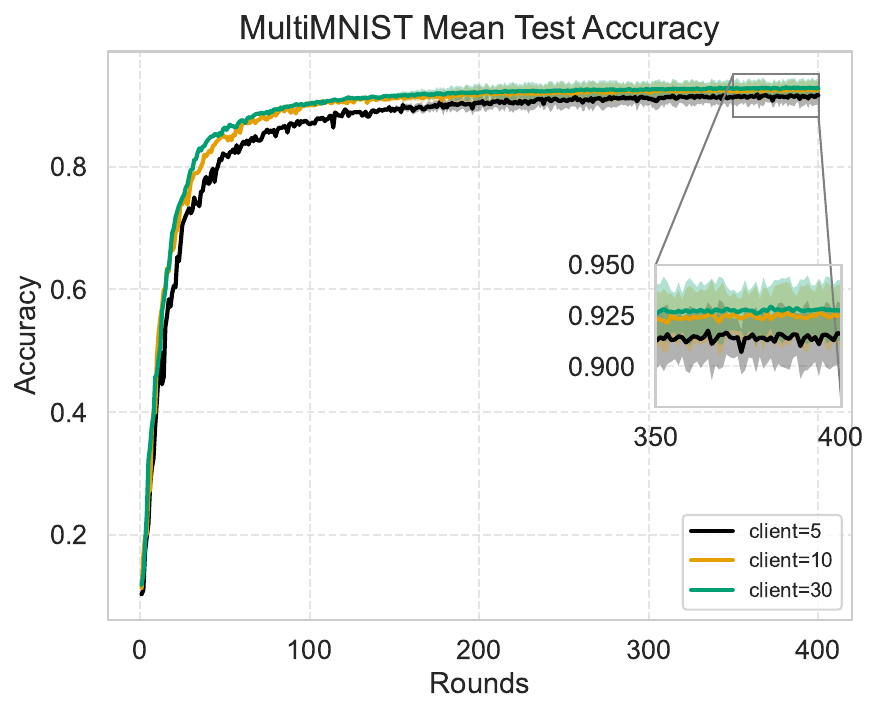}
\end{minipage}\hfill
\begin{minipage}{0.33\textwidth}
\centering
\includegraphics[width=\linewidth]{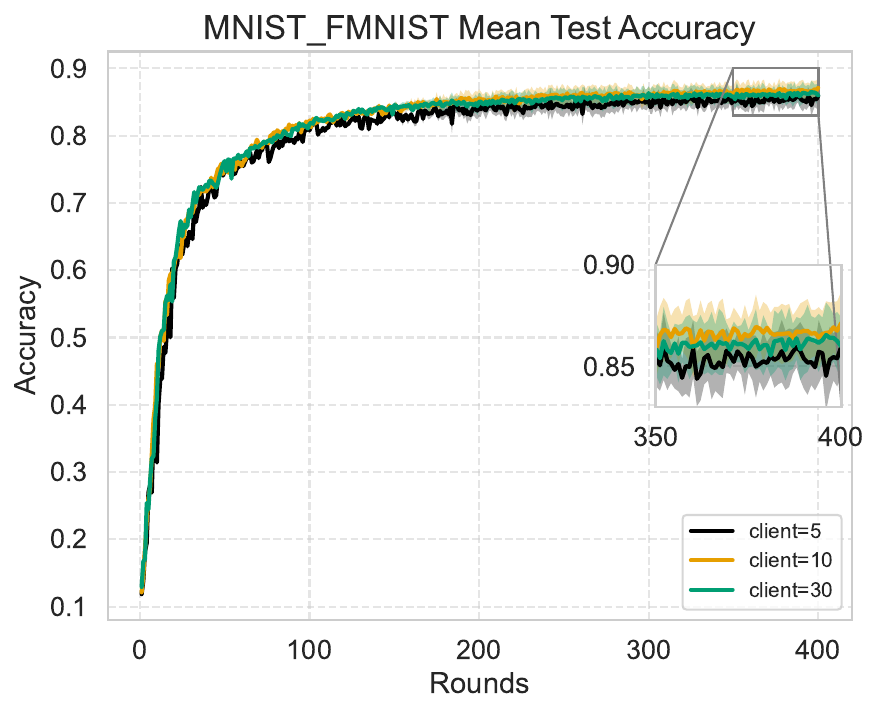}
\end{minipage}\hfill
\begin{minipage}{0.33\textwidth}
\centering
\includegraphics[width=\linewidth]{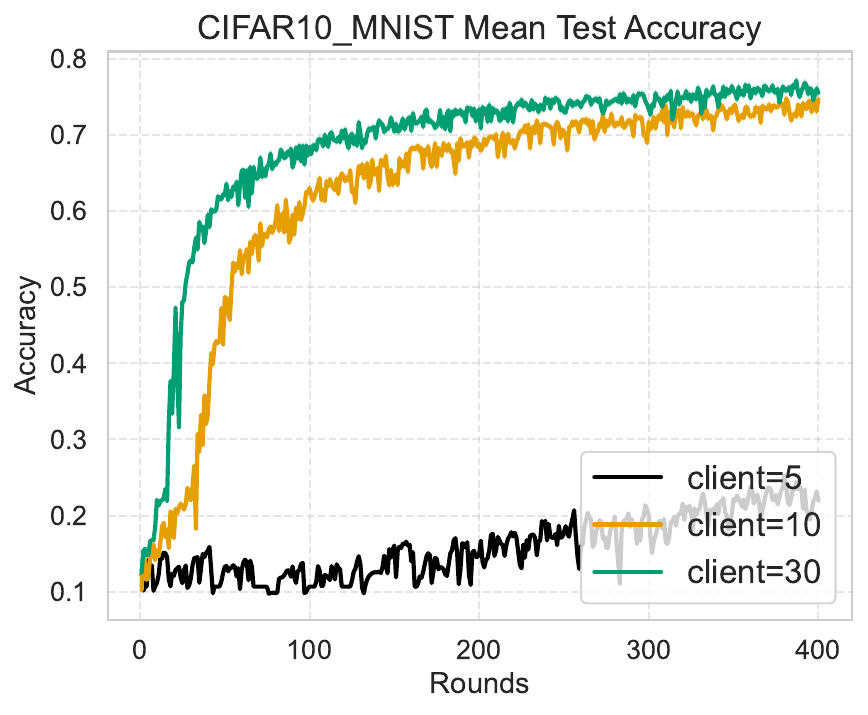}
\end{minipage}
\caption{Mean test accuracy of different client numbers on MultiMNIST, MNIST+FMNIST and CIFAR10+MNIST datasets.}\label{fig:e}
\end{figure}

\begin{figure}[!ht]
\begin{minipage}{0.33\textwidth}
\centering
\includegraphics[width=\linewidth]{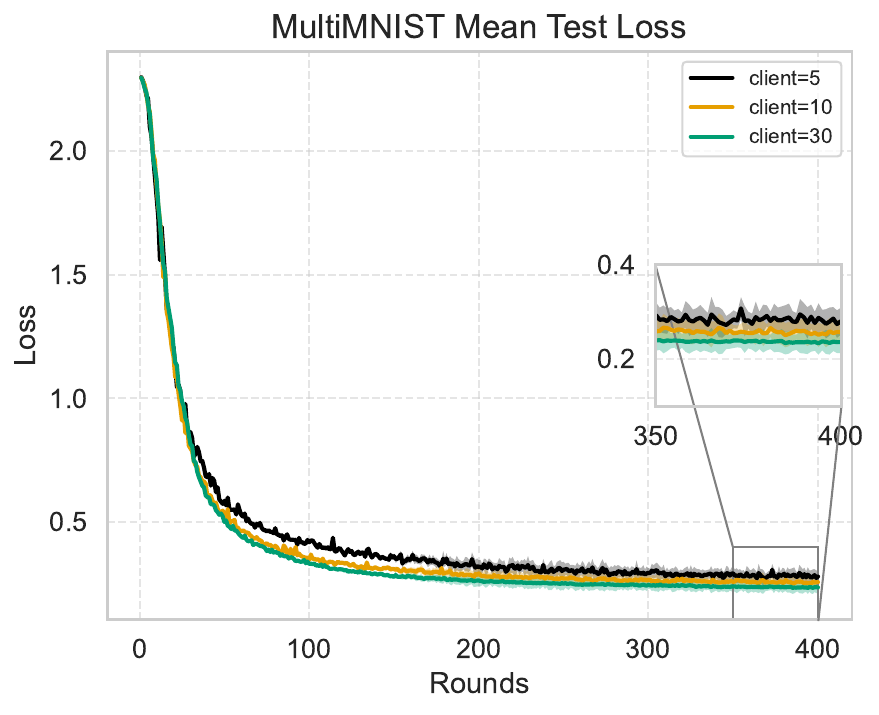}
\end{minipage}\hfill
\begin{minipage}{0.33\textwidth}
\centering
\includegraphics[width=\linewidth]{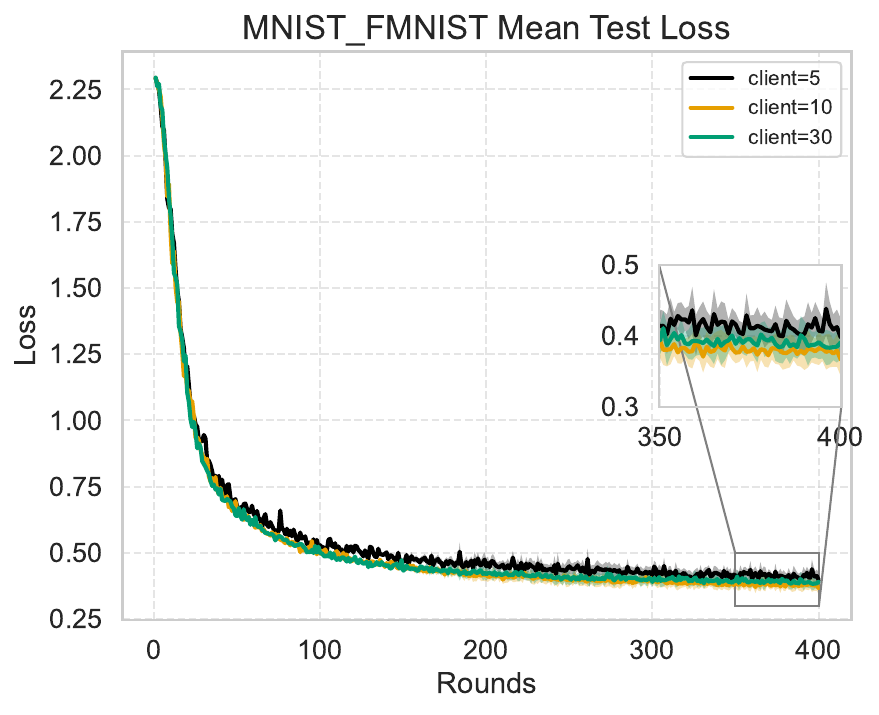}
\end{minipage}\hfill
\begin{minipage}{0.33\textwidth}
\centering
\includegraphics[width=\linewidth]{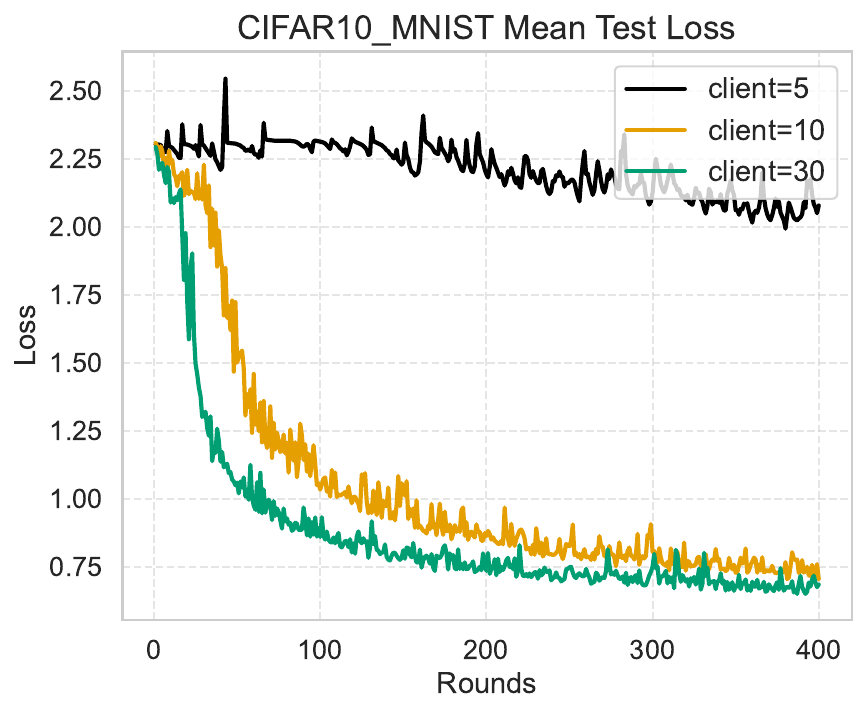}
\end{minipage}
\caption{Mean test loss of different client numbers on MultiMNIST, MNIST+FMNIST and CIFAR10+MNIST datasets. }\label{fig:f}
\end{figure}

\noindent\textbf{3) Impact of local steps $K$:}
As shown in Figures \ref{local:a} and \ref{local:b}, which depict the test accuracy and loss curves, moderately increasing the number of local steps $K$---typically within a low to medium range---effectively accelerates convergence, improves final accuracy, and reduces loss.
For instance, increasing the number of local steps $K$ from 1 to 5 yields clear and consistent improvements in both performance and training stability across all datasets. However, further increasing the number of local steps $K$, for example to 50, yields diminishing returns, with no significant additional improvement in accuracy or reduction in loss. 

\begin{figure}[!ht]
\centering
\begin{minipage}{0.33\textwidth}
\centering
\includegraphics[width=\linewidth]{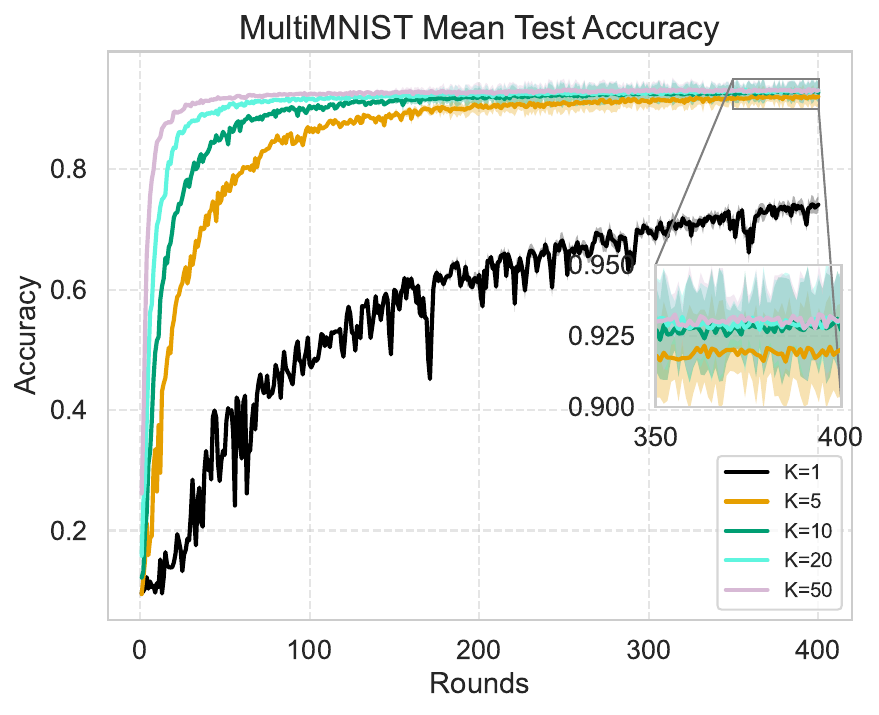}
\end{minipage}\hfill
\begin{minipage}{0.33\textwidth}
\centering
\includegraphics[width=\linewidth]{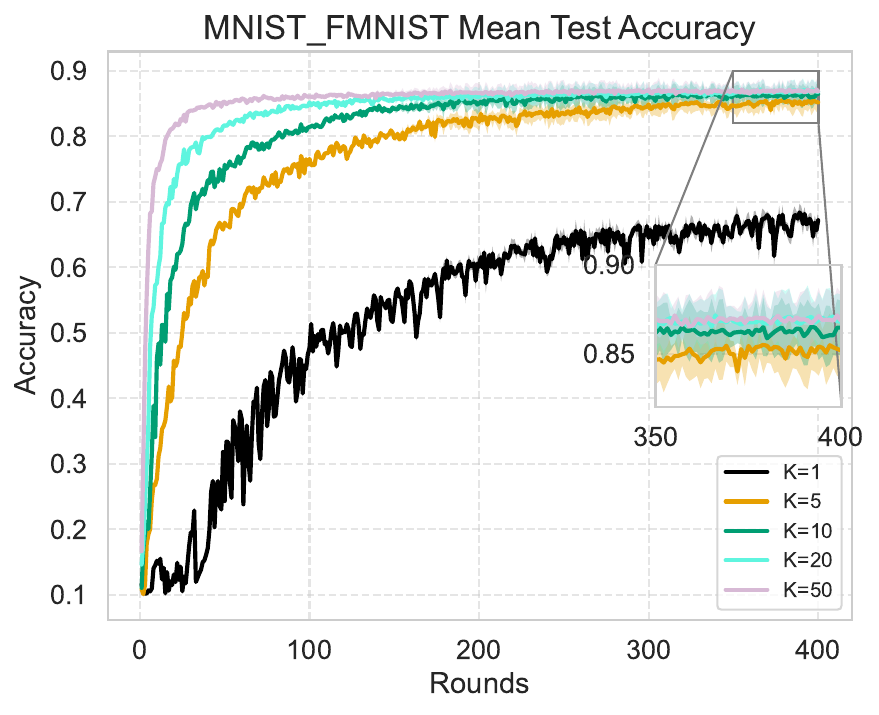}
\end{minipage}\hfill
\begin{minipage}{0.33\textwidth}
\centering
\includegraphics[width=\linewidth]{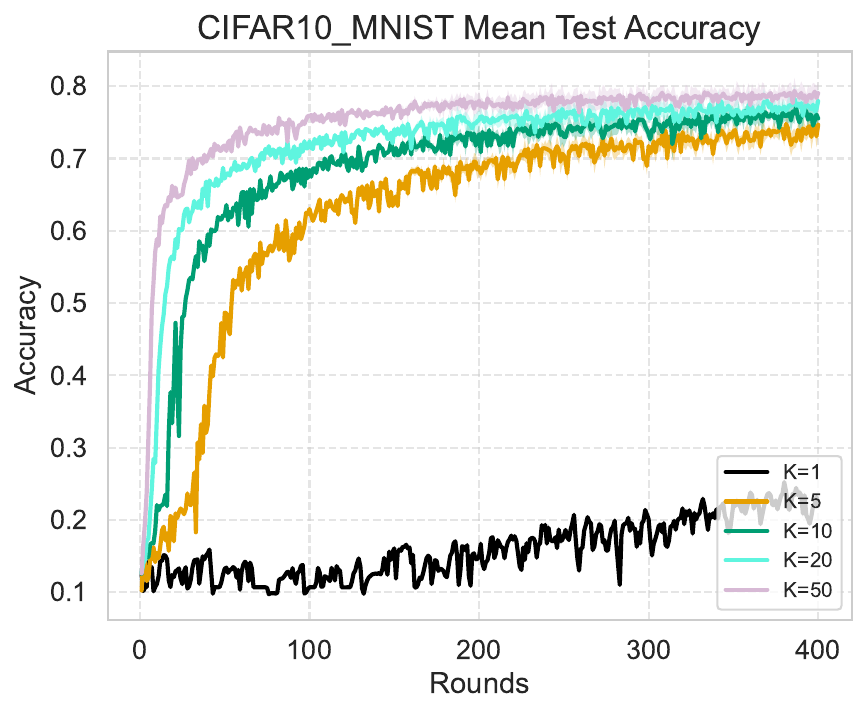}
\end{minipage}
\caption{Mean test accuracy for different numbers of local steps $K$.}\label{local:a}
\end{figure}

\begin{figure}[!ht]
	\centering
\begin{minipage}{0.33\textwidth}
\centering
\includegraphics[width=\linewidth]{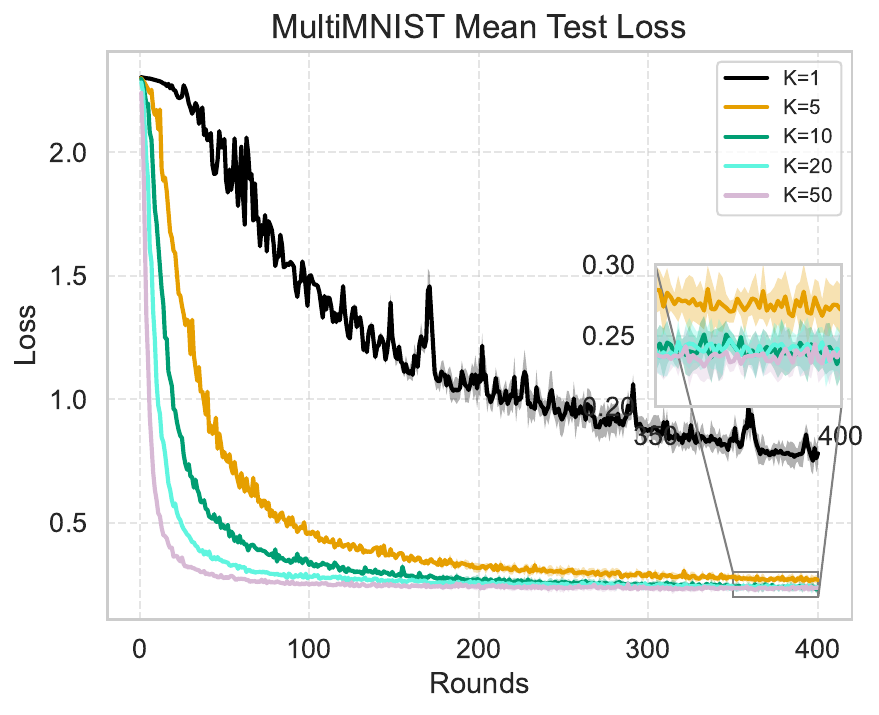}
\end{minipage}\hfill
\begin{minipage}{0.33\textwidth}
\centering
\includegraphics[width=\linewidth]{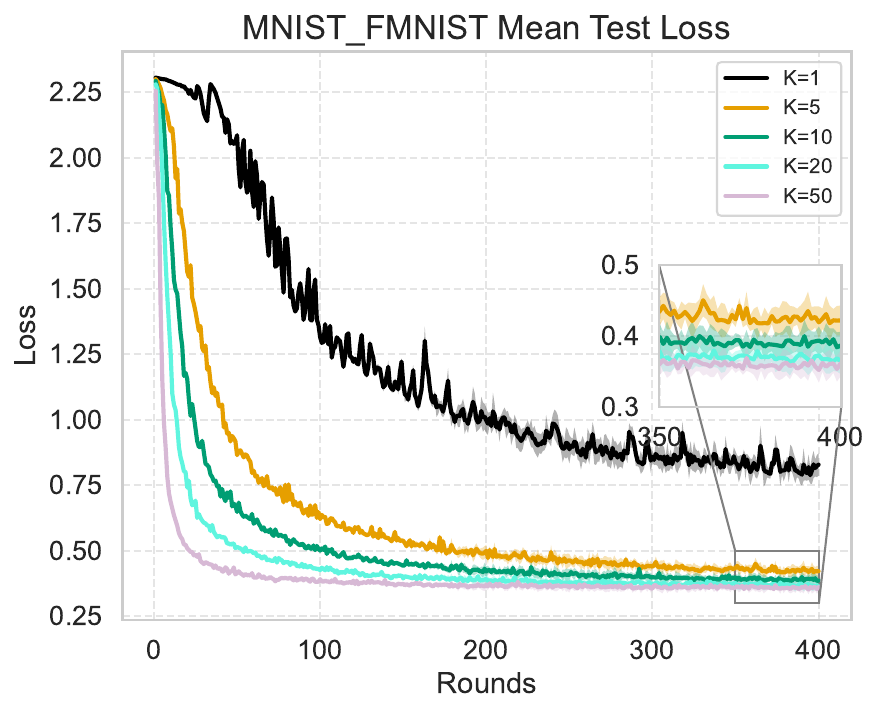}
\end{minipage}\hfill
\begin{minipage}{0.33\textwidth}
\centering
\includegraphics[width=\linewidth]{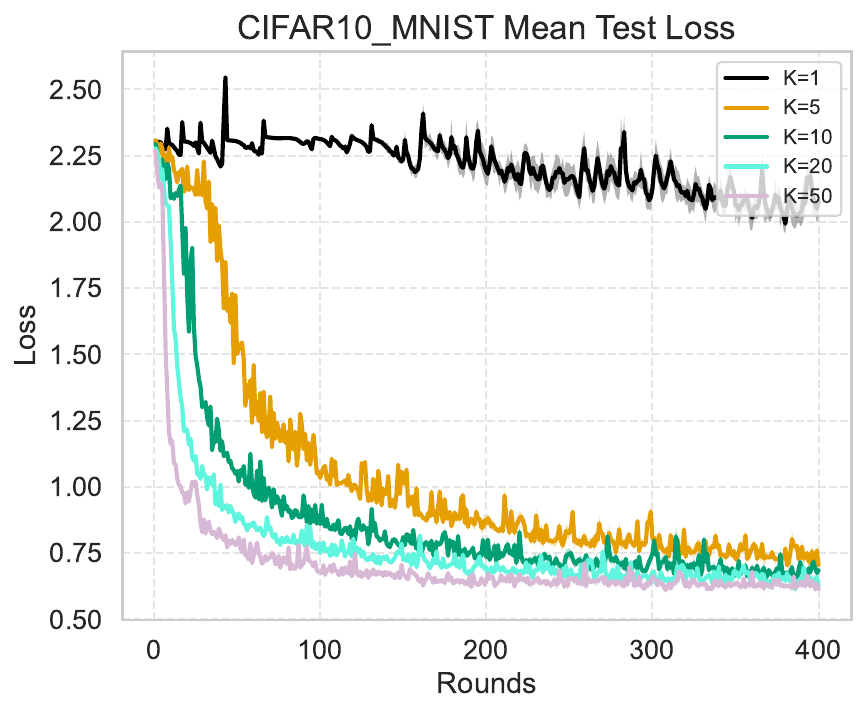}
\end{minipage}
\caption{Mean test loss for different numbers of local steps $K$.}\label{local:b}
\end{figure}

\section{Conclusion}

This paper addressed federated multiobjective optimization by proposing a momentum-based variance-reduced algorithm that incorporates a momentum-driven gradient estimator into the local updates to reduce the variance and bias of stochastic updates. Theoretically, under the stated assumptions, the proposed algorithm achieves a convergence rate of $\mathcal{O}(T^{-2/3})$, improving upon the $\mathcal{O}(T^{-1/2})$ rates of existing methods such as FSMGDA and FedCMOO. Extensive experiments on federated multiobjective optimization benchmarks further demonstrated the practical effectiveness and competitive performance of the proposed method.

\paragraph{Acknowledgements.} This work was supported in part by NSFC (No. 12271067), the Natural Science Foundation of Chongqing (No. CSTB2025NSCQ-LZX0060), the Science and Technology Research Program of Chongqing Education Commission (No. KJQN202400760), and the Australian Research Council (ARC DP230101749).

\end{document}